\documentclass[accepted]{uai2026}
\usepackage[american]{babel}

\usepackage{natbib}
\usepackage{mathtools}
\usepackage{booktabs}
\usepackage{tikz}

\usepackage{graphicx}
\graphicspath{{../}{./}}
 
\usepackage[inline]{enumitem}
\usepackage{microtype}

\usepackage{graphicx}
\usepackage{subcaption}
\usepackage{booktabs}
\usepackage{makecell}
\usepackage{csquotes}
\usepackage{xurl} 
\usetikzlibrary{shapes.geometric, arrows.meta, positioning}

\usepackage{amsmath}
\usepackage{amssymb}
\usepackage{mathtools}
\usepackage{amsthm}
\usepackage{thmtools}
 
\usepackage[capitalize]{cleveref}

\theoremstyle{plain}

\theoremstyle{definition}
\newtheorem{definition}{Definition}[section]
\newtheorem{assumption}{Assumption}[section]
\theoremstyle{remark}

\newtheorem{example}{Example}[section]
\usepackage[most]{tcolorbox}
\tcolorboxenvironment{example}{
  enhanced jigsaw,
  breakable,
  colback=black!5,
  frame hidden,
  boxrule=0pt,
  left=5pt, right=5pt, top=4pt, bottom=4pt,
  before skip=\topsep, after skip=\topsep,
}
\crefname{assumption}{Assump.}{Assumptions}
\Crefname{assumption}{Assumption}{Assumptions}
\crefname{definition}{Def.}{Definitions}
\Crefname{definition}{Definition}{Definitions}

\DeclareMathOperator*{\argmin}{arg\,min}

\newcommand{\cM}{\mathcal{M}}

\newcommand{\Var}{\text{Var}}

\newcommand{\Hclass}{\mathcal{T}}
\newcommand{\cS}{S}
\newcommand{\Sagree}{\mathcal{S}}
\newcommand{\Sdisagree}{\overline{\mathcal{C}}}
\newcommand{\Sle}{\mathcal{S}_{\downarrow}}
\newcommand{\Sge}{\mathcal{S}_{\uparrow}}
\newcommand{\Sletight}{\mathcal{S}_{\Downarrow}}
\newcommand{\Sgetight}{\mathcal{S}_{\Uparrow}}
\newcommand{\Scorrect}{\mathcal{C}}
\newcommand{\Slehat}{\widehat{\mathcal{S}}_{\downarrow}}
\newcommand{\Sgehat}{\widehat{\mathcal{S}}_{\uparrow}}
\newcommand{\Sletighthat}{\widehat{\mathcal{S}}_{\Downarrow}}
\newcommand{\Sgetighthat}{\widehat{\mathcal{S}}_{\Uparrow}}
\newcommand{\bsL}{b_{L}}
\newcommand{\bsU}{b_{U}}
\newcommand{\lneu}{\mathsf{neu}}\newcommand{\lmon}{\mathsf{mon}}
\newcommand{\lcor}{\mathsf{cor}}\newcommand{\ldef}{\mathsf{def}}
\newcommand{\bLneu}{\mathbf{1}\{\bsL=\lneu\}}\newcommand{\bLmon}{\mathbf{1}\{\bsL=\lmon\}}
\newcommand{\bLcor}{\mathbf{1}\{\bsL=\lcor\}}\newcommand{\bLdef}{\mathbf{1}\{\bsL=\ldef\}}
\newcommand{\bUneu}{\mathbf{1}\{\bsU=\lneu\}}\newcommand{\bUmon}{\mathbf{1}\{\bsU=\lmon\}}
\newcommand{\bUcor}{\mathbf{1}\{\bsU=\lcor\}}\newcommand{\bUdef}{\mathbf{1}\{\bsU=\ldef\}}
\newcommand{\bsLhat}{\hat{b}_{L}}
\newcommand{\bsUhat}{\hat{b}_{U}}
\newcommand{\bLneuhat}{\mathbf{1}\{\bsLhat=\lneu\}}\newcommand{\bLmonhat}{\mathbf{1}\{\bsLhat=\lmon\}}
\newcommand{\bLcorhat}{\mathbf{1}\{\bsLhat=\lcor\}}\newcommand{\bLdefhat}{\mathbf{1}\{\bsLhat=\ldef\}}
\newcommand{\bUneuhat}{\mathbf{1}\{\bsUhat=\lneu\}}\newcommand{\bUmonhat}{\mathbf{1}\{\bsUhat=\lmon\}}
\newcommand{\bUcorhat}{\mathbf{1}\{\bsUhat=\lcor\}}\newcommand{\bUdefhat}{\mathbf{1}\{\bsUhat=\ldef\}}

\def\ci{\perp\!\!\!\perp}

\usepackage[dvipsnames]{xcolor}
\usepackage{hyperref}
\hypersetup{
  colorlinks=true,
  linkcolor=RawSienna,
  citecolor=gray,
  urlcolor=gray,
}
\allowdisplaybreaks

\title{Bounding the Causal Impact of ML-assisted Decision-Making \\ via Counterfactual Correctness}

\author[1]{\href{mailto:jzhan367@jh.edu?Subject=Your UAI 2026 paper}{Jonathan~Zhang}}
\author[1]{Erik~Skalnes}
\author[1]{Jacob~M.~Chen}
\author[1]{Michael~Oberst}

\affil[1]{Computer Science Dept., Johns Hopkins University, Baltimore, Maryland, USA}

\begin{document}
\maketitle

\begin{abstract}
  Predictive machine learning (ML) models are increasingly used to aid human decision-makers across various high-risk domains such as healthcare and criminal justice.  There is a growing recognition of the need to evaluate the causal impact of deploying these systems on downstream outcomes, such as patient survival or crime recidivism. Randomized control trials (RCTs) can provide high-quality evidence on the impact of a deployed model, but they run into a challenge: it is often infeasible to run repeated trials when models are updated or retrained to improve predictive performance.  In this work, we present a partial-identification approach to using prior RCT data to construct bounds on the causal effect of a new model. The core innovation in our approach is to leverage assumptions relating fine-grained predictive accuracy to downstream outcomes.   We do so via two monotonicity assumptions: first, on individual-level `counterfactual correctness' (all else being equal, a correct prediction leads to non-inferior outcomes); and second, on the relation between subgroup predictive performance and outcomes, interpretable as an assumption regarding trust in model outputs. We demonstrate our method with a simulation study, illustrating how incorporating this information can lead to more informative bounds compared to prior work.
\end{abstract}

\section{Introduction}\label{sec:intro}
Predictive machine learning (ML) models are being deployed more and more into high-risk decision-making domains, such as recidivism prediction for criminal justice or diagnostic aids for physicians \citep{Travaini2022MachineLAA, esteva2017dermatologist}.  In these domains, there is a growing recognition of the need to treat artificial intelligence (AI) systems as interventions, going beyond predictive accuracy to reason about their impact on decision-makers and downstream outcomes \citep{ben-michael2024-54,liu2025-ee,joshi2025-42,van-amsterdam2026-53}.

In some settings, it is feasible to evaluate the impact of predictive systems via randomized control trials (RCTs), which are increasingly being used to assess ML/AI interventions in medicine \citep{shaukat2022-23, yao2021-7c, upton2024-aa, gohil2024-10, Chen2025TheFTA, Shaban2025EmpoweringBCA} and criminal justice \citep{imai2023experimental}. Such trials are often designed to evaluate the impact of deploying an ML/AI system as a discrete intervention, including impact on decision-makers and downstream outcomes, keeping in mind that the impact of predictive models is mediated by human behavior \citep{raji2025-6c}.

However, there are limits to the approach of viewing distinct ML models as distinct interventions. For example, it is typically infeasible to run a new RCT for every proposed update to a model. This infeasibility creates a challenge:  If a randomized trial has been conducted for one model, and a new model (with higher predictive accuracy) is developed afterwards, what can we say about the expected impact of the latter model?

Prior work has made some steps towards addressing this gap. For instance, \citet{nercessian2025data} propose an adaptive trial design under the assumption that downstream outcomes do not differ between models whose predictions agree. \citet{chen2025} propose a causal framework for partial identification that weakens this assumption, by allowing for differences in outcomes at the individual level, even if predictions are identical. They motivate this possibility from the perspective of user trust: For instance, a model that flags every patient as high-risk is likely to influence clinical actions less than a model with higher precision, even on the cases where they agree.  To model this aspect of performance, they assume that residual variation in outcomes is driven by an observable per-model measure of global predictive performance (e.g., precision).

However, both of these frameworks fail to incorporate the accuracy of individual-level model predictions---in many scenarios, we might expect that correct predictions are no worse than suboptimal predictions. We refer to this notion as \textbf{counterfactual correctness}: Whether or not the predictions of a new model would have been correct on individual subjects in the original RCT.  With that notion in mind, we make the following contributions in this work: 
\begin{enumerate*}[label=(\roman*)]
\item We extend the causal framework of~\citet{chen2025} to incorporate known `correct' prediction labels: if a new model produces the correct prediction, we assume that outcomes are no worse than those of models that made incorrect predictions.
\item We further extend this framework to consider conditional subgroup predictive performance as influencing outcomes, rather than a global measure of performance, reflecting the idea that users are likely to develop mental heuristics of where models perform well or poorly.
\item We provide population-level upper and lower bounds under these assumptions, and further provide a consistent and asymptotically normal estimator of those bounds.  Crucially, unlike~\citet{chen2025}, we do so without assuming that model performance measures are known a priori, and allow for the fact that e.g., sub-group performance needs to be estimated from data.
\item In a semi-synthetic simulation study, we illustrate that incorporating this information can lead to more informative bounds compared to prior work.
\end{enumerate*}

Our framework is most useful in settings where the ``correct'' predictions (which we often refer to as the ``true labels'') are observed for all individuals in the RCT, and explicitly assumes that predictions themselves do not influence the true label.  These conditions may co-occur when the prediction task is fundamentally diagnostic, with the possibility of retrospective (even if expensive) assessment of the correct prediction for each individual. Our framework can be extended to settings where the true label is not always observed, under specific mechanisms of missingness. We discuss one of these examples in greater detail in~\cref{missingness}. We also discuss distinctions between our work and other related areas in~\cref{app:related_work}.

\section{Setup \& Causal Model}

\textbf{Notation:}
We use upper case letters $X$ to denote random variables, lower case $x$ to denote realizations, and calligraphic font $\mathcal X$ to denote the set of possible values. $\mathbf{1}\{E\}$ is the indicator function of an event $E$, equal to $1$ if the event occurs and $0$ otherwise. Unless otherwise noted, $\|\cdot\|$ refers to the $L^2(P)$ norm and $|\cdot|$ refers to the absolute value. 

\textbf{Problem Setup}: We model the data-generating process as the directed acyclic graph (DAG) in \cref{fig:dag_a}. A directed edge $A\rightarrow B$ means that $A$ may cause $B$. No edge means $A$ does not cause $B$. Each node is a random variable. $D \in \{1, \ldots, d\}$ are the arms of the RCT (e.g. with and without ML assistance). $\Hclass$ is the set of models being trialed, with each model corresponding to a specific arm. The model we would like to evaluate (denoted $\pi_e$) is untrialed, so $\pi_e \notin \Hclass$.
$X\in \mathcal{X}$ are features used by the prediction model, $Z \in \mathcal{Z}$ is the ``true label'', and $\hat{Z}\in \mathcal{Z}$ is the predicted label. The performance of the model $\Pi$ for individuals with covariates $X$ is a real number $M \in \mathcal{M} \subseteq \mathbb{R}$, where $\mathcal{M}$ denotes the range of achievable metric values. $A \in \mathcal{A}$ is the action taken by the decision-maker, which we do not model explicitly and generally allow to be unobserved. Finally, $Y\in\mathbb{R}$ is the downstream outcome, influenced by $A$ and $X$. We employ potential outcomes notation \citep{richardson2013single} where $Y(\Pi=\pi)$ represents the counterfactual value of $Y$ that would be observed had we set $\Pi = \pi$.
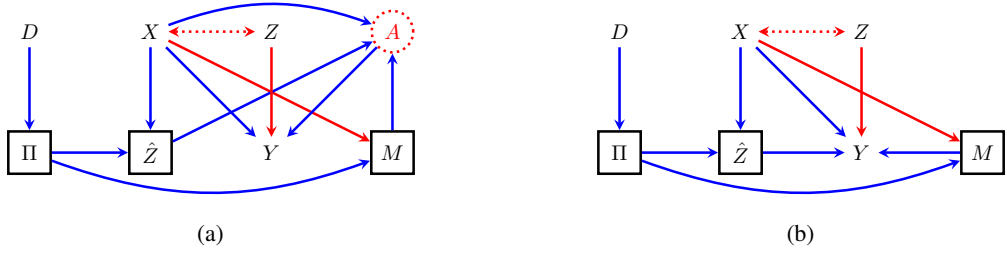
\begin{figure*}[t]
\centering
\begin{subfigure}[b]{0.45\textwidth}
\centering
\scalebox{0.80}{
\begin{tikzpicture}[>=stealth, node distance=2cm]
\tikzstyle{square} = [draw, very thick, minimum size=7mm, inner sep=3pt]
\begin{scope}
\path[->, very thick]
node[draw, square] (p) {$\Pi$}
node[above of=p] (d) {$D$}
node[right of=p, draw, square] (a) {$\hat{Z}$}
node[right of=a] (y) {$Y$}
node[right of=y, draw, square] (m) {$M$}
node[above of=a] (x) {$X$}
node[above of=y ] (z) {$Z$}
node[right of=z, red, draw, circle, dotted] (q) {$A$}
(p) edge[blue, bend right=20] (m)
(x) edge[red] (m)
(x) edge[blue, bend left=20] (q)
(p) edge[blue] (a)
(a) edge[blue] (q)
(q) edge[blue] (y)
(m) edge[blue] (q)
(d) edge[blue] (p)
(x) edge[blue] (a)
(x) edge[red, dotted, <->] (z)
(x) edge[blue] (y)
(z) edge[red] (y)
;
\end{scope}
\end{tikzpicture}
}
\caption{}
\label{fig:dag_a}
\end{subfigure}
\begin{subfigure}[b]{0.45\textwidth}
\centering
\scalebox{0.80}{
\begin{tikzpicture}[>=stealth, node distance=2cm]
\tikzstyle{square} = [draw, very thick, minimum size=7mm, inner sep=3pt]
\begin{scope}
\path[->, very thick]
node[draw, square] (p) {$\Pi$}
node[above of=p] (d) {$D$}
node[right of=p, draw, square] (a) {$\hat{Z}$}
node[right of=a] (y) {$Y$}
node[right of=y, draw, square] (m) {$M$}
node[above of=a] (x) {$X$}
node[above of=y ] (z) {$Z$}
(p) edge[blue, bend right=20] (m)
(x) edge[red] (m)
(p) edge[blue] (a)
(d) edge[blue] (p)
(x) edge[blue] (a)
(x) edge[red, dotted, <->] (z)
(x) edge[blue] (y)
(z) edge[red] (y)
(a) edge[blue] (y)
(m) edge[blue] (y)
;
\end{scope}
\end{tikzpicture}
}
\caption{}
\label{fig:dag_b}
\end{subfigure}

\caption{\cref{fig:dag_a} is the causal DAG representing the structure of an RCT, where $X$ denotes features, and $D$ the arm of the RCT, which determines the model $\Pi$ used for prediction. The model $\Pi$ outputs predictions $\hat{Z}$ as a deterministic function of $X$, where $Z$ is the true label.  The dotted bi-directed edge between $X, Z$ represents the possibility that $X$ and $Z$ share an unobserved common cause.  $M$ is a subgroup performance metric of $\Pi$, $A$ is the decision-maker's actions, and $Y$ is the downstream outcome. A solid square indicates the node is deterministic given its parents, and a dotted circle indicates the node is hidden. In \cref{fig:dag_b}, we show the same graph after applying the latent projection operator on $A$, which only considers the observed variables in the RCT. Edges also present in prior work by \citet{chen2025} are shown in blue (re-interpreting~\citet{chen2025} as having an implicit action $A$), while edges novel to our work are shown in red.}
\label{fig:RCT_DAG}
\end{figure*}

\begin{example}[Diagnosis Recommendation System]\label{example:diagnosis}
To build intuition for our notation in the context of a concrete example, consider a machine learning model ($\pi$) designed to help radiologists make diagnoses. An RCT to determine whether model deployment improves patient outcomes might randomly assign physicians to either the control arm (no model, $D = 1$) or the treatment arm (model assistance, $D = 2$). The model takes in a chest x-ray, as well as other general patient information ($X$) as input, and returns a predicted diagnosis ($\hat{Z}$), such as pneumonia, pneumothorax, or healthy. This prediction may or may not agree with the true diagnosis ($Z$), which may be recorded later in time (e.g., at discharge).
Over time, physicians will learn about the performance of the deployed models ($M$) on different types of patients. This learning can be captured quantitatively by calculating a per-subgroup metric on each of the models such as the sensitivity/specificity.
Together, the knowledge of these variables will affect the physician's actions ($A$), in turn affecting some outcome of interest such as survival rate ($Y$). These actions are left unmodeled, and in general could represent a complex set of actions (e.g., specific language used in a radiology report).  Our goal is to evaluate expected patient outcomes $Y$ under the deployment of a new model $\pi_e$ not included in the trial.
\end{example}

We now formally introduce our causal framework, shown graphically in \cref{fig:RCT_DAG}.

\begin{restatable}[Data Generating Process]{assumption}{DataGeneratingProcess}\label{asmp:dgp}
We assume that $(D, \Pi, X, Z, \hat{Z}, M, Y)$ are generated according to the following structural causal model (SCM) \cite{pearl2009causality}, where all noise variables $\epsilon$ are mutually independent:
\begin{align*}
    D &=f_D(\epsilon_D) & \Pi &= f_{\Pi}(D) \\
    X, Z &= f_{XZ}(\epsilon_{XZ}) & \hat{Z} &= {\Pi(X)}\\
     Y &= f_Y(X, Z,M,\hat{Z},  \epsilon_Y) & M &= f_M(X, \Pi)
\end{align*}
\end{restatable}

Note that our SCM allows for the possibility of various causal relationships between $X$ and $Z$ (e.g., where $X$ causes $Z$, vice versa, or where both are related by a common cause). We make a few additional observations:  First, our causal model (reflected in~\cref{fig:dag_b}) does not include $A$, but can be seen as a projection of the directed acyclic graph (DAG) shown in~\cref{fig:dag_a}, where $A$ mediates the effect of $\hat{Z}$ on $Y$. This mediation reflects our natural intuition that the decision-maker, not the model, is the final conduit of change, and serves as conceptual motivation for several of our assumptions. Second, we allow for $X, Z$ to be arbitrarily related (e.g., by a common cause), but we assume that $\hat{Z}$ does not have any causal effect on $Z$.  Finally, several elements are \textit{deterministic} given their parents, $\hat{Z}$ and $M$ in particular.

\textbf{Our Goal:} Consider a model $\pi_e\not\in \Hclass$. Our goal is to estimate the expectation of $Y$ for $\pi_e$ had it been in the RCT.
\begin{equation}\label{eq:our_goal}
\mathbb{E}[Y(\pi_e)] = \mathbb{E}[Y(\hat{Z}=\pi_e(X), M=f_M(X, \pi_e))].
\end{equation}
This notation is used in conjunction with our SCM notation defined in \cref{asmp:dgp}. In particular, the right-hand side of~\cref{eq:our_goal} can be written as 
$\mathbb{E}\left[f_Y(X, Z, f_M(X, \pi_e), \pi_e(X), \epsilon_Y)\right]$,\footnote{Throughout, we write $f_Y$ with its arguments in the order given in~\cref{asmp:dgp}, namely $(X, Z, M, \hat{Z}, \epsilon_Y)$.} where we set $\hat{Z} = \pi_e(X), M = f_M(X, \pi_e)$ and evaluate the expectation. We further illustrate this estimand with an example. 

\begin{example}[continues=example:diagnosis]\label{example:diagnosis2}
    Suppose we have data from a previous RCT on radiology models. The data records the values of $Z$ and $\hat{Z}$ for all models on all patients. For each patient, we observe the value of $Y$ under the deployment of a single model. In this context, given a new model, our goal would be to estimate its effect on patient survival ($Y$) without running a new RCT, leveraging only knowledge of the new model's predictions ($\hat{Z}$) on the data from that previous RCT.
\end{example}

\section{Structural Assumptions and Falsification Tests}
To draw conclusions about the downstream outcomes of an untrialed model, we need to make additional \textit{structural} assumptions on the data generating process, beyond those implied by the causal graph. First, we formalize our notion of counterfactual correctness.
\begin{restatable}[Counterfactual Correctness]{assumption}{CorrectActionsCauseNoHarm}\label{assumptions:correct}
   We assume that correct predictions cause no harm. That is, for all $m,m' \in \cM$ and $z, z'\in \mathcal{Z}$ with $z' \ne z$,
   \begin{align*}
     & Y(\hat{Z} = z, M=m,Z = z) \geq Y(\hat{Z} = z',M=m',Z = z)
   \end{align*}
\end{restatable}   
This assumption states that, when predictions are correct, we don't expect them to harm outcomes.\footnote{This inequality can be written equivalently under~\cref{asmp:dgp} as requiring that for all values of $z, x, m, \epsilon_Y$ and $z' \neq z$, $f_Y( x,z, m, z, \epsilon_Y) \geq f_Y(x, z , m', z', \epsilon_Y)$, where potential outcomes are taken to be pointwise functions of the noise variables.} While intuitive, this assumption may fail if e.g., physicians act contrary to predictions.  That said, this assumption has testable implications, which can be assessed using RCT data.

\begin{restatable}[Falsification of \cref{assumptions:correct}]{proposition}{FalsificationAssumptioncorrect}\label{prop:falsification_correct}
Let $\pi_c, \pi_w\in \Hclass$ be two trialed models (one of which will be ``correct'' and the other ``wrong''), and let $\mathcal{X}_{c\neq w} \coloneqq \{x \in \mathcal{X} \mid \pi_c(x) \ne \pi_w(x)\}$
be the set of covariates on which their predictions disagree. Consider the subpopulation where $X \in \mathcal{X}_{c \neq w}$ and $Z = \pi_c(X)$ (i.e., where $\pi_c$ is correct and $\pi_w$ incorrect).  Then, the observation that outcomes are better under $\pi_w$,\footnote{All falsification tests require that the relevant conditioning sets have non-zero probability of occurring.  For this case, if $D$ is uniformly distributed and independent of $X, Z$, it suffices that $P\big(X \in \mathcal{X}_{c\neq w},\, Z = \pi_c(X)\big) > 0$.}
\begin{align*}
    &\mathbb{E}[Y \mid X\in\mathcal{X}_{c \neq w},\, \Pi=\pi_c,\, Z=\pi_c(X)] \\
    & < \mathbb{E}[Y \mid X\in\mathcal{X}_{c \neq w},\, \Pi=\pi_w,\, Z=\pi_c(X)]
\end{align*}
falsifies \cref{assumptions:correct}.
\end{restatable}
This falsification test allows practitioners to use a simple hypothesis test (comparing two means) to assess if~\cref{assumptions:correct} is contradicted by RCT data. This falsification test, and others we propose, cannot conclusively demonstrate that an assumption is true, only that it is false, and passing the falsification tests does not validate the assumptions: The assumption may hold on RCT data but not more broadly across all future, untrialed models, and  our assumptions are defined pointwise on potential outcomes, while these falsifications test for inequalities in expectation. Nonetheless, these tests have the appeal of being simple to implement. Proofs for all results, including the proof for the falsification tests, are provided in the \hyperref[app:overview]{Appendix}.

Next, we introduce three additional assumptions adapted from \citet{chen2025}, but modified to better align with real-world conditions and our new data generating process. While the previous assumption dealt with outcomes under different predictions, the next assumptions focus on outcomes under the same prediction. We first introduce a notion of performance monotonicity, which was also explored in \citet{chen2025}.
\begin{restatable}[Conditional Performance monotonicity]{assumption}{PerformanceMonotonicity}\label{assumption:monotonicity}
Potential outcomes are monotonic in model performance relative to correctness:  If 
$m_i< m_j$ for $m_i, m_j\in\cM$, then for all $z\in\mathcal{Z}$ and all $\hat{z}\in\mathcal{Z}$ with $\hat{z}\ne z$,
\begin{align*}
Y(\hat{Z}= z, M = m_i, Z = z)&\leq Y(\hat{Z}={z}, M=m_j, Z = z)\\
Y(\hat{Z}=\hat{z}, M = m_i, Z = z)&\geq Y(\hat{Z}=\hat{z}, M=m_j, Z = z)
\end{align*}
\end{restatable}
This assumption states that, conditioned on correctness, potential outcomes are monotonic with respect to model performance. When the prediction is correct ($\hat{z} = z$), potential outcomes are monotonically increasing with respect to performance, and vice versa when the prediction is incorrect ($\hat{z}\ne z$).
Conceptually, this assumption reflects the idea that decision-makers are more likely to trust higher performing models, and follow through on their predictions. When the prediction is correct, this increased trust leads to better outcomes, but when predictions are incorrect, this reliance leads to worse outcomes. \citet{chen2025} considered a simplified form of this assumption, where improved performance always leads to better outcomes, which is likely invalid outside a narrow range of settings.  They motivated the assumption in a binary alerting setting, where false alerts are assumed to not be harmful at the individual level, and $\hat{z}=0$ does nothing to affect decision-making.  We similarly present a falsification test for~\cref{assumption:monotonicity}. This falsification test only covers the case when $\hat{Z} = z$; we provide an analogous test for the case when $\hat{Z} \ne z$ in \hyperref[prop:falsification_monotonicity2]{the Appendix}.
\begin{restatable}[Falsification of \cref{assumption:monotonicity}]{proposition}{FalsificationAssumptionThreeOne}\label{prop:falsification_monotonicity}
    Let $\pi_+,\pi_- \in\Hclass$ be two trialed models. 
    Consider a sub-population with $\mathcal{X}_{better}\subseteq\mathcal{X}$ such that for all $x\in \mathcal{X}_{better}$, $f_M(x,\pi_+) > f_M( x, \pi_-)$, and a subpopulation with covariates $\mathcal{X}_{agree} \subseteq \mathcal{X}$ such that for all $x \in \mathcal{X}_{agree}$, $\pi_+(x)= \pi_-(x)$. Then, for any $\mathcal{X}_{s} =  \mathcal{X}_{better}\cap \mathcal{X}_{agree}$
    \begin{align*}
	& \mathbb{E}[Y|X\in\mathcal{X}_{s}, \Pi = \pi_+, Z=\pi_+(X)] \\
        &< \mathbb{E}[Y|X\in\mathcal{X}_{s}, \Pi = \pi_-, Z=\pi_+(X)]
    \end{align*}
    falsifies \cref{assumption:monotonicity}.
 \end{restatable}

Furthermore, we introduce the concept of a \enquote{neutral prediction}, and assume that performance of the model does not affect the outcome in this case. 
\begin{restatable}[Neutral prediction]{assumption}{NeutralAction}\label{assumption:neutral}
There exists $\hat{z}_0 \in \mathcal{Z}$ such that for all $m_i, m_j \in \cM$ and $z \in \mathcal{Z}$
\begin{equation*}
Y(\hat{Z}=\hat{z}_0, M=m_i, Z=z) = Y(\hat{Z}=\hat{z}_0, M=m_j, Z=z)
\end{equation*}
\end{restatable}
We note that this assumption is ``optional'', in the sense that the method still produces bounds without it, at the cost of those bounds being potentially looser. This assumption is largely the same as Assumption 3.2 of \citet{chen2025}, and plays the role of linking the ``control arm'' of a trial to evaluation of a newer model, and can be seen as a special case of \cref{assumption:monotonicity} in which the nonstrict inequality becomes an equality. In our running example, such a ``prediction'' could be choosing to defer to a radiologist, and the control arm could be viewed as an ``always defer'' policy. The intuition for this assumption is that when the model takes the neutral prediction, the physician behaves as if there were no model at all, and hence model performance does not impact actions. We also provide a falsification test for this assumption. 
\begin{restatable}[Falsification of \cref{assumption:neutral}]{proposition}{FalsificationAssumptionThreeTwo}\label{prop:neutral}
   Let $\pi_1, \pi_2 \in \Hclass$ be two trialed models. Letting $\mathcal{X}_{\hat{z}_0} := \{x\in\mathcal{X}|\pi_1(x)=\pi_2(x)=\hat{z}_0\}$, the observation that for some $z\in\mathcal{Z}$:
\begin{align*}
&\mathbb{E}\left[Y|X\in \mathcal{X}_{\hat{z}_0}, \Pi = \pi_1, Z=z \right] \\
& \ne  \mathbb{E}\left[Y|X\in \mathcal{X}_{\hat{z}_0}, \Pi = \pi_2, Z=z \right]
\end{align*}
falsifies \cref{assumption:neutral}.
\end{restatable}
Finally, similar to \citet{chen2025}, we require that $Y$ itself be bounded by known maximum and minimum values, which can be trivially falsified by observing values of $Y$ outside of those bounds.
\begin{restatable}[Bounded outcomes]{assumption}{BoundedOutcomes}\label{assumption:bounded}
    There exist $Y_{min}, Y_{max} \in \mathbb{R} \text{ such that } Y_{min} \leq Y \leq Y_{max}$
\end{restatable}

\section{Partial Identification Bounds}

In this section, we introduce our main result, which provides upper and lower bounds on the expected outcome under deployment of a new model (\cref{eq:our_goal}).  Before we can introduce this result,  we first introduce some notation for partitions of the model and covariate space in our RCT. These partitions will be used in the derivation of our bounds.

\begin{restatable}[Partitions]{definition}{PartitionsDef}\label{def:partitions}
	   We use $\pi_i$ as a model in the model space and $\pi_e$ as the untrialed model. 
       For each value of $x$, we define the set of trialed models that make the same prediction as our untrialed model, $\pi_e(x)$\footnote{As in \citet{chen2025}, these sets should be written with $\pi_e$ as an argument, but we omit this for brevity.  Note that $\Sagree$ stands for ``same'' and $\Scorrect$ stands for ``correct''.}. 
	\begin{align*}
		\Sagree(x) & \coloneqq \{\pi_i \in \Hclass \mid \pi_i(x) = \pi_e(x)\}                
	\end{align*}
   We further partition $\Sagree(x)$ into models whose predictive performances imply outcomes that are no better / no worse than those of $\pi_e$
   \begin{align*}
      		\Sle(x,z) & \coloneqq \{\pi_i \in \Sagree \mid f_M(x, \pi_i) \leq f_M(x,\pi_e), \pi_e(x)=z\} \\
            &\quad\cup\{\pi_i \in \Sagree \mid f_M(x, \pi_i) \geq f_M(x,\pi_e), \pi_e(x)\ne z\}\\
		\Sge(x,z) & \coloneqq \{\pi_i \in \Sagree \mid f_M(x, \pi_i) \geq f_M(x,\pi_e), \pi_e(x) = z\} \\
            &\quad\cup \{\pi_i \in \Sagree \mid f_M(x, \pi_i) \leq f_M(x,\pi_e), \pi_e(x)\ne z\}
   \end{align*}
	We further define subsets of $\Sle, \Sge$ containing only the models whose performance is closest to that of $\pi_e$.
	\begin{align*}
		\Sletight(x,z) & \coloneqq \argmin_{\pi_i \in \Sle(x,z)} |f_M(x, \pi_i)-f_M(x, \pi_e)|\\
      \Sgetight(x,z) & \coloneqq \argmin_{\pi_i \in \Sge(x,z)} |f_M(x, \pi_i)-f_M(x, \pi_e)|
	\end{align*}
	Finally, for some value of $x$ and an optimal prediction $z$, we define the set of all trialed models that give the correct prediction and its complement:
	\begin{equation*}\Scorrect(x, z) \coloneqq \{\pi_i \in \Hclass \mid \pi_i(x) = z\}\
   \end{equation*}
    \begin{equation*}
         \Sdisagree(x,z) \coloneqq \{\pi_i \in \Hclass \mid \pi_i(x) \ne z\}
    \end{equation*}
   We note that any of the defined sets can possibly be empty. 
\end{restatable}
\begin{example}[continues=example:diagnosis2]\label{example:diagnosis3}
Suppose in our radiology RCT, we have three models: $\pi_1, \pi_2, \pi_3$. The precision of these models are 0.1, 0.8, and 0.9, respectively, at a specific covariate value $x$. For a specific patient with covariates $x$, whose true diagnosis, $z$, is pneumonia, we want to evaluate an untrialed model, $\pi_e$, with precision 0.7. Suppose $\pi_2$, $\pi_3$ and $\pi_e$ all suggest the diagnosis of pneumonia, while $\pi_1$ suggests bronchitis for this particular patient.
Under this scenario, $\Sagree(x) = \{\pi_2, \pi_3\}$, and $\pi_e$ was optimal, and thus since both $\pi_2$ and $\pi_3$ have greater performance than $\pi_e$, $\Sge(x,z) = \{\pi_2, \pi_3\}$ and $\Sle(x,z) = \varnothing$. If $\pi_e$ were not optimal the sets would be reversed. Further, $\Sgetight(x,z) = \{\pi_2\}$, since its performance metric is closest to that of $\pi_e$.  Finally, $\Scorrect(x,z) = \{\pi_2, \pi_3\}, \Sdisagree(x,z) = \{\pi_1\}$. 

\end{example}
Next, under the notation established in \cref{def:partitions}, we introduce a \emph{branch selector} that assigns each $(x,z)$ to the single assumption used to bound the outcome.
\begin{restatable}[Branch selector]{definition}{IndicatorDef}\label{def:indicators}
For each $(x,z)$ we define lower- and upper-bound branch selectors
$\bsL,\bsU:\mathcal{X}\times\mathcal{Z}\to\{\lneu,\lmon,\lcor,\ldef\}$ (neutral, monotonic, correct, default),  evaluated by first match:
    \begin{align*}
   \bsL(x,z) &=
     \begin{cases}
       \lneu & \pi_e(x)= \hat{z}_0,\ \Sagree(x)\ne\varnothing\\
       \lmon & \pi_e(x)\ne \hat{z}_0,\ \Sle(x,z)\ne\varnothing\\
       \lcor & \pi_e(x) = z,\ \Sle(x,z)=\varnothing,\ \Sdisagree(x,z)\ne\varnothing\\
       \ldef & \text{otherwise,}
     \end{cases}\\[4pt]
   \bsU(x,z) &=
     \begin{cases}
       \lneu & \pi_e(x)= \hat{z}_0,\ \Sagree(x)\ne\varnothing\\
       \lmon & \pi_e(x)\ne \hat{z}_0,\ \Sge(x,z)\ne\varnothing\\
       \lcor & \pi_e(x) \ne z,\ \Sge(x,z)=\varnothing,\ \Scorrect(x,z)\ne\varnothing\\
       \ldef & \text{otherwise.}
     \end{cases}
\end{align*}
Because each selector is a function evaluated by first match, it returns exactly one
value, so the induced indicators $\mathbf{1}\{\bsL=k\}, \mathbf{1}\{\bsU=k\}$ partition the space of $\mathcal{X} \times \mathcal{Z}$. 
\end{restatable}
With this selector, we can now produce our main theorem, which provides bounds on the expected outcome of deploying an untrialed model in our novel RCT setting.
\begin{restatable}[Bounds]{theorem}{BoundsTheorem}\label{def:bounds}
    Under~\cref{asmp:dgp,assumptions:correct,assumption:monotonicity,assumption:neutral,assumption:bounded},
   we have the following bounds on the impact of our untrialed model $\pi_e$:
   \begin{equation*}L(\pi_e) \leq \mathbb{E}[Y(\hat{Z}=\pi_e(X), M=f_M( X, \pi_e))] \leq U(\pi_e)\end{equation*}
   Where: 
   \begin{align}
	L(\pi_e) 
    &= \mathbb{E}[\bLneu \mathbb{E}[Y|X,\Pi \in \Sagree(X),Z]\label{eq:lower_bound_neutral}\\
	         & \quad + \bLmon\mathbb{E}[Y|X,\Pi \in \Sletight(X,Z),Z] \nonumber\\
	         & \quad + \bLcor\max_{i\in \Sdisagree(X,Z)}\mathbb{E}[Y|X,\Pi = i,Z] \nonumber\\
	         & \quad + \bLdef Y_{min} \label{eq:lower_bound_default}
	].
\end{align}
\begin{align}
	U(\pi_e) 
    &= \mathbb{E}[\bUneu \mathbb{E}[Y|X,\Pi \in \Sagree(X), Z]\nonumber\\
	         & \quad + \bUmon\mathbb{E}[Y|X,\Pi \in \Sgetight(X,Z), Z] \nonumber\\
	         & \quad + \bUcor\min_{i\in \Scorrect(X,Z)}\mathbb{E}[Y|X,\Pi = i, Z] \nonumber\\
	         & \quad + \bUdef
             Y_{max} \label{eq:upper_bound_default}
	].
\end{align}
\end{restatable}

Note that each value of $X, Z$ is covered by exactly one ``branch'', and the bound is constructed via pointwise upper and lower bounds on the resulting outcome under $\pi_e$. We illustrate the intuition for the lower bound here, with symmetric logic for the upper bound: 
\begin{enumerate*}[label=(\roman*)]
\item When $\bsL=\lneu$, our untrialed model takes the neutral prediction, and there are models in the RCT that agree. By \cref{assumption:neutral}, the downstream outcome will be equal regardless of performance across these neutral-prediction-taking models.
\item When $\bsL=\lmon$, our untrialed model does not take the neutral prediction, and there are models in the RCT that agree with it and have worse or equal performance. By \cref{assumption:monotonicity}, the downstream outcome of our untrialed model will be no worse than the best performing model in this set.
\item When $\bsL=\lcor$, our untrialed model takes the optimal prediction, and there are no models in the RCT that agree with it and have worse or equal performance. We also know that $\bsL\ne\lneu$ and  $\Sdisagree(x,z)$, or the set of incorrect models, is non-empty. By \cref{assumptions:correct}, the downstream outcome of our untrialed model will be no worse than any of these incorrect models, since they must have taken a suboptimal prediction.
\item Finally, when $\bsL=\ldef$, none of the prior cases apply, leaving us with the maximally conservative choice of $Y_{min}$.
\end{enumerate*}
We provide a flowchart for the lower bound in \cref{fig:lower-bound-flowchart}, and note that the logic is similar (with slightly different relevant sets) for the upper bound.

\paragraph{Tightness of Bounds} 
\citet{chen2025} construct similar bounds under a different set of assumptions, and show that they are tight, in the sense that there exists a pair of SCMs that are allowable under their assumptions, and for which the upper and lower bounds match the causal effect.  We do not claim that our bounds are tight, and provide intuition here as to why constructing tight bounds is not feasible under our counterfactual correctness assumption.  Focusing on $L(\pi_e)$, \cref{assumptions:correct} is stated pointwise on potential outcomes, while our bounds instead maximize over the conditional expectations. To illustrate, consider an RCT in which the untrialed model takes the optimal prediction. The data contains $\pi_1(x), \pi_2(x)$ that take incorrect predictions, while $\pi_e(x)$ takes the optimal prediction. Both trialed models have the same conditional expectation $E[Y|X, \Pi, Z] = 0.75$, so our lower-bound for $\pi_e(x)$ would be $0.75$.  A (likely tighter) lower bound would instead be $E[\max_{i=1,2} Y(\pi_i) \mid X, Z]$, but implementing such a bound is not feasible given that we only observe outcomes for a single trialed model for any individual.

\section{Estimation}

The bounds in~\cref{def:bounds} have several values that must be estimated. For example, \cref{eq:lower_bound_neutral} requires $\mathbb{E}[Y \mid X, \Pi \in \Sagree(X), Z]$ for all $x$ such that $\bLneu=1$. If $\Sagree(X)$ contains more than one model, it is impossible to have data on $Y$ under both. 
We address this with consistent and asymptotically normal estimators, and asymptotically valid confidence intervals for the lower and upper bounds.
Constructing these estimators requires addressing three challenges that do not arise in \citet{chen2025}:
\begin{enumerate*}[label=(\roman*)]
   \item estimating $\mathbb{E}[Y\mid X,Z,\pi]$ via nuisance models, 
   \item handling non-smooth max/min operators, and 
   \item estimating model performance from finite data rather than assuming oracle knowledge.  
\end{enumerate*}
We address these challenges with an augmented inverse propensity weighted estimator for the bounds, and we prove its consistency and asymptotic normality.

\begin{definition}[Nuisance Functions]\label{def:est_defs}
   We define the following nuisance functions and estimators:
   \begin{align}
   \mu_j(x, z) &:= \mathbb{E}[Y|x, \Pi = \pi_j, z] \label{def:mu}\\
\hat{d}_L(x,z)&\;:=\;\arg\max_{j\in \Sdisagree(x,z)} \widehat{\mu}_j(x,z)\nonumber\\
\hat{d}_U(x,z)&\;:=\;\arg\min_{j\in \Scorrect(x,z)} \widehat{\mu}_j(x,z)\label{eq:d_hat}\\
\widehat{\lambda}_{j}(Y, \Pi, x,z)&\;:=\;\frac{\mathbf{1}\{\Pi=j\}}{P(\Pi= j\mid X=x)}\bigl(Y- \widehat{\mu}_j(x,z)\bigr)\label{eq:lambda_hat}\\
	\widehat	\Sletight(x,z) & \coloneqq \argmin_{\pi \in \Slehat(x,z)} |\hat{f}_M(x, \pi)-\hat{f}_M(x, \pi_e)|\nonumber\\
      \Sgetighthat(x,z) & \coloneqq \argmin_{\pi \in \Sgehat(x,z)} |\hat{f}_M(x, \pi)-\hat{f}_M(x, \pi_e)|\label{eq:S_tilde_leq_hat}\\
{\varphi}_L\bigl(x,z; \hat{d}_L\bigr)&\;:=\; \widehat{\mu}_{\hat{d}_L(x,z)}(x,z) + \widehat{\lambda}_{\hat{d}_L(x,z)}(Y, \Pi, x,z)\nonumber\\
{\varphi}_U\bigl(x,z; \hat{d}_U\bigr)&\;:=\; \widehat{\mu}_{\hat{d}_U(x,z)}(x,z) + \widehat{\lambda}_{\hat{d}_U(x,z)}(Y, \Pi, x,z)\label{eq:phi_hat},
\end{align}
Where $\Slehat(x,z),\Sgehat(x,z)$ are defined as in \cref{def:partitions}, with $\hat{f}$ in place of $f$, and $\widehat{\mu}(x,z)$ is the plug-in estimator of ${\mu}$.
\end{definition}
For our estimators to be consistent and asymptotically normal, we require the following assumptions:

\begin{assumption}[Known Propensities]\label{assump:known_prop}
The value of $P(\Pi\in\cS\mid X)$ is known for all $\cS\subseteq\Hclass$, and is strictly positive for all $\cS, X$.
\end{assumption}
Given our RCT setting, this assumption will be satisfied by design. A similar assumption is implicit in \citet{chen2025}, but we make it explicit here.  We require additional conditions to handle non-smooth max/min operators, which we describe next.
\begin{assumption}[Margin Conditions]\label{assump:margin}
Under \cref{def:est_defs}, there exists $\alpha > 0$ such that for any $t\geq 0$, we have
\begin{equation*}
P\left[\min_{j \ne d(X,Z)}\left|\mu_{d(X,Z)}(X,Z)-\mu_{j}(X,Z)\right|\leq t\right] \leq t^{\alpha}
\end{equation*}
for both $d_L(x,z), d_U(x,z)$, as defined in \cref{eq:d_hat}, and $\mu$ defined in \cref{def:mu}. Furthermore, there exists $\beta > 0$ such that for any $t\geq 0$, we have
\begin{equation*}
P\left[\min_{i\ne j}\left|f_M(X, \pi_j)-f_M(X,\pi_i)\right|\leq t\right]\leq t^\beta
\end{equation*}
where $f_M$ is defined in \cref{thmt@@DataGeneratingProcess}, and $\pi_i,\pi_j\in\Hclass \cup\{\pi_e\}$.
\end{assumption}
These assumptions state that there is sufficient separation between the downstream outcomes of the best two models, the worst two models, and between the performance of any two models, untrialed or trialed. 
\begin{assumption}[Convergence of Nuisance Functions]\label{asmp:convergence_nuisance}
  The nuisance functions $\widehat{\mu}_j(X,Z), \widehat{f}_M(x; \pi_i)$, defined in \cref{def:est_defs} converge at rates $o_p(n^{-\frac{1}{2(1+\alpha)}})$ and $o_p(n^{-\frac{1}{2\beta}})$, respectively,  for all $\pi_i\in \Hclass, x \in \mathcal{X}$, where $\alpha$ and $\beta$ are the same constants as in the Margin Conditions (\cref{assump:margin}). That is:
   \begin{equation*}
      \|\widehat{\mu}_j(X,Z) - \mu_j(X,Z)\| = o_p(n^{-\frac{1}{2(1+\alpha)}})
   \end{equation*}
   \begin{equation*}
      \|\widehat{f}_M(x; \pi_i) - f_M(x, \pi_i)\| = o_p(n^{-\frac{1}{2\beta}})
   \end{equation*}
   where $\pi_i\in\Hclass\cup \{\pi_e\}$
\end{assumption}
We note that $o_p(n^{-1/4})$ convergence rates have been demonstrated for a variety of modern machine learning methods under e.g., sparsity constraints \citep{chernozhukov2018double}, making~\cref{asmp:convergence_nuisance} plausible for moderate values of $\alpha, \beta$ (e.g., $\alpha = 1, \beta = 2$).  Finally, we can define our estimator of the upper and lower bounds as the average of the pseudo-outcomes defined below.

\begin{definition}[Pseudo-Outcomes]\label{def:estimator} 
   We first define
\begin{align*}
   &\hat{h}(X,Y,Z,\cS) := \\
   &\qquad  \widehat{\mu}_{{\cS}}(X,Z)
        +
        \frac{\mathbf{1}\{\Pi\in {\cS}\}}{P(\Pi\in {\cS}\mid X)}
        \bigl(Y-\widehat{\mu}_{{\cS}}(X,Z)\bigr)
\end{align*}
where $\cS\subseteq\Hclass$ is an arbitrary set of models.  Then, we define pseudo-outcomes for the lower and upper bounds as 
   \begin{align*}
      & \widehat{\psi}_L(Y, X, \Pi, Z; \widehat{f}_M) = \\
      & \bLneuhat\,\hat{h}(X,Y,Z,\Sagree) + \bLmonhat\, \hat{h}(X,Y,Z,\Sletighthat) \\
      & + \bLcorhat\, {\varphi_L}(X,Z;\hat d_L) + \bLdefhat\,  Y_{\min}\\
      &\widehat{\psi}_U(Y, X, \Pi, Z; \widehat{f}_M) = \\
      & \bUneuhat\,  \hat{h}(X,Y,Z,\Sagree) + \bUmonhat\, \hat{h}(X,Y,Z,\Sgetighthat) \\
      &+ \bUcorhat\,{\varphi_U}(X,Z;\hat d_U) + \bUdefhat\, Y_{\max}
   \end{align*}
   where the argument after the semicolon denotes the performance function used to construct the partitions and branch selectors: $\bsLhat, \bsUhat$ and the estimated sets are computed from the estimated performance $\hat{f}_M$ in place of $f_M$. 
\end{definition}
\begin{restatable}[Consistency and asymptotic normality of $\widehat{L}(\pi_e), \widehat{U}(\pi_e)$]{theorem}{MainResult}\label{theorem:estimator}
If~\cref{assump:known_prop,assump:margin,asmp:convergence_nuisance} hold for some $\alpha, \beta > 0$, then
\begin{equation*}
\widehat{L}(\pi_e) = n^{-1}\sum_{i=1}^n \widehat{\psi}_L(Y_i, X_i, \Pi_i, Z_i; \widehat{f}_M) \xrightarrow{P} L(\pi_e)
\end{equation*}
\begin{equation*}
\widehat{U}(\pi_e) = n^{-1}\sum_{i=1}^n \widehat{\psi}_U(Y_i, X_i, \Pi_i, Z_i; \widehat{f}_M) \xrightarrow{P} U(\pi_e)
\end{equation*}
Where $\xrightarrow{P}$ denotes convergence in probability. Furthermore,
\begin{equation*}
\sqrt{n}(\widehat{L}(\pi_e)-{L}(\pi_e)) \xrightarrow{d} N(0, \sigma_L^2)\end{equation*}
\begin{equation*}
    \sqrt{n}(\widehat{U}(\pi_e)-{U}(\pi_e)) \xrightarrow{d} N(0, \sigma_U^2)
\end{equation*}
Where $\xrightarrow{d}$ denotes convergence in distribution, and $\sigma_L^2 = \Var({\psi}_L), \sigma_U^2 = \Var({\psi}_U)$, with ${\psi}_L, {\psi}_U$ denoting the pseudo-outcomes of \cref{def:estimator} evaluated with the true nuisance functions $\mu_j, f_M$ in place of their estimates.
\end{restatable}
Since our estimator is asymptotically normal, this provides us the following corollary: 
\begin{restatable}[Confidence Intervals for $L(\pi_e), U(\pi_e)$]{corollary}{ConfidenceIntervals}\label{corollary:ci}
Under the conditions of \cref{theorem:estimator}, we can construct the following confidence intervals for $L(\pi_e), U(\pi_e)$:
\begin{align*}
   &\widehat{L}(\pi_e) \pm z_{1-\gamma/2}\sqrt{\widehat{Var}(\widehat{\psi}_L)/n}\\
   &\widehat{U}(\pi_e) \pm z_{1-\gamma/2}\sqrt{\widehat{Var}(\widehat{\psi}_U)/n}
\end{align*}
Where $z_{1-\gamma/2}$ is the $1-\gamma/2$ quantile of the standard normal distribution, and $\widehat{Var}(\widehat{\psi}_L), \widehat{Var}(\widehat{\psi}_U)$ are the empirical variance of $\widehat{\psi}_L, \widehat{\psi}_U$, respectively.
\end{restatable}

We note that our results require~\cref{assump:margin} to hold for some value of $\alpha, \beta$. The first ($\alpha$) condition controls the bias from selecting the wrong argmax/argmin in the $\bLcor, \bUcor$ branches; the second ($\beta$) condition controls a separate source of error specific to our setting, namely that $\hat\psi_L, \hat\psi_U$ are evaluated at the estimated partitions $\Sletighthat, \Sgetighthat$ rather than at the population partitions in \cref{def:bounds}, and $\hat f_M$'s error can permute the relevant ordering. The probability of a permutation is bounded polynomially by the margin condition, contributing $o_p(n^{-1/2})$ to the overall error.
\section{Experiments}\label{sec:experiments}

We now shift discussion to a semi-synthetic simulation study to demonstrate the efficiency of our proposed bounds. All experiment code is available on Github at \href{https://github.com/oberst-lab/counterfactual-correctness-experiments}{oberst-lab/counterfactual-correctness-experiments}.

We will use data from \citet{imai2023experimental}, an experiment evaluating the impact of showing judges algorithmically generated risk scores (1-6) prior to bail sentencing.  This data also records the action taken by judges, which includes assigning high bail, low bail, or release without bail to suspects (note that ``imprisonment without bail'' does not appear as an option in the data).  Downstream outcomes were recorded for each suspect, such as failure to appear (FTA) at a trial hearing. The full data includes suspect covariates, judge group, judge decision, risk score, and FTA for every individual. FTA distributions in the data are shown in \cref{tab:fta-score-distribution}.

\textbf{Simulation Setup} We use this data to construct a semi-synthetic simulation study, re-using the covariates of individuals in the dataset, as well as the original risk scores.  We introduce simulation in order to reason about the `ground truth' counterfactual outcomes under different forms of AI assistance.  We focus primarily on hypothetical alerting models that notify the judge of high-risk individuals: This framing allows us to consider the ``control arm'' of the trial as providing information on judge behavior in the absence of an alert.  We make the additional simplification, for convenience, of considering ``requires bail'' $(A = 1)$ and ``does not require bail'' as the only actions, collapsing the ``high'' and ``low'' bail options.  Using the original data, we learn
\begin{align*}
    \hat{f}_{a}(X) &:= \hat{P}(A=1 \mid X, \text{Arm} = \text{control})\\
    \hat{f}_{y}(X,A) &:= \hat{P}(\text{FTA} = 1 \mid X, A, \text{Arm} = \text{control})
\end{align*}
where both $\hat{f}_{a}$ and $\hat{f}_y$ represent our estimates of the conditional probability of (a) the judge requiring some form of bail in the absence of ML assistance, and (b) the defendant failing to appear, conditioned on the type of bail required. {We also retain the risk score (from 1-6) that is provided in the original dataset, which we denote $\mathfrak{R}$, and which is used to inform some of the models that we evaluate later on.}

We then simulate judge and defendant behavior as follows: If no alert is raised ($\hat{Z}=0$), then judges make decisions according to $\hat{f}_a(X)$, and if an alert is raised ($\hat{Z} = 1$), then their probability increases proportionally to the performance (in this case, the precision) of the model as follows
\begin{align*}
    &P(A = 1 \mid X, \Pi, \hat{Z}) \\
    &= \sigma(\text{logit}( \hat{f}_a(X)) + \mathbf{1}\{\hat{Z} = 1\} \cdot f_M(X, \Pi))
\end{align*}
where $\sigma$ is the sigmoid function, and $f_M(X, \Pi)$ is the performance of the model $\Pi$ at $X$.  We then simulate failure-to-appear using $P(\text{FTA} = 1 \mid X, A) = \hat{f}_y(X, A)$, and construct our outcome $Y$ as a net utility that rewards court appearance while imposing a penalty for requiring bail
\begin{align*}
    Y_{\text{raw}} &= 1 - c_{\text{FTA}}\cdot\text{FTA} - c_{\text{det}}\cdot A, &
    Y &= \frac{Y_{\text{raw}} + c_{\text{det}}}{1 + c_{\text{det}}},
\end{align*}
with $c_{\text{FTA}} = 1$ and $c_{\text{det}} = 0.1$. The affine rescaling is for interpretability, as $Y_\text{raw}$ is already bounded by $[-0.1, 1]$. The bail penalty $c_{\text{det}}$ reflects our intuition that over-detention is harmful and allows incorrect predictions, particularly false positives, to cause harm. 

Finally, we take $Z$ (the `ground truth' label) to be whether or not the defendant would fail to appear if released without bail. While in reality, $Z$ would only be observed for individuals released without bail, for the purpose of the illustrative simulation here, we take it to be observed for all individuals.\footnote{In~\cref{missingness} we discuss assumptions under which our method can handle missingness, and in Exp.\ 5 we consider a misspecified experiment with missing values of $Z$.}
More broadly, while we use real data to inform our simulation of judge and defendant behavior, we do not claim that our simulation is a realistic surrogate for real-world behavior: Rather, we design our simulation primarily for ease of exposition, and to illustrate the application of our method. We now briefly describe the experiments run, with additional details and figures in \cref{app:experiments}. 

\begin{table}[t]
    \centering
    \small
\begin{tabular}{lcc|ccc}
	\toprule
	Exp \# & \makecell{Width \\ (Ours)} & \makecell{Width \\ (Chen)} & \makecell{Lower Bd. \\(Ours)} & \makecell{Upper Bd. \\(Ours)} & \makecell{True \\ value} \\
	\midrule
	1      & 0.331                      & 0.704                      & 0.430                         & 0.761                         & 0.726                    \\
	2      & 0.121                      & 0.362                      & 0.616                         & 0.737                         & 0.728                    \\
	3      & 0.140                      & 0.281                      & 0.624                         & 0.764                         & 0.696                    \\
	4      & 0.454                      & 0.802                      & 0.349                         & 0.803                         & 0.695                    \\
	5      & 0.204                      & 0.281                      & 0.586                         & 0.790                         & 0.697                    \\
	\bottomrule
\end{tabular}

    \caption{Semi-synthetic experiment results (see~\cref{sec:experiments}) showing the mean width of our bounds 
    $\hat{U}(\pi_e)-\hat{L}(\pi_e)$ compared to those derived via the method of \citet{chen2025}, alongside the average upper and lower bound values and the true value. We see that our bounds are consistently tighter and always contain the true value of $Y$. }
   \label{tab:simulation_results}
\end{table}

\textbf{Evaluation of changing alert thresholds and re-training of a pre-existing model (Exp 1, 2)}

To start, we consider a simulated setting of an RCT that has two arms: a control arm in which no alerts are raised, and a trial arm in which alerts are raised when the risk score $\mathfrak{R}$ present in the \citet{imai2023experimental} dataset is greater than 3. Judge decisions and outcomes are then simulated as described above.
In Experiment 1\label{exp_1}, we evaluate a new alerting system that uses the same risk scores $\mathfrak{R}$, but raises an alert at a lower (less conservative) threshold.  In particular, alerts are raised for every individual with a risk score $\mathfrak{R}$ greater than 1.
In Experiment 2\label{exp_2}, we create a new risk score algorithm and evaluate an untrialed model that raises an alert when this new risk score is greater than 3. This algorithm is discussed in greater detail in \cref{app:add_exp}. This setting represents the motivating scenario where a freshly trained model needs to be evaluated. 

Overall results are shown in~\cref{tab:simulation_results}, where we observe that the bounds produced by~\citet{chen2025} are extremely wide (0.70 and 0.36 for Exp.\ 1 and 2, respectively) given that the outcome (and hence the target to estimate) is bounded between zero and one.  This gap reflects the fact that in both cases, the untrialed model tends to produce alerts in cases where no alerts occurred in the original trial (trialed model raises alerts on 28.1\% of individuals, whereas the untrialed model raises an alert on 80.2\% and 45.6\% of individuals in Exp.\ 1 and 2, respectively). Our method produces bounds that are consistently narrower than those of prior work (0.33 and 0.12, respectively), because we are able to produce tighter lower bounds in settings where these previously unobserved alerts are correct, and tighter upper bounds when they are incorrect.

In the Appendix, we conduct a further ablation study, using data from Experiment 1, finding that the primary driver of improvement over \citet{chen2025} is the use of the counterfactual correctness bounds, rather than e.g., stratification of $M$ on $X$. In \cref{tab:coverage-experiment}, we also show the results of 100 Monte Carlo trials of Exp.\ 1, demonstrating that all methods consistently capture the true expected outcome, but our method produces a smaller bound interval and narrower confidence intervals.

\textbf{Evaluation in the absence of an RCT (Exp 3, 4)} 

When no RCT has been conducted, but data exists from prior to the deployment of any AI system, such data is effectively a ``single-arm'' RCT that only contains a control arm that always takes the ``neutral'' action.  We illustrate the application of our method in this setting.  Here, the prediction of ``no alert'' is taken to be a neutral action (leading to similar outcomes as those in prior data), while the prediction of ``alert'', when correct/incorrect, implies that outcomes will be no worse/no better than those observed in prior data. In Experiment 3\label{exp_3}, we evaluate an untrialed model that raises an alert for individuals with a risk score greater than 3, and in Experiment 4\label{exp_4}, we evaluate an untrialed model that does the same for those with a risk score greater than 1.

Overall results are similarly shown in~\cref{tab:simulation_results}, where we again find our bounds are consistently tighter than \citet{chen2025}. We also notice that bound width grows with the alert rate of the untrialed model. Experiment 3's conservative threshold ($\mathfrak{R}>3$) raises few alerts and leads to a narrow bound ($0.140$), while Exp.\ 4's aggressive threshold ($\mathfrak{R}>1$) raises many more alerts and leads to a much wider bound ($0.454$). Additionally, in \cref{fig:simulation_results1}, we illustrate our bounds across different age groups, compared to the method from \citet{chen2025}, after modifying the latter method to use per-age group performances. We see that across all groups, our bounds capture the true value and are tighter, agreeing with \cref{tab:simulation_results}. This result suggests that our notion of counterfactual correctness is the primary driver of empirical improvement, rather than stratifying $M$ by $X$. 

\textbf{Misspecified example with missing data (Exp 5)}\label{exp_5} 

As a sanity check, we consider the setting of Experiment 3, but where we hide the counterfactual failure-to-appear values for offenders that were not released without bail.  This modification violates our assumptions, and this experiment is designed to illustrate the impact of missing values of $Z$, as discussed further in \cref{missingness}.  

In~\cref{tab:simulation_results}, we observe that our bounds are wider (0.2 in Exp.\ 5 vs 0.14 in Exp.\ 3), but still contain the true value.  We show further details in \cref{tab:exp_3_vs5} in the Appendix, where we demonstrate that the bounds still capture the true values, and the confidence intervals for the lower and upper bounds still capture the ground truth values. This result suggests that the bias introduced by the misspecification in this particular simulation is not significant enough to yield invalid results from applying our method. 

\begin{figure}
   \centering
   \includegraphics[width=\linewidth]{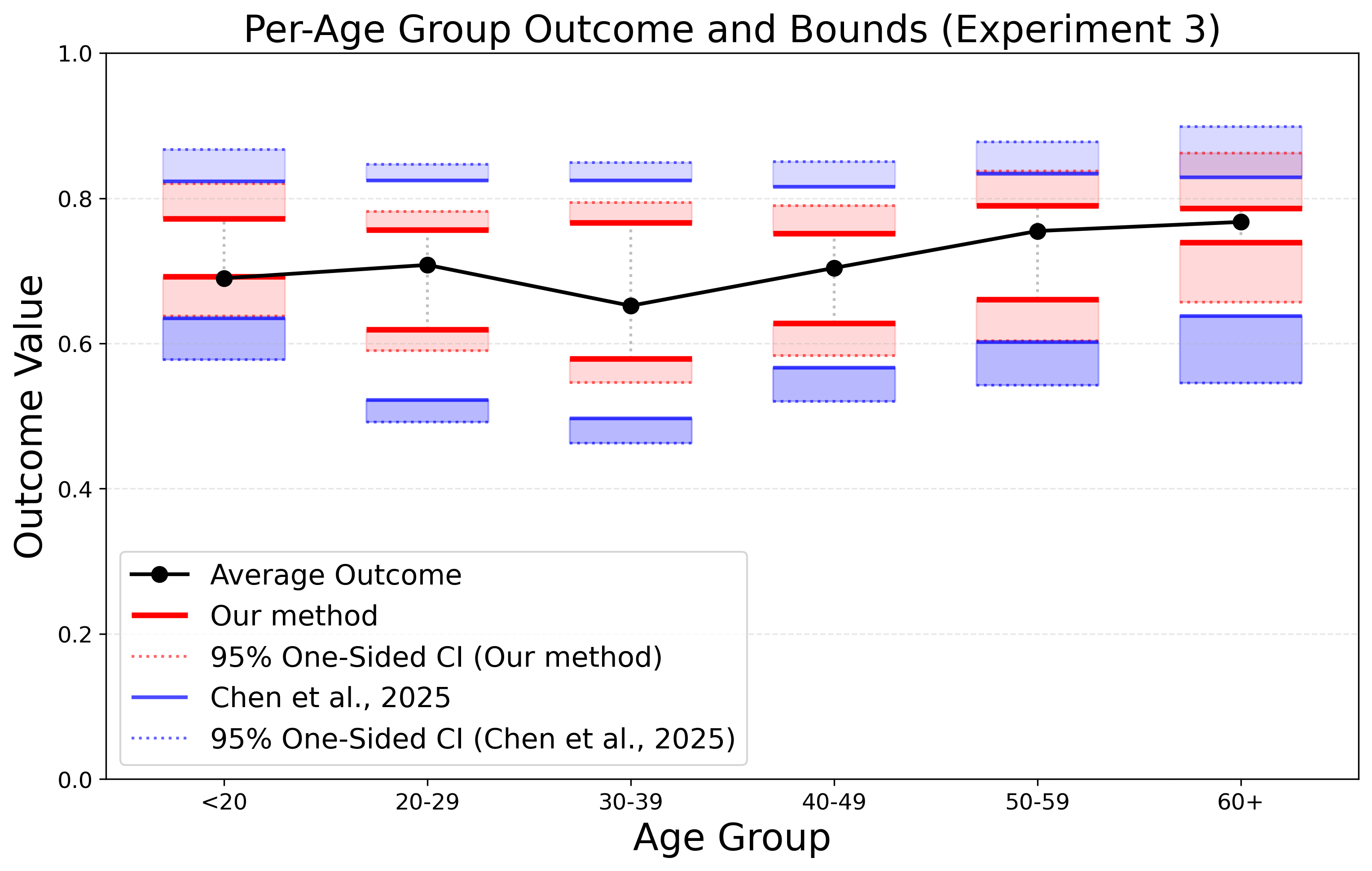}
   \caption{(Experiment 3) Comparison of bounds on downstream outcomes for an untrialed model with a single trialed model in the RCT. The figure displays the lower bound and upper bound computed using our proposed method and a modified version of the method from \citet{chen2025}, also plotted in \cref{fig:bounds_B} with performance as functions of individual age.   }
   \label{fig:simulation_results1}
\end{figure}
\section{Conclusion and Limitations}
When a newly trained model has better predictive performance than models previously evaluated in an RCT, there will naturally be limited data on how its predictions impact outcomes, insofar as it makes correct predictions where prior models failed. Our work improves over prior work in partial identification of the causal impact of AI systems, by recognizing that, when ground truth labels exist in a prior RCT, the model's predictions can be evaluated against that ground truth, and correctness used as a directional indicator of relative impact at the individual level.

Several limitations and areas for future investigation merit discussion. Since our falsification tests are in-expectation, individual-level assumption violations may go undetected. Handling cases where an identical prediction from a better model has a worse outcome, or where neutral action depends on performance, is an important direction for future work. Many deployed models also do not provide discrete decision recommendations. Instead, many may output continuous predictions. Our method can still be applied by imposing discrete buckets, or using $\bLcor,\bUcor,\bLdef,\bUdef$ branches only, but more efficient methods of evaluation may be of interest. As discussed, our method has limitations when values of $Z$ are missing, which presents avenues for future research. Finally, investigating an approach that incorporates observational data, which is often significantly more available than RCT data, would be valuable.

\newpage
\bibliography{uai2026-template}

\newpage

\onecolumn

\appendix
\crefalias{section}{appendix}
\crefalias{subsection}{subappendix}
\crefalias{subsubsection}{subsubappendix}

\title{Bounding the Causal Impact of ML-assisted Decision-Making \\ via Counterfactual Correctness (Supplementary Material)}
\maketitle

\setcounter{table}{0}
\renewcommand{\thetable}{A.\arabic{table}}

\setcounter{figure}{0}
\renewcommand{\thefigure}{A.\arabic{figure}}

\section{Appendix Overview}
\label{app:overview}
The appendix contains the following sections: 
\begin{itemize}
	\item \cref{app:proofs}: Proofs and Additional Materials.\\
	      This section provides proofs for falsification tests for the core assumptions introduced in the main text. If these tests fail, then the assumptions do not hold in the data, and any results may not be valid.
	\item \cref{app:bounds_thm}: Proof of \cref{def:bounds}.\\
	      This section provides a formal proof of the core theorem of the paper. To prove the validity of the bounds, we write the analytical form of the bounds and demonstrate the expected value of the untrialed model falls within the bounds. 
	\item \cref{app:estimator}: Proof of \cref{theorem:estimator}.\\
	      This section provides a formal proof that the empirical estimator for our bounds is consistent and asymptotically normal, enabling us to construct confidence intervals for our bounds. Since some of the branches of the bounds contain discontinuous ``max/min'' operators, the techniques of \citet{chen2025} cannot be directly applied to the entire estimator. We instead formulate an AIPW estimator for our bounds. By Theorem 2 of \citet{byun2024}, we can demonstrate that such an estimator is consistent and asymptotically normal if seven conditions are met. For each assumption, we verify the conditions are met and directly apply the theorem.
    \item \cref{app:experiments}: Further Simulation Details.\\
    This section provides some brief additional information about the simulation study. 
    \item \cref{app:related_work}: Additional Related Work.
    This section provides an overview of works in related areas that are relevant to our work, but were not discussed in the main text. 
    \item \cref{missingness}: Missing Data.\\
    This section discusses the handling of cases in which $Z$ is missing from our data. 
    \item \cref{app:flowchart}: Flowchart.\\
    This section provides a flowchart visualizing the branch-selection logic ($\bsL$) used to construct the bounds.
\end{itemize}
\newpage
\section{Proofs and Additional Materials}
\label{app:proofs}
The following assumptions are used in the proofs of the propositions and theorems in this paper.\\
\begin{restatable}[Consistency]{corollary}{ConsistencyCorollary}\label{cor:consistency}
Under \cref{asmp:dgp}, if $\hat{Z}=\hat{z},\ M=m,\ Z=z$ then
\begin{equation*}Y = Y(\hat{Z}=\hat{z}, M=m, Z=z).\end{equation*}
\end{restatable}

\begin{restatable}[Conditional Ignorability]{corollary}{ConditionalIgnorability}\label{cor:conditional_ignorability}
Under \cref{asmp:dgp}, since $Y(\hat{z},m)\perp\!\!\!\perp (\hat{Z},M)\mid (X,Z)$, we have
\begin{equation}
   P\left(Y(\hat{z},m)\mid X,Z\right)=P\left(Y\mid \hat{Z}=\hat{z}, M=m, X,Z\right).
\end{equation}
\end{restatable}
\subsection{Falsification of Assumption 3.1}
\CorrectActionsCauseNoHarm*
\FalsificationAssumptioncorrect*
\begin{proof}
Suppose \cref{asmp:dgp} holds. Choose $x\in\mathcal{X}_{c \neq w}$. Then we have the following:
\begin{align}
   \mathbb{E}&\left[Y\mid X=x, \Pi=\pi_c, Z=\pi_c(x)\right] - \mathbb{E}\left[Y\mid X=x, \Pi=\pi_w, Z=\pi_c(x)\right] \nonumber \\
   &= \mathbb{E}\left[Y\mid X=x, \hat{Z}=\pi_c(x), M=f_M(x,\pi_c), Z=\pi_c(x)\right] - \mathbb{E}\left[Y\mid X=x, \hat{Z}=\pi_w(x), M=f_M(x,\pi_w), Z=\pi_c(x)\right]\label{eq:correct_line1}\\
      &= \mathbb{E}\left[Y(\hat{Z}=\pi_c(x), M=f_M(x,\pi_c), Z=\pi_c(x)) - Y(\hat{Z}=\pi_w(x), M=f_M(x,\pi_w), Z=\pi_c(x))\mid X=x\right] \label{eq:correct_line2} \\
      &\ge 0 \label{eq:correct_line3}
\end{align}
Where \cref{eq:correct_line1} follows from \cref{asmp:dgp}, \cref{eq:correct_line2} follows from \cref{cor:consistency}, and \cref{eq:correct_line3} follows from \cref{assumptions:correct}. We now aggregate.
\begin{align}
   &\mathbb{E}[Y\mid X\in\mathcal{X}_{c\neq w}, \Pi=\pi_c, Z=\pi_c(X)] - \mathbb{E}[Y\mid X\in\mathcal{X}_{c\neq w}, \Pi=\pi_w, Z=\pi_c(X)]\nonumber\\
   &= \int_x \mathbb{E}[Y\mid X=x, \Pi=\pi_c, Z=\pi_c(x)]\, dP(x\mid X\in\mathcal{X}_{c\neq w}, \Pi=\pi_c, Z=\pi_c(X))\nonumber\\
   &\quad - \int_x \mathbb{E}[Y\mid X=x, \Pi=\pi_w, Z=\pi_c(x)]\, dP(x\mid X\in\mathcal{X}_{c\neq w}, \Pi=\pi_w, Z=\pi_c(X))\label{eq:integration2}\\
   &= \int_x \Big(\mathbb{E}[Y\mid X=x, \Pi=\pi_c, Z=\pi_c(x)] - \mathbb{E}[Y\mid X=x, \Pi=\pi_w, Z=\pi_c(x)]\Big)\, dP(x\mid X\in\mathcal{X}_{c\neq w}, Z=\pi_c(X))\label{eq:integration3}\\
   &\geq 0 \label{eq:final_res3}
\end{align}
Where \cref{eq:integration2} follows by the tower property, \cref{eq:integration3} by $\Pi\ci(X,Z)$ (dropping $\Pi$
 from both measures and combining), and \cref{eq:final_res3} by \cref{eq:correct_line3}.
\end{proof}
\subsection{Falsification of Assumption 3.2}
\PerformanceMonotonicity*
\FalsificationAssumptionThreeOne*
We will use the following lemma: 
   \begin{restatable}[Observational Equivalence]{lemma}{ObsEquivLemma}\label{lem:obs_equiv}
      For any $\pi_i\in\Hclass$ and any $x\in\mathcal{X},\, z\in\mathcal{Z}$ such that $P(X=x,\, Z=z,\, \Pi=\pi_i)>0$:
\begin{equation}
   \mathbb{E}\left[Y|X=x, \Pi = \pi_i, Z=z\right] = \mathbb{E}\left[Y(\hat{Z}=\pi_i(x), M=f_M(x, \pi_i), Z=z)|X=x, Z=z\right]\label{eq:obs_line4}
\end{equation}
  \end{restatable}
\begin{proof}
         \begin{align}
	\mathbb{E}\left[Y|X=x, \Pi = \pi_i, Z=z\right] & = \mathbb{E}\left[Y|X=x, \hat{Z}=\pi_i(x), M=f_M(x, \pi_i), Z=z\right]                                           \label{eq:obs_line1}\\
      & = \mathbb{E}[Y(\hat{Z}=\pi_i(x), M=f_M(x, \pi_i), Z=z)|X=x, \hat{Z}=\pi_i(x),\nonumber\\
      &\qquad M=f_M(x, \pi_i), Z=z] \label{eq:obs_line2}\\
      & = \mathbb{E}\left[Y(\hat{Z}=\pi_i(x), M=f_M(x, \pi_i), Z=z)|X=x, Z=z\right]\label{eq:obs_line4.1}
\end{align}
Where \cref{eq:obs_line1} follows because $\hat{Z}=\Pi(X)$ and $M=f_M(X,\Pi)$ are deterministic given $(X,\Pi)$, and $Y\ci \Pi\mid X,\hat{Z},M,Z$ under \cref{asmp:dgp}, since $\Pi$ affects $Y$ only through $(\hat{Z},M)$. \cref{eq:obs_line2} follows from Consistency (\cref{cor:consistency}), noting that given $Z=z$, we have $Y(\hat{z},m,z)=Y(\hat{z},m)$ under \cref{asmp:dgp}. \cref{eq:obs_line4.1} follows from Conditional Ignorability (\cref{cor:conditional_ignorability}), i.e., $Y(\hat{z},m)\ci(\hat{Z},M)\mid X,Z$.
\end{proof}
This lemma enables us to prove Proposition~\ref{prop:falsification_monotonicity} and \ref{prop:neutral}.\\
\begin{proof}
Choosing $x\in \mathcal{X}_s$ and setting $z=\pi_+(x)$, we have that:
\begin{align}
   \mathbb{E}\left[Y\mid X=x, \Pi=\pi_+, Z = \pi_+(x)\right] &- \mathbb{E}\left[Y\mid X=x, \Pi=\pi_-, Z = \pi_+(x)\right]\nonumber\\
   &= \mathbb{E}[Y(\hat{Z}= \pi_+(x), M=f_M(x, \pi_+), Z = \pi_+(x))-\\
   &\qquad
  Y(\hat{Z}= \pi_+(x), M=f_M(x,\pi_-), Z = \pi_+(x))\mid X=x, Z=\pi_+(x)]\label{eq:lemres}\\
   &\geq 0\label{eq:geqzero}
\end{align}
Where \cref{eq:lemres} follows from Lemma~\ref{lem:obs_equiv} with $z=\pi_+(x)$, noting that since $x\in\mathcal{X}_s\subseteq\mathcal{X}_{agree}$ we have $\pi_-(x)=\pi_+(x)$, so both potential outcomes share the prediction $\hat{Z}=\pi_+(x)$ and differ only in $M$. \cref{eq:geqzero} then follows from \cref{assumption:monotonicity}, which compares same-prediction outcomes that differ only in performance $M$. For this to hold over the entire $\mathcal{X}_s$, we have:
\begin{align}
   &\mathbb{E}[Y\mid X\in\mathcal{X}_s, \Pi=\pi_+, Z=\pi_+(X)] - \mathbb{E}[Y\mid X\in\mathcal{X}_s, \Pi=\pi_-, Z=\pi_+(X)]\nonumber\\
   &= \int_x \mathbb{E}[Y\mid X=x, \Pi=\pi_+, Z=\pi_+(x)]\, dP(x\mid X\in\mathcal{X}_s, \Pi=\pi_+, Z=\pi_+(X))\nonumber\\
   &\quad - \int_x \mathbb{E}[Y\mid X=x, \Pi=\pi_-, Z=\pi_+(x)]\, dP(x\mid X\in\mathcal{X}_s, \Pi=\pi_-, Z=\pi_+(X))\label{eq:mon_int_1}\\
   &= \int_x \Big(\mathbb{E}[Y\mid X=x, \Pi=\pi_+, Z=\pi_+(x)] - \mathbb{E}[Y\mid X=x, \Pi=\pi_-, Z=\pi_+(x)]\Big)\, dP(x\mid X\in\mathcal{X}_s, Z=\pi_+(X))\label{eq:mon_int_2}\\
   &\geq 0 \label{eq:final_res_mon}
\end{align}
Where \cref{eq:mon_int_1} follows from the tower property, \cref{eq:mon_int_2} follows from $\Pi\ci(X,Z)$ (dropping $\Pi$ from both measures, which makes them identical and lets the integrals combine), and \cref{eq:final_res_mon} follows from \cref{eq:geqzero}, thus contradicting our observation. 
\end{proof}
This falsification test only tests for violations of \cref{assumption:monotonicity} when the models give the correct prediction. The following falsification test also exists for the case when the models give the incorrect prediction. 
\begin{restatable}[Falsification of \cref{assumption:monotonicity}]{proposition}{FalsificationAssumptionThreeOne2}\label{prop:falsification_monotonicity2}
    Let $\pi_+,\pi_- \in\Hclass$ be two trialed models. 
    Consider a sub-population with $\mathcal{X}_{better}\subseteq\mathcal{X}$ such that for all $x\in \mathcal{X}_{better}$, $f_M(x,\pi_+) > f_M( x, \pi_-)$, and a subpopulation with covariates $\mathcal{X}_{agree} \subseteq \mathcal{X}$ such that for all $x \in \mathcal{X}_{agree}$, $\pi_+(x)= \pi_-(x)$. Then, for any $\mathcal{X}_{s} =  \mathcal{X}_{better}\cap \mathcal{X}_{agree}$
    \begin{align*}
	& \mathbb{E}[Y|X\in\mathcal{X}_{s}, \Pi = \pi_+, Z\ne\pi_+(X)] \\
        &> \mathbb{E}[Y|X\in\mathcal{X}_{s}, \Pi = \pi_-, Z\ne\pi_+(X)]
    \end{align*}
    falsifies \cref{assumption:monotonicity}.
 \end{restatable}
The proof is identical to the above, save for changes in the inequality direction and the value of $Z$ in the expectation.
\subsection{Falsification of Assumption 3.3}
\NeutralAction*
\FalsificationAssumptionThreeTwo*
\begin{proof}
Suppose \cref{assumption:neutral} holds. Choose $x\in \mathcal{X}_{\hat{z}_0}$ and $z\in\mathcal{Z}$, then we have the following:
\begin{align}
   \mathbb{E} & [Y|X=x, \Pi = \pi_2, Z=z] - \mathbb{E}[Y|X=x, \Pi = \pi_1, Z=z] \nonumber \\
            & =\mathbb{E}[Y(\hat{Z}=\pi_2(x), M=f_M(x,\pi_2), Z=z)- \nonumber\\
            &\qquad Y(\hat{Z}=\pi_1(x), M=f_M(x,\pi_1), Z=z)\mid X=x, Z=z] \label{eq:neutral_proof_line2} \\
            & =\mathbb{E}[Y(\hat{Z}=\hat{z}_0, M=f_M(x,\pi_2), Z=z)- Y(\hat{Z}=\hat{z}_0, M=f_M(x, \pi_1), Z=z)\mid X=x, Z=z] \label{eq:neutral_proof_line3} \\
            & =0 \label{eq:neutral_proof_line4}
\end{align}
Where \cref{eq:neutral_proof_line2} follows from Lemma~\ref{lem:obs_equiv}. \cref{eq:neutral_proof_line3} follows from the definitions of $\pi_1$ and $\pi_2$. \cref{eq:neutral_proof_line4} follows from \cref{assumption:neutral}. We then aggregate: 
\begin{align}
    \mathbb{E}[&Y|X\in \mathcal{X}_{\hat{z}_0}, \Pi = \pi_2, Z=z]-\mathbb{E}[Y|X\in \mathcal{X}_{\hat{z}_0}, \Pi = \pi_1, Z=z]\\
    &=\int_x\mathbb{E}[Y|X=x, \Pi = \pi_2, Z=z]dP(x|\Pi = \pi_2, X\in\mathcal{X}_{\hat{z}_0}, Z=z)-\nonumber\\
    &\qquad\int_x\mathbb{E}[Y|X=x, \Pi = \pi_1, Z=z]dP(x|\Pi = \pi_1, X\in\mathcal{X}_{\hat{z}_0}, Z=z)\label{eq:integration20}\\
    &=\int_x\mathbb{E}[Y|X=x, \Pi = \pi_2, Z=z]dP(x|X\in\mathcal{X}_{\hat{z}_0}, Z=z)-\nonumber\\
    &\qquad\int_x\mathbb{E}[Y|X=x, \Pi = \pi_1, Z=z]dP(x|X\in\mathcal{X}_{\hat{z}_0}, Z=z)\\
    &=\int_x\mathbb{E}[Y|X=x, \Pi = \pi_2, Z=z]-\mathbb{E}[Y|X=x, \Pi = \pi_1, Z=z]dP(x|X\in\mathcal{X}_{\hat{z}_0}, Z=z)\\
    & =0\label{eq:final_res2}
\end{align}
Where the logic is the same as before, contradicting our observation. 
\end{proof}

\newpage       
\section{Proof of Theorem 4.1}
\label{app:bounds_thm}
\IndicatorDef*
\BoundsTheorem*
\subsection{Proof of Theorem 4.1 (Lower Bound)}
\label{app:mu_ex}
Since $\bsL$ is a function defined by first-match \texttt{cases} (\cref{def:indicators}), exactly one of the indicators $\mathbf{1}\{\bsL=k\}$, $k\in\{\lneu,\lmon,\lcor,\ldef\}$, equals $1$, so they are mutually exclusive and exhaustive. We will sometimes omit the arguments for brevity. \\\\
Thus the equality follows:
\begin{align}
	1 & =  \bLneu +\bLmon +\bLcor +\bLdef\label{eq: split}
\end{align}

We have the following:
\begin{align}
   \mathbb{E}&[Y(\hat{Z} = \pi_e(X), M = f_M(X,\pi_e), Z = z)]\nonumber\\
    & = \mathbb{E}[\mathbb{E}[Y(\hat{Z} = \pi_e(X), M = f_M(X,\pi_e), Z=z) \mid X]] \label{eq:law_iterated_expectation}\\
    & = \mathbb{E}[\mathbb{E}[Y \mid X, \hat{Z} = \pi_e(X), M = f_M(X,\pi_e), Z=z]] \label{eq:conditional_expectation} \\
    & = \mathbb{E}[\mathbb{E}[Y \mid X, \hat{Z} = \pi_e(X), M = f_M(X,\pi_e), Z=z](\bLneu + \nonumber                                          \\
    & \quad \bLmon +\nonumber                                                                                                     \\
    & \quad\bLcor +\nonumber                                                                                                      \\
    & \quad\bLdef)] \label{eq:indicator_expansion}                                                                                                               \\
    & = \mathbb{E}[\bLneu \mathbb{E}[Y \mid X, \hat{Z} = \pi_e(X), M = f_M(X,\pi_e), Z=z] \label{eq:case_agree_neutral} +                          \\
    & \quad \bLmon \mathbb{E}[Y \mid X, \hat{Z} = \pi_e(X), M = f_M(X,\pi_e), Z=z] \label{eq:case_agree_nonneutral_next_lowest} +                          \\
    & \quad \bLcor \mathbb{E}[Y \mid X, \hat{Z} = \pi_e(X), M = f_M(X,\pi_e), Z=z]  \label{eq:case_no_agree_optimal} +                            \\
    & \quad \bLdef \mathbb{E}[Y \mid X, \hat{Z} = \pi_e(X), M = f_M(X,\pi_e), Z=z] ]\label{eq:case_otherwise}
\end{align}
Where \cref{eq:law_iterated_expectation} follows from the law of iterated expectations. \cref{eq:conditional_expectation} follows from conditional ignorability (Corollary~\ref{cor:conditional_ignorability}) and the fact that under \cref{asmp:dgp}, $Y(\hat{z},m,z) = Y(\hat z, m)$ under the conditioning. \cref{eq:indicator_expansion} follows from \cref{eq: split}. \cref{eq:case_otherwise}-\cref{eq:case_agree_nonneutral_next_lowest} follow from distributing the expectation over the sum. \\\\
We will now demonstrate that each line is greater than or equal to the corresponding line in the lower bound. \\

  \cref{eq:case_agree_neutral} We use the analogous proof from \cite{chen2025}. \\
  We first must demonstrate $Y\ci M \mid X, \hat{Z} = \hat{z}_0$: 
  \begin{align}
    P(Y \mid X, \hat{Z} = \hat{z}_0, M=m,Z=z) & = P(Y(\hat{Z} = \hat{z}_0, M = m, Z=z) \mid X = x, M = m, \hat{Z} = \hat{z}_0,Z=z) \label{eq:case_agree_neutral_consistency}\\
    &= P(Y(\hat{Z} = \hat{z}_0, M =m', Z=z) \mid X = x, M = m, \hat{Z} = \hat{z}_0,Z=z) \label{eq:case_agree_neutral_assp}\\
    &= P(Y(\hat{Z}=\hat{z}_0, M = m', Z=z) \mid X = x, M=m', \hat{Z} = \hat{z}_0,Z=z) \label{case_agree_neutral_independence}\\
      & = P(Y \mid X=x, \hat{Z} = \hat{z}_0, M = m',Z=z) \label{eq:case_agree_neutral_end}
  \end{align}
  Where \cref{eq:case_agree_neutral_consistency} follows from Corollary~\ref{cor:consistency}. \cref{eq:case_agree_neutral_assp} follows from \cref{assumption:neutral}. \cref{case_agree_neutral_independence} follows from the structure of \cref{thmt@@DataGeneratingProcess} implying $Y(\hat{z},m) \ci M, \hat{Z} \mid X$. \cref{eq:case_agree_neutral_end} follows from Corollary~\ref{cor:consistency}. Thus we have demonstrated $Y\ci M \mid X, \hat{Z} = \hat{z}_0$. \\
   Next, we recall that the lower selector takes the value $\lneu$ exactly when
    \begin{equation*}\bsL=\lneu:\quad \pi_e(x) = \hat{z}_0,\ \Sagree(x)\ne \varnothing.\end{equation*}
    Thus we have:
\begin{align}
    &\bLneu \mathbb{E}[Y \mid X, \hat{Z} = \pi_e(X), M = f_M(X,\pi_e),Z=z]\nonumber\\
     &= \bLneu \mathbb{E}[Y \mid X, \hat{Z} = \hat{z}_0, M = f_M(X,\pi_e),Z=z] \label{case_agree_neutral1}   \\
    &= \bLneu \mathbb{E}[Y \mid X, \hat{Z} = \hat{z}_0,Z=z] \label{eq:case_agree_neutral2}\\
      &= \bLneu \mathbb{E}[\mathbb{E}[Y \mid X = x, \hat{Z}= \hat{z}_0,Z=z]\mid X = x, \hat{Z} = \pi_e(x), \Pi \in  \Sagree(x)] \label{eq:case_agree_neutral3}\\
      &= \bLneu \mathbb{E}[\mathbb{E}[Y \mid X = x, \hat{Z}= \hat{z}_0,M,Z=z]\mid X = x, \hat{Z} = \pi_e(x), \Pi \in  \Sagree(x)] \label{eq:case_agree_neutral4}\\
      &= \bLneu \mathbb{E}[\mathbb{E}[Y \mid X = x, \hat{Z}= \pi_e(x),M,Z=z]\mid X = x, \hat{Z} = \pi_e(x), \Pi \in  \Sagree(x)] \label{eq:case_agree_neutral5}\\
      &= \bLneu \mathbb{E}[\mathbb{E}[Y \mid X = x, \hat{Z}= \pi_e(x),M, \Pi \in \Sagree(x),Z=z]\mid X = x, \hat{Z} = \pi_e(x), \Pi \in  \Sagree(x)] \label{eq:case_agree_neutral6}\\
      &= \bLneu \mathbb{E}[Y \mid X = x, \hat{Z}= \pi_e(x), \Pi \in \Sagree(x),Z=z]\label{eq:case_agree_neutral7}\\
      &= \bLneu \mathbb{E}[Y \mid X = x, \Pi \in \Sagree(x),Z=z]\label{eq:case_agree_neutral8}
\end{align}
Where \cref{case_agree_neutral1} follows from implication of the indicators. \cref{eq:case_agree_neutral2} follows from $Y\ci M \mid X, \hat{Z} = \hat{z}_0$ demonstrated above. \cref{eq:case_agree_neutral3} follows from the law of total expectation. \cref{eq:case_agree_neutral4} follows from \cref{thmt@@NeutralAction}. \cref{eq:case_agree_neutral5} follows from implication of the indicators. \cref{eq:case_agree_neutral6} follows from $Y\ci \Pi \mid X, \hat{Z}, M$ implied by the structure of \cref{thmt@@DataGeneratingProcess}. \cref{eq:case_agree_neutral7} follows from the law of total expectation. \cref{eq:case_agree_neutral8} follows from implication of the indicators. \\

  \cref{eq:case_agree_nonneutral_next_lowest} We have the following:
  First we recall that the lower selector takes the value $\lmon$ exactly when
   \begin{equation*}\bsL=\lmon:\quad \pi_e(x)\ne \hat{z}_0,\ \Sle(x,z)\ne\varnothing.\end{equation*}
   Thus we have: 
\begin{align}
	 \bLmon& \mathbb{E}[Y \mid X, \hat{Z} = \pi_e(X), M = f_M(X,\pi_e),Z=z]   \nonumber\\
	 & \geq \bLmon \mathbb{E}[Y|X,\Pi \in \Sletight(X,Z),Z=z]\label{eq:case_agree_nonneutral_next_lowest3}
\end{align}
  Where \cref{eq:case_agree_nonneutral_next_lowest3} follows from \cref{assumption:monotonicity} and the definition of $\Sletight(x)$ and the fact that it takes only the next-worst policies and not all worse policies. Thus, we have that the expected value of the outcome under the untrialed model is greater than or equal to the lower bound.

  \cref{eq:case_no_agree_optimal} We have the following: \\
  We recall that the lower selector takes the value $\lcor$ exactly when
  \begin{equation*}\bsL=\lcor:\quad \pi_e(x) = z,\ \Sle(x,z)=\varnothing,\ \Sdisagree(x,z)\ne\varnothing.\end{equation*}
  For all $i\in \Hclass$, we have the following:
\begin{align}
	 \bLcor \mathbb{E}[Y \mid X, \hat{Z} = \pi_e(X), M = f_M(X,\pi_e),Z=z]
	 & = \bLcor \mathbb{E}[Y|X,\hat{Z} = z, M = f_M(X,\pi_e),Z=z]                        \label{eq:case_no_agree_optimal2}\\
    &\geq \bLcor \max_{i\in \Sdisagree(X,Z)}\mathbb{E}[Y|X,\Pi = i,Z=z] \label{eq:case_no_agree_optimal5}
\end{align}
\cref{eq:case_no_agree_optimal2} holds by implication of the indicators.
\cref{eq:case_no_agree_optimal5} follows by \cref{assumptions:correct}, since $\pi_e(x)=z$, while every $\pi_i \in \Sdisagree(X,Z)$ has $\pi_i(x) \ne z$ by definition, so each such model takes an incorrect prediction. \\

  \cref{eq:case_otherwise} We have the following:
  \begin{equation}
      \bLdef \mathbb{E}[Y \mid X, \hat{Z} = \pi_e(X), M = f_M(X,\pi_e),Z=z] \geq Y_{min} \label{eq:case_otherwise2} 
  \end{equation} 
  Which holds by \cref{thmt@@BoundedOutcomes}. \\

\subsection{Proof of Theorem 4.1 (Upper Bound)}
Analogous to the lower bound, $\bsU$ is a function defined by first-match \texttt{cases} (\cref{def:indicators}), so exactly one of the indicators $\mathbf{1}\{\bsU=k\}$, $k\in\{\lneu,\lmon,\lcor,\ldef\}$, equals $1$, and they are mutually exclusive and exhaustive. \\
\begin{align}
	1 & =  \bUneu +\bUmon + \bUcor + \bUdef\label{eq: usplit}
\end{align}
From the previous computation, we have the following:
\begin{align}
	\mathbb{E}[Y(\hat{Z} = \pi_e(X), M = f_M(X,\pi_e), Z = z)]
	 & = \mathbb{E}[\bUneu \mathbb{E}[Y \mid X, \hat{Z} = \pi_e(X), M = f_M(X,\pi_e),Z=z] \label{eq:ucase_agree_neutral} +                          \\
	 & \quad \bUmon \mathbb{E}[Y \mid X, \hat{Z} = \pi_e(X), M = f_M(X,\pi_e),Z=z] \label{eq:ucase_agree_nonneutral_next_lowest} +                          \\
	 & \quad \bUcor \mathbb{E}[Y \mid X, \hat{Z} = \pi_e(X), M = f_M(X,\pi_e),Z=z]  \label{eq:ucase_no_agree_optimal} +                            \\
	 & \quad \bUdef \mathbb{E}[Y \mid X, \hat{Z} = \pi_e(X), M = f_M(X,\pi_e),Z=z]] \label{eq:ucase_otherwise}
\end{align}
Using the equivalent indicator statements established in \cref{eq: usplit}, we have that each line is less than or equal to the corresponding line in the upper bound. \\

  \cref{eq:ucase_agree_neutral} Identical to the proof in the lower bound. \\

  \cref{eq:ucase_agree_nonneutral_next_lowest} We recall that the upper selector takes the value $\lmon$ exactly when
   \begin{equation*}\bsU=\lmon:\quad \pi_e(x)\ne \hat{z}_0,\ \Sge(x,z) \ne \varnothing.\end{equation*}
   Thus we have:

\begin{align}
	 & \bUmon \mathbb{E}[Y \mid X, \hat{Z} = \pi_e(X), M = f_M(X,\pi_e), Z=z] \nonumber\\
	 & \leq \bUmon \mathbb{E}[Y|X,\Pi \in \Sgetight(X,Z), Z=z]\label{eq:ucaseupper2}
\end{align}
\cref{eq:ucaseupper2} holds by \cref{assumption:monotonicity} and the definition of $\Sgetight(x,z)$ and conditional independence. \\\\

  \cref{eq:ucase_no_agree_optimal} We recall that the upper selector takes the value $\lcor$ exactly when
  \begin{equation*}\bsU=\lcor:\quad \pi_e(x) \ne z,\ \Sge(x,z)=\varnothing,\ \Scorrect(x,z)\ne \varnothing.\end{equation*}
  Thus, for all $i\in \Hclass$, we have the following: 
\begin{align}
	\bUcor \mathbb{E}[Y \mid X, \hat{Z} = \pi_e(X), M = f_M(X,\pi_e), Z=z]
	 & = \bUcor \mathbb{E}[Y|X,\hat{Z} \ne z, M = f_M(X,\pi_e), Z=z]                       \label{eq:ucase_no_agree_optimal1} \\
	 & \leq \bUcor \min_{i\in \Scorrect(X,Z)}\mathbb{E}[Y|X,\Pi = i, Z=z]\label{eq:ucase_no_agree_optimal4}
\end{align}
\cref{eq:ucase_no_agree_optimal1} holds by implication of the indicators. \cref{eq:ucase_no_agree_optimal4} follows from \cref{assumptions:correct}, since $\pi_e(x)\ne z$ while every $\pi_i \in \Scorrect(X,Z)$, we have $\pi_i(x) = z$. \\

  \cref{eq:ucase_otherwise} Holds by the definition of $Y_{max}$. \\

Thus, we have that the expected value of the outcome under the untrialed model is less than or equal to the upper bound.
\newpage

\section{Proof of Theorem 5.1} 
\label{app:estimator}
\MainResult*

\begin{proof}
We first investigate the lower bound. Recall the following lines (\cref{eq:d_hat}, \cref{eq:lambda_hat}, \cref{eq:S_tilde_leq_hat}, \cref{eq:phi_hat}).
We also define $\hat{\mu}_{\cS}(X,Z)$ as the estimator for $E[Y|X,\Pi\in \cS,Z]=\mu_{\cS}(X,Z)$, and $\Sdisagree(x,z)$ is the set of policies, indexed by $i$, whose outputs do not agree with the true label at $x$.\\

Recall that conceptually, $d_L(X,Z)$ is the index of the model with the highest expected outcome among those that do not agree with $\pi_e$ at $X$. $\Sletight(x,z)$ is the best-performing set of agreeing models with performance less than or equal to the untrialed model. ${\varphi}_L\bigl(X,Z; \hat{d}_L\bigr)$ is the estimator for the outcome under the model with index $d_L(X,Z)$.

First, we will demonstrate that $\hat{L}(\pi_e)$ is a consistent estimator of $L(\pi_e)$, as follows. Recall our definition of the estimator from \cref{def:estimator}, using the sets defined in \cref{def:partitions} and \cref{thmt@@BoundsTheorem}:
\begin{align}
\widehat{L}(\pi_e)=
\mathbb{E}_n\left[\widehat{\psi}_L(Y_i, X_i, \Pi_i, Z_i; \widehat{f}_M)\right]
&=
\mathbb{E}_n\!\Big[
    \bLneuhat\, \hat{h}(X,Y,\Sagree)
    \label{eq:myref}
    \\[4pt]
&\quad
    {}+\,\bLmonhat\, \hat{h}(X,Y,\Sletighthat)
    \label{eq:myref2}
\\[4pt]
&\quad
    {}+\,\bLcorhat\, {\varphi_L}(X,Z;\hat d)
    \nonumber
\\[4pt]
&\quad
    {}+\,\bLdefhat\, Y_{\min}
\Big].
    \nonumber
\end{align}
Where we use the following substitution for brevity: 
\begin{align*}
	\hat{h}(X,Y, {\cS}) &:= \left(
        \widehat{\mu}_{{\cS}}(X,Z)
        +
        \frac{\mathbf{1}\{\Pi \in {\cS}\}}{P(\Pi \in \cS\mid X)}
        \bigl(Y-\widehat{\mu}_{{\cS}}(X,Z)\bigr)
    \right)
\end{align*}
Where we use $\Pi$ to denote the random variable and $\Hclass$ as its support. $\cS$ is used to denote a subset of $\Hclass$, and $\pi$ to denote an instance of $\Pi$. (See \cref{def:partitions} where we defined notation).

We first demonstrate the following Lemma, in which ${h}$, ${\varphi_L}$, and $d_L$ denote the functionals of \cref{def:est_defs,def:estimator} evaluated with the true nuisance functions $\mu_j$ and true performance $f_M$ in place of their estimates:
\begin{restatable}[Unbiased Estimator]{lemma}{UnbiasedLemma}\label{lem:unbiased}
\begin{align}
{L}(\pi_e)=
\mathbb{E}\left[{\psi}_L(Y_i, X_i, \Pi_i, Z_i; f_M)\right]
&=
\mathbb{E}\!\Big[
    \bLneu\, {h}(X,Y,\Sagree)
    \nonumber
    \\[4pt] 
&\quad
    {}+\,\bLmon\, {h}(X,Y,\Sletight)
    \nonumber
\\[4pt]
&\quad
    {}+\,\bLcor\, {\varphi_L}(X,Z;d)
    \nonumber
\\[4pt]
&\quad  
    {}+\,\bLdef\, Y_{\min}
\Big].
    \nonumber
\end{align}
\end{restatable}
By linearity of expectation, we have: 
\begin{align}
&=
\mathbb{E}\!\Big[
    \bLneu\, {h}(X,Y,\Sagree)\Big]
    \label{eq:lote1}
    \\[4pt] 
&\quad
    {}+\mathbb{E}\!\Big[\,\bLmon\, {h}(X,Y,\Sletight)\Big]
    \label{eq:lote2}
\\[4pt] 
&\quad
    {}+\mathbb{E}\!\Big[\,\bLcor\, {\varphi_L}(X,Z;d)\Big]
    \label{eq:lote3}
\\[4pt]
&\quad
    {}+\mathbb{E}\!\Big[\,\bLdef\, Y_{\min}
\Big].
    \label{eq:lote4}
\end{align}
For \cref{eq:lote1}-\cref{eq:lote3}, each term consists of a regression term and a correction term as follows: 
\begin{equation}
    \mathbb{E}[\mathbf{1}\{\cdot\}(\mu(X,Z) + \lambda(X,Z))]
\end{equation}
$\mu$ is defined as the corresponding conditional expectation term in \cref{def:bounds}. For $\lambda$:
\begin{align*}
    \mathbb{E}[\mathbf{1}\{\cdot\}\lambda(X,Z)]&=\mathbb{E}[\mathbb{E}[\mathbf{1}\{\cdot\}\lambda(X,Z)\mid X,Z]]\\
    &=\mathbb{E}[\mathbf{1}\{\cdot\}\mathbb{E}[\lambda(X,Z)\mid X,Z]]\\
    &=\mathbb{E}[\mathbf{1}\{\cdot\}0]\\
    &=0
\end{align*}
Where the first line follows by the tower property and the fact that the branch selector $\mathbf{1}\{\cdot\}$ is measurable with respect to $(X,Z)$. The second line follows from condition 2 of \cref{lem:byunv}.
Thus, we have:
\begin{align*}
    {L}(\pi_e)=
\mathbb{E}\left[{\psi}_L(Y_i, X_i, \Pi_i, Z_i; f_M)\right]
&=
\mathbb{E}\!\Big[
    \bLneu\, {\mu}_{{\Sagree}}(X,Z)
    \nonumber
    \\[4pt]
&\quad
    {}+\,\bLmon\, {\mu}_{{\Sletight}}(X,Z)
    \nonumber
\\[4pt]
&\quad
    {}+\,\bLcor\, {\mu_{d_L(X,Z)}}(X,Z)
    \nonumber
\\[4pt]
&\quad
    {}+\,\bLdef\, Y_{\min}
\Big].
    \nonumber
\end{align*}
Which is precisely $L(\pi_e)$. The proof for the upper bound is analogous.

We have (omitting the arguments $(Y_i, X_i, \Pi_i, Z_i)$ for brevity, retaining only the performance function after the semicolon; the unhatted $\psi_L(\,\cdot\,; f_M)$ denotes the pseudo-outcome computed with the true nuisance functions and true performance):
\begin{align}
	\widehat{L}(\pi_e)- L(\pi_e) & = \mathbb{E}_n\left[\widehat{\psi}_L(Y_i, X_i, \Pi_i, Z_i; \widehat{f}_M)\right] - \mathbb{E}\left[\psi_L(Y, X, \Pi, Z; f_M)\right]              \label{lemma_use}\\
	                             & = \underbrace{\mathbb{E}_n\left[\widehat{\psi}_L(\,\cdot\,; \widehat{f}_M)\right] -\mathbb{E}_n\left[\widehat{\psi}_L(\,\cdot\,; f_M)\right]}_{R_1}
	+\underbrace{\mathbb{E}_n\left[\widehat{\psi}_L(\,\cdot\,; f_M)\right]- \mathbb{E}\left[\psi_L(\,\cdot\,; f_M)\right]}_{R_2}\nonumber
\end{align}
Where \cref{lemma_use} follows from \cref{lem:unbiased}.

We will demonstrate
\begin{align*}
   \widehat{L}(\pi_e)- L(\pi_e) &= R_1+R_2\\
   &= O_P(n^{-1/2})
\end{align*}
We will focus on $R_2$ first.\\
\begin{align*}
	R_2 
	= & \underbrace{(\mathbb{E}_n[\bLneu\hat{h}(X,Y,\Sagree)] - \mathbb{E}[\bLneu\mathbb{E}[Y|X,\Pi\in \Sagree,Z]])    }_{R_{2,1}}                                                                                                                                                                                                                                \\
	    & +     \underbrace{(\mathbb{E}_n[\bLmon\hat{h}(X,Y,\Sletight)] - \mathbb{E}[\bLmon\mathbb{E}[Y|X,\Pi \in \Sletight,Z]])}_{R_{2,2}}                                                                                                                                     \\
	    & + \underbrace{(\mathbb{E}_n[\bLcor{\varphi_L}(X,Z; \hat{d})] - \mathbb{E}[\bLcor\max_{i\in\Sdisagree(X,Z)}\mathbb{E}[Y|X,\Pi = i,Z]])}_{R_{2,3}}                                                                                             \\
	    & + \underbrace{Y_{min}(\mathbb{E}_n[\bLdef] - \mathbb{E}[\bLdef])   }_{R_{2,4}}  \\
	    & = R_{2,1} + R_{2,2} + R_{2,3} + R_{2,4}
\end{align*}
We will leave $R_{2,4}$ as is. We will consider $R_{2,3}$ first. We note that the proofs for $R_{2,1}, R_{2,2}$ are identical except for the sets $\cS$ being considered, as they can be interpreted as being a maximum over a subset of $\Hclass$ that contains one element. For example: 
\begin{equation}
   \mathbb{E}[\bLneu\mathbb{E}[Y|X,\Pi\in \Sagree,Z]] = \mathbb{E}[\bLneu\max_{j\in \{1\}}\mathbb{E}[Y|X,\Pi\in \Sagree,Z]]
\end{equation}

The term $R_{2,3}$ corresponds to the following branch in the lower bound:
\begin{equation*}\mathbb{E}[\bLcor\max_{i\in\Sdisagree(X,Z)}\mathbb{E}[Y|X,\Pi = i,Z]]\end{equation*}
We note this term has an \enquote{expectation of max} form. This structure has non-smooth properties that prevent straightforward estimation via plugins or other methods. As a result, we can utilize an augmented inverse probability-weighted estimator. Previous literature exists on similar topics. For example, \cite{byun2024} considers a problem of using AIPW to estimate the expected maximum of a set of functions. We will adapt their results to our setting. 

\begin{definition}\label{def:byun_class}
	Consider the following class of estimators.  Let
   $\theta_j(W;P), \lambda_j(W;P)$ for all $j\in \{1, \dots, J\}$ represent some arbitrary set of smooth functions and their associated influence functions, where $W$ is the observed data and $P$ is the true data distribution. Use of $\widehat{P}$ indicates the parameter was taken with respect to an estimated distribution. We further define the following functionals:  
     \begin{align}
         \theta_j & := \mathbb{E}[\theta_j(W;P)] \label{eq:theta_j}\\
         \widehat{\theta}_j & := \mathbb{E}_n[\theta_j(W;\widehat{P})] \label{eq:theta_hat_j}\\
         \psi & := \mathbb{E}\left[\max_{j\in \{1, \dots, J\}} \theta_j(W;P) \right] = \mathbb{E}\left[\theta_{d(W)}(W;P)\right]\label{eq:psi}\\
         d(W) &:= \arg\max_{j\in \{1, \dots, J\}} \theta_j(W;P) \label{eq:d}\\
         \widehat{\psi}_j & := \mathbb{E}_n[\theta_j(W;\widehat{P}) + \lambda_j(W;\widehat{P})] \label{eq:psi_hat_j}\\
         \widehat{\psi} & := \mathbb{E}_n[\varphi(W;\widehat{P}, \widehat{d})] \label{eq:psi_hat}\\ 
         \varphi(W;\widehat{P}, \widehat{d}) & := \theta_{\widehat{d}(W)}(W;\widehat{P}) + \lambda_{\widehat{d}(W)}(W;\widehat{P}) \label{eq:varphi_def}\\
         \widehat{d}(W) &:= \arg\max_{j\in \{1, \dots, J\}} \theta_j(W;\widehat{P}) \label{eq:d_hat2}
      \end{align}

\end{definition}
We note that these definitions are copied verbatim from \citet{byun2024}. We can now state their main lemma:
\begin{restatable}[Lemma 10 and 11 of \citet{byun2024}]{lemma}{ByunLemmaVerbatim}\label{lem:byunv}
	{For the class of estimators defined in \cref{def:byun_class}, if the following conditions hold}:  
\begin{enumerate}
   \item For every $j \in \{1, \dots, J\}, \lambda_j(W;P)$ and $\theta_j(W;P)$ are both uniformly bounded by constants with respect to $n$.
   \item For every \( j \in \{1, \dots, J\} \), \(\mathbb{E}[\lambda_j(W;P)]=0\)
   \item For every \( j \in \{1, \dots, J\} \), \(\|\widehat{\theta}_j- \theta_j\| = o_P(1)\)
   \item There exists \( \alpha > 0 \) such that: \(P\left[\min_{j\ne d(W)}|\theta_{d(W)}(W)-\theta_j(W)|\leq t\right]\lesssim t^{\alpha}\)
   \item In \cref{eq:psi_hat}, the expectation is taken with respect to \(\widehat{P}_1\), while the estimator \(\varphi(W;\widehat{P}_2, \widehat{d})\) uses an independent sample \(\widehat{P}_2\).
\end{enumerate}  
Then the following is true: 
\begin{align}
   \widehat{\psi}-\psi &= \mathbb{E}_n[\varphi(W;\ P, d)] - \mathbb{E}[\varphi(W; P, d)] \\
   &\quad + O_P\left(\|\widehat{\theta}_j- \theta_j\|_{\infty}^{1+\alpha} + \max_{j\in \{1, \dots, J\}}\mathbb{E}\left[\theta_j(W;\widehat{P}) + \lambda_j(W;\widehat{P}) - \theta_j(W;P)\right]\right) \\
   &\quad + o_P(n^{-1/2})
\end{align}
If two additional conditions hold:
\begin{enumerate}
   \item[6.] \(\|\widehat{\theta}_j- \theta_j\|_{\infty}^{1+\alpha} = o_P(n^{-1/2})\)
   \item[7.] For each \( j \in \{1, \dots, J\} \), \(\mathbb{E}[\theta_j(W;\widehat{P}) + \lambda_j(W;\widehat{P}) - \theta_j(W;P)] = o_P(n^{-1/2})\)
\end{enumerate}
Then the following simplification holds:
\begin{align}
   \widehat{\psi}-\psi &= \mathbb{E}_n[\varphi(W;\ P, d)] - \mathbb{E}[\varphi(W; P, d)] + o_P(n^{-1/2})
\end{align}
And the error is asymptotically normal: 
\begin{equation}
   \sqrt{n}(\widehat{\psi}-\psi) \xrightarrow{d} N(0, Var(\varphi(W; P, d)))
\end{equation}
\end{restatable}
We recognize that our problem fits into this framework with the following identifications:
\begin{align*}
    W & \leftrightarrow (X,Z)\\
   \theta_j(W;P) & \leftrightarrow {\mu}_j(X,Z)\\
   \widehat{\theta}_j(W;\widehat{P}) & \leftrightarrow \hat{\mu}_j(X,Z)\\
   \theta_j & \leftrightarrow \mathbb{E}[{\mu}_j(X,Z)]\\
   \widehat{\theta_j} & \leftrightarrow \mathbb{E}_n[\hat{\mu}_j(X,Z)]\\
   \lambda_j(W;P) & \leftrightarrow \lambda_j(X,Z)\\
   d(W) & \leftrightarrow d_L(X,Z)\\
   \widehat{d}(W) & \leftrightarrow \hat{d}(X,Z)\\
   J & \leftrightarrow \Sdisagree(x,z)
\end{align*}
Where any unspecified functionals are defined analogously to their counterparts in \cref{eq:theta_j} - \cref{eq:d_hat2}. Here the conditioning argument $W$ of \citet{byun2024} is identified with the pair $(X,Z)$, since each branch of \cref{def:bounds} conditions on $(X,Z)$. Because model assignment is randomized, $\Pi \perp\!\!\!\perp Z \mid X$ and hence $P(\Pi=j\mid X,Z) = P(\Pi=j\mid X)$, so the propensity nuisance is unchanged and only the outcome regression $\mu_j$ acquires the $Z$ argument.
Thus, we can apply \cref{lem:byunv} to our setting,
\begin{restatable}{lemma}{ByunLemma}
\label{lem:byun}
Suppose the following conditions hold:
\begin{enumerate}
   \item For every $j \in \Sdisagree(x,z)$, $\lambda_j(X,Z)$ and $\hat{\mu}_j(X,Z)$ are both uniformly bounded by constants.
   \item $\mathbb{E}[\lambda_j(X,Z)\mid X,Z] = 0$ for all $j \in \Sdisagree(x,z).$
   \item For every \( j \in \Sdisagree(x,z) \),
\[
| \mathbb{E}_n[\hat{\mu}_j(X,Z)] - \mathbb{E}[{\mu}_j(X,Z)] | = o_P(1)
\]
   \item There exists \( \alpha > 0 \) such that:
\[
P\left[\min_{j\ne d_L(X,Z)}\left|\mu_{d_L(X,Z)}(X,Z)-{\mu}_{j}(X,Z)\right|\leq t\right] \leq t^{\alpha}
\]
   \item Independent samples are used for $\mathbb{E}_n\left[\widehat{\psi}_L(Y_i, X_i, \Pi_i, Z_i; \widehat{f}_M)\right]$ and ${\varphi_L}(X,Z; \hat{d})$.
\end{enumerate}
Then the difference between the estimator and the true parameter can be expressed as:
\begin{align*}
	\widehat{\psi}^l-\psi^l & = \mathbb{E}_n[{\varphi_L}(X,Z; d)] - \mathbb{E}[{\varphi_L}(X,Z; d)]                                                                                                                                    \\
	                        & + O_P\left(\| \mathbb{E}_n\left[\hat{\mu}_j(X,Z)\right] - \mathbb{E}[{\mu}_j(X,Z)]\|_{\infty}^{1+\alpha} + \max_{j\in \Sdisagree(x,z)}\mathbb{E}\left[\widehat{\mu}_{j}(X,Z) + \hat{\lambda}_{j}(X,Z) - {\mu}_{j}(X,Z)\right]\right) \\
	                        & + o_P(n^{-1/2})
\end{align*}
We can further simplify this with the following 2 additional conditions:
\begin{enumerate}
   \item[6.] The following holds:\[
\| \mathbb{E}_n[\hat{\mu}_j(X,Z)] - \mathbb{E}[{\mu}_j(X,Z)] \|_\infty^{1+\alpha}
= o_P(n^{-1/2})\]
$\quad \text{where } \alpha \text{ is the same as in \cref{assump:margin}}.
$
   \item[7.] For each \( j \in \Sdisagree(x,z) \),
\[
\mathbb{E}[\hat{\mu}_j(X,Z) + \hat{\lambda}_j(X,Z) - {\mu}_j(X,Z)] = o_P(n^{-1/2}).
\]
\end{enumerate}
Thus the following holds: 
\begin{align*}
	\widehat{\psi}^l-\psi^l & = \mathbb{E}_n[{\varphi_L}(X,Z; d)] - \mathbb{E}[{\varphi_L}(X,Z; d)] + o_P(n^{-1/2})
\end{align*}
\end{restatable}

We will demonstrate the conditions hold for $R_{2,3}$. 
	      \begin{enumerate}
		      \item (Boundedness): For every $j\in \Sdisagree(x,z)$, $\lambda_{j}(X,Z)$ and $ \hat{\mu}_j(X,Z)$ are both uniformly bounded by constants. \\

		            This is true by \cref{thmt@@BoundedOutcomes} and can be enforced by clipping the value of the nuisance function. \\
		      \item (Zero-mean Correction Term): $\mathbb{E}[\lambda_{j}(X,Z)|X,Z] = 0$ for all $j\in \Sdisagree(x,z)$. \\
                  \begin{align}
                  \mathbb{E}[\lambda_{j}(X,Z)|X,Z]
                                        & = \mathbb{E}\left[\frac{\mathbf{1}\{\Pi = j\}}{P(\Pi=j|X)}(Y- {g}_j(X,Z))\Big|X,Z\right]\label{eq:biascorr-47} \\
                                        & = \frac{\mathbb{E}\left[\mathbf{1}\{\Pi= j\}(Y- {g}_j(X,Z))|X,Z\right]}{P(\Pi= j|X)} \label{eq:biascorr-48} \\
                                        & = \frac{P\left(\Pi= j | X,Z\right)\left(\mathbb{E}\left[Y|X,\Pi= j,Z\right] - {g}_j(X,Z)\right)}{P(\Pi= j|X)} \label{eq:biascorr-49} \\
                                        & = \left(\mathbb{E}\left[Y|X,\Pi= j,Z\right] - {g}_j(X,Z)\right) \label{eq:biascorr-50} \\
                                        & = 0\label{eq:biascorr-51}
                  \end{align}
                  \cref{eq:biascorr-47} follows by the definition of $\lambda_{j}(X,Z)$. \cref{eq:biascorr-48} follows because the propensity is constant under the conditioning. \cref{eq:biascorr-49} follows because the expectation of the indicator is the probability of the event, and ${g}_j(X,Z)$ is a constant under the conditions. \cref{eq:biascorr-50} follows because randomized model assignment gives $\Pi \perp\!\!\!\perp Z \mid X$, so $P(\Pi= j|X,Z) = P(\Pi= j|X)$ and the ratio of propensities is one. \cref{eq:biascorr-51} follows by the definition of ${g}_j(X,Z)$ as a plug-in estimator for $\mathbb{E}[Y|X,\Pi= j,Z]$. \\
		      \item (Consistent Plug-in Estimator): For every $j\in \Sdisagree(x,z)$, \begin{equation*}\| \mathbb{E}_n\left[\hat{\mu}_j(X,Z)\right] - \mathbb{E}[{\mu}_j(X,Z)]\| = o_P(1)\end{equation*}\\
		            This will be satisfied by the validation of Condition 7, so discussion is deferred to there. \\
		      \item (Margin Condition): There exists $\alpha>0$ such that:
		            \begin{equation*}P\left[\min_{j\ne d_L(X,Z)}\left|\mu_{d_L(X,Z)}(X,Z)-{\mu}_{j}(X,Z)\right|\leq t\right] \leq t^{\alpha}\end{equation*}
		            This condition asserts that the probability of a small gap between the outcome regressions $\mu_j$ of the best two models is bounded polynomially. This holds by the first margin condition of \cref{assump:margin}.\\
		      \item (Independent Samples):\\

		            This holds by construction. As described in \cref{app:experiments}, data are split into $K\geq 2$ folds, with all nuisance functions ($\widehat{\mu}_j$ and $\widehat{f}_M$, etc.) fit on the first and the bounds estimated on the second. \\ \\
		      \item (Convergence of plug-in estimators)
		            \begin{equation*}\| \mathbb{E}_n\left[\hat{\mu}_j(X,Z)\right] - \mathbb{E}[{\mu}_j(X,Z)]\|_{\infty}^{1+\alpha}=o_P(n^{-\frac{1}{2}}), \text{where $\alpha$ is the same as in \cref{assump:margin}}. \end{equation*}
		            This is stating that the maximum difference between the empirical expectation of our machine learning-derived estimate from the true expected outcome, across all $j$, converges at $o_P(n^{-\frac{1}{2}})$ rate. We will prove this by demonstrating that all plug-in estimators converge sufficiently fast. \\

		            We want to show:
		            \begin{equation*}\forall j, \|\mathbb{E}_n\left[\hat{\mu}_j(X,Z)\right] - \mathbb{E}[{\mu}_j(X,Z)]\|^{1+\alpha} = o_P(n^{-\frac{1}{2}})\end{equation*}

		            Consider the following:

		            For each $j\in \Sdisagree(x,z)$, we have:
		            \begin{align}
			            \mathbb{E}_n\left[\hat{\mu}_j(X,Z)\right] - \mathbb{E}[{\mu}_j(X,Z)] & = \mathbb{E}_n\left[\hat{\mu}_j(X,Z)\right] - \mathbb{E}\left[\hat{\mu}_j(X,Z) \right] + \mathbb{E}\left[\hat{\mu}_j(X,Z)\right] - \mathbb{E}[{\mu}_j(X,Z)]     \\
			                                                                       & = (\mathbb{E}_n\left[\hat{\mu}_j(X,Z)\right] - \mathbb{E}\left[\hat{\mu}_j(X,Z) \right]) + (\mathbb{E}\left[\hat{\mu}_j(X,Z)\right] - \mathbb{E}[{\mu}_j(X,Z)])
		            \end{align}
		            Looking solely at the first term, we have:
		            \begin{align}
			            \mathbb{E}_n\left[\hat{\mu}_j(X,Z)\right] - \mathbb{E}\left[\hat{\mu}_j(X,Z) \right] & = \frac{1}{n}\sum_{i=1}^{n}\left(\hat{g}_j(X_i,Z_i) - \mathbb{E}\left[\hat{g}_j(X_i,Z_i)\right]\right) \\
			            \sqrt{n}\left(\mathbb{E}_n\left[\hat{\mu}_j(X,Z)\right] - \mathbb{E}\left[\hat{\mu}_j(X,Z) \right]\right) & \xrightarrow{d} N(0, \sigma^2)                                                       \\
			            \implies \mathbb{E}_n\left[\hat{\mu}_j(X,Z)\right] - \mathbb{E}\left[\hat{\mu}_j(X,Z) \right] & = O_p(n^{-\frac{1}{2}})
		            \end{align}
		            Which follows from the CLT. \\

		            The second term is the bias of the machine learning estimator. In order for the whole norm to converge at $o_P(n^{-\frac{1}{2}})$, we require this term to converge at a rate of $o_P(n^{-\frac{1}{2(1+\alpha)}})$. For a reasonable $\alpha$, such as $\alpha=1$, this is satisfied by a convergence rate of $o_P(n^{-\frac{1}{4}})$. It is reasonable to assume our model converges at this rate. \\
		            We now have:
\begin{align}
   \|O(n^{-\frac{1}{2}}) + o_P(n^{-\frac{1}{2(1+\alpha)}})\|^{1+\alpha}
     & \leq \big(\|O(n^{-\frac{1}{2}})\| + \|o_P(n^{-\frac{1}{2(1+\alpha)}})\|\big)^{1+\alpha} \label{eq:tri}\\
     & \leq 2^{\alpha}\big(\|O(n^{-\frac{1}{2}})\|^{1+\alpha} + \|o_P(n^{-\frac{1}{2(1+\alpha)}})\|^{1+\alpha}\big) \label{eq:cauchy}\\
     & = O(n^{-\frac{1+\alpha}{2}}) + o_P(n^{-\frac{1}{2}}) \label{eq:alpha}\\
     & = o_P(n^{-\frac{1}{2}}) \label{eq:alpha2}
\end{align}
		            Where \cref{eq:tri} follows from the triangle inequality. \cref{eq:cauchy} follows from the $C_r$ inequality with $r=1+\alpha$ and the constant gets absorbed. \cref{eq:alpha} evaluates the rates. \cref{eq:alpha2} since $\alpha>0$. Since this holds for all $j\in \Sdisagree(x,z)$, we have shown our result. \\

		      \item (Convergence of bias-corrected estimator) For each \begin{equation*}j\in \Sdisagree(x,z), \mathbb{E}\left[\widehat{\mu}_{j}(X,Z) + \hat{\lambda}_{j}(X,Z) - {\mu}_{j}(X,Z)\right] = o_P(n^{-\frac{1}{2}}).\end{equation*}
		            Where $\Sdisagree(x,z)$ is the set of models in the branch. \\

		            For clarity, we rewrite the full bias-corrected estimator as:
		            \begin{equation*}{\varphi_L}(X,Z; \hat{d}) = \widehat{\mu}_{\hat{d}}(X,Z) + \frac{\mathbf{1}\left\{\pi=\hat{d}\right\}}{P(\Pi= \hat{d}|X)}(Y-\widehat{\mu}_{\hat{d}}(X,Z))\end{equation*}
		            Thus, we have the following nuisance functions:
		            \begin{equation*}{\mu}_{j}(X,Z), \quad P(\Pi= j|X)\end{equation*}
		            We will assume that ${\mu}_j(X,Z)$ converges at a rate of $o_P(n^{-\frac{1}{4}})$. We also assume that $P(\Pi= j|X)$ is known. \\
		            We have, for arbitrary $j\in \Sdisagree(x,z)$:
		            \begin{align}
			            \mathbb{E}\left[\widehat{\mu}_{j}(X,Z) + \hat{\lambda}_{j}(X,Z) - {\mu}_{j}(X,Z)\right]
			             & = \mathbb{E}\left[\widehat{\mu}_{j}(X,Z)-{\mu}_{j}(X,Z) + \frac{\mathbf{1}\left\{\Pi= j\right\}}{{P}(\Pi= j|X)}(Y-\widehat{\mu}_{j}(X,Z)) \right]  \label{eq:biascorr-1}                             \\
			             & = \mathbb{E}\left[\mathbb{E}\left[\widehat{\mu}_{j}(X,Z)-{\mu}_{j}(X,Z) + \frac{\mathbf{1}\left\{\Pi= j\right\}}{{P}(\Pi= j|X)}(Y-\widehat{\mu}_{j}(X,Z))\mid X,Z\right]\right]  \label{eq:biascorr-2} \\
			             & = \mathbb{E}\left[\widehat{\mu}_{j}(X,Z)-{\mu}_{j}(X,Z)+\frac{1}{{P}(\Pi= j|X)}\mathbb{E}\left[\mathbf{1}\left\{\Pi= j\right\}(Y-\widehat{\mu}_{j}(X,Z)) \mid X,Z\right]\right] \label{eq:biascorr-3}  \\
			             & = \mathbb{E}\left[\widehat{\mu}_{j}(X,Z)-{\mu}_{j}(X,Z)+\frac{P\left(\Pi= j|X,Z\right)}{{P}(\Pi= j|X)}\mathbb{E}\left[(Y-\widehat{\mu}_{j}(X,Z)) \mid X, \Pi= j, Z\right]\right]  \label{eq:biascorr-4}      \\
			             & = \mathbb{E}\left[\widehat{\mu}_{j}(X,Z)-{\mu}_{j}(X,Z)+\frac{P\left(\Pi= j|X,Z\right)}{{P}(\Pi= j|X)}\left({\mu}_j(X,Z)-\widehat{\mu}_{j}(X,Z)\right)\right] \label{eq:biascorr-6}                          \\
			             & = \mathbb{E}\left[\left(\widehat{\mu}_{j}(X,Z)-{\mu}_{j}(X,Z)\right)\left(1-\frac{P\left(\Pi= j|X,Z\right)}{{P}(\Pi= j|X)}\right)\right] \label{eq:biascorr-7}                                 \\
			             & = \mathbb{E}\left[\left(\widehat{\mu}_{j}(X,Z)-{\mu}_{j}(X,Z)\right)\left(1-1\right)\right] \label{eq:biascorr-7b}                                                                     \\
			             & = 0                                                                                                                                                                      \\
			             & = o_P(n^{-\frac{1}{2}}) \label{eq:biascorr-13}
		            \end{align}
		            \cref{eq:biascorr-1} follows from \cref{assump:known_prop}. \cref{eq:biascorr-2} follows from the law of iterated expectation. \cref{eq:biascorr-3} follows because the relevant expressions are constants under the conditioning. \cref{eq:biascorr-4} follows from removing the indicator from the definition of conditional expectation.
		            \cref{eq:biascorr-6} follows by the definition of ${\mu}_j(X,Z)$.
		            \cref{eq:biascorr-7} is an algebraic rearrangement; \cref{eq:biascorr-7b} uses $P(\Pi= j|X,Z) = P(\Pi= j|X)$, which holds because randomized model assignment gives $\Pi \perp\!\!\!\perp Z \mid X$. \cref{eq:biascorr-13} follows vacuously from the definition of convergence in probability.
		            \\
		            Thus, we have verified Condition 7. This also verifies Condition 3. Since the bias corrected estimator converges at a rate of $o_P(n^{-\frac{1}{2}})$ to the true value, it follows that the uncorrected plug-in estimator is at least consistent, since the correction term has mean 0. \\
	      \end{enumerate}
         		            We have thus verified all necessary conditions. As a result, the following equality holds: 
                           \begin{equation*}
                               R_{2,3} = \mathbb{E}_n[\bLcor\varphi(X,Z;\hat{d})]-\mathbb{E}[\bLcor\varphi(X,Z;\hat{d})] + o_P(n^{-1/2})
                           \end{equation*}
                           It follows that the same applies to $R_{2,1}, R_{2,2}, R_{2,4}$, since the $\hat{h}(X,Y,\cS)$ functionals are identical in structure to $\varphi(x,z;\hat{d})$, except for the subscript of $\hat{\mu}$, which instead is a single nuisance function that converges to the true value at $o_P(n^{-\frac{1}{2}})$ under \cref{asmp:convergence_nuisance}, allowing the whole function $\hat{h}(X,Y,\cS)$ to converge at $o_P(n^{-\frac{1}{2}})$ rate under analogous logic.
                           
                           Thus we have: 
                           \begin{align*}
                               R_2 =&\mathbb{E}_n[\bLneu h(X,Y,\Sagree)]-\mathbb{E}[\bLneu h(X,Y,\Sagree)] +\\
                               &\mathbb{E}_n[\bLmon h(X,Y,\Sletight)]-\mathbb{E}[\bLmon h(X,Y,\Sletight)]+ \\
                               &\mathbb{E}_n[\bLcor\varphi(X,Z;{d})]-\mathbb{E}[\bLcor\varphi(X,Z;{d})] +\\
                               &  \mathbb{E}_n[\bLdef Y_{min}]-\mathbb{E}[\bLdef Y_{min}] + o_P(n^{-1/2})
                           \end{align*}
We will now consider $R_1 = \mathbb{E}_n\left[\widehat{\psi}_L(\,\cdot\,; \widehat{f}_M)\right] -\mathbb{E}_n\left[\widehat{\psi}_L(\,\cdot\,; f_M)\right]$:\\\\
The theoretical construction of ${\psi}_L$ involves the construction of the set $\Sle(x,z)$, defined in \cref{def:partitions} as:
\begin{align*}
\Sle(x,z) &= \{\pi_i \in \Sagree(x) \mid f_M(x, \pi_i) \leq f_M(x,\pi_e),\ \pi_e(x)=z\}\\
&\quad\cup\{\pi_i \in \Sagree(x) \mid f_M(x, \pi_i) \geq f_M(x,\pi_e),\ \pi_e(x)\ne z\}
\end{align*}
Where $f_M(x, \pi_i)$ is the true performance of model $\pi_i$ at covariate $x$. If we instead use the estimated performance $\widehat{f}_M$ to construct the set, as in \cref{def:est_defs}, we have:
\begin{align*}
\Slehat(x,z) &= \{\pi_i \in \Sagree(x) \mid \widehat{f}_M(x, \pi_i) \leq \widehat{f}_M(x,\pi_e),\ \pi_e(x)=z\}\\
&\quad\cup\{\pi_i \in \Sagree(x) \mid \widehat{f}_M(x, \pi_i) \geq \widehat{f}_M(x,\pi_e),\ \pi_e(x)\ne z\}
\end{align*}
However, we only concern ourselves with the model(s) whose performance is closest to that of the untrialed model, per \cref{def:partitions,def:est_defs}:
\begin{equation*}\Sletight(x,z) = \argmin_{\pi_i\in \Sle(x,z)} |f_M(x, \pi_i)-f_M(x, \pi_e)|\end{equation*}
\begin{equation*}\Sletighthat(x,z) = \argmin_{\pi_i\in \Slehat(x,z)} |\widehat{f}_M(x, \pi_i)-\widehat{f}_M(x, \pi_e)|\end{equation*}
However, error in our estimation may permute the ordering of the models relative to the untrialed model, in terms of their estimated performance, resulting in:
\begin{equation*}\Sletighthat(x,z) \ne \Sletight(x,z)\end{equation*}
This leads to two distinct cases:
\begin{enumerate}
	\item Case 1: $\Sletighthat(x,z) = \Sletight(x,z)$
	\item Case 2: $\Sletighthat(x,z) \ne \Sletight(x,z)$
\end{enumerate}
In case 1, we have that:
\begin{equation*}\widehat{\psi}_L(\,\cdot\,; \widehat{f}_M) = \widehat{\psi}_L(\,\cdot\,; f_M)\end{equation*}
Since $\widehat{\psi}_L$ depends on the performance function only through the set $\Sletight(x,z)$. If the estimates have no impact on the set, then the estimates have no impact on the value of $\widehat{\psi}_L$.
In case 2, we know that the error is bounded: $\left|\widehat{\psi}_L(\,\cdot\,; f_M)-\widehat{\psi}_L(\,\cdot\,; \widehat{f}_M)\right|\leq \left|Y_{max}-Y_{min}\right|=B$, by \cref{assumption:bounded}.\footnote{This assumption is technically stronger than necessary, as only sufficient separation between the performance of the two closest models with $\pi_e$ on both sides is necessary. }

To prove convergence, we utilize the second margin condition of \cref{assump:margin}: there exists $\beta > 0$ such that for any $t\geq 0$, we have:
\begin{equation*}P\left[\min_{i\ne j}\left|f_M(x;\pi_j)-f_M(x;\pi_i)\right|\leq t\right]\leq t^\beta\end{equation*}
For all $x \in \mathcal{X}$ and $\pi_i,\pi_j\in\Hclass \cup\{\pi_e\} $. \\

  This condition states that the probability that the minimum performance gap between any two models is less than $t$ decays at least polynomially in $t$.
Under this assumption, we can show that the first term converges to 0 as $n\to \infty$ by demonstrating the indicator converges to 0:\\\\
We first define: 
\begin{equation*}
    \delta_n := 2\max_{i}\left|\left| f_M(x;\pi_i)-\widehat{f}_M(x;\pi_i)\right|\right|_{\infty}
\end{equation*}
\begin{equation*}
    g(X) := \min_{i,j}\left|{f}_M(X;\pi_i)-  f_{M}( X;\pi_j)\right|
\end{equation*}

It follows that:
\begin{align}
	|R_1| & \leq \mathbb{E}_n\left[ |\widehat{\psi}_L(\,\cdot\,; \widehat{f}_M) -\widehat{\psi}_L(\,\cdot\,; f_M)|\right] \label{eq:diff-1}                                                                                                                                                                       \\
	    & = \mathbb{E}_n\left[|\mathbf{1}\left\{\Sletighthat(X,Z) \ne \Sletight(X,Z)\right\} \left(\widehat{\psi}_L(\,\cdot\,; \widehat{f}_M) -\widehat{\psi}_L(\,\cdot\,; f_M)\right) |\right]\label{eq:diff-3}                                                                                        \\
	    & \leq \mathbb{E}_n\Bigg[|\mathbf{1}\left\{\min_{i,j}\left|{f}_M(X;\pi_i)-  f_{M}( X;\pi_j)\right|\leq 2\max_{i}\left|\left| f_M(x;\pi_i)-\widehat{f}_M(x;\pi_i)\right|\right|_{\infty}\right\} \nonumber\\\
       &\qquad \left(\widehat{\psi}_L(\,\cdot\,; \widehat{f}_M) -\widehat{\psi}_L(\,\cdot\,; f_M)\right) |\Bigg] \label{eq:diff-4} \\
       & = \mathbb{E}_n\Bigg[|\mathbf{1}\left\{g(X)\leq \delta_n\right\} \left(\widehat{\psi}_L(\,\cdot\,; \widehat{f}_M) -\widehat{\psi}_L(\,\cdot\,; f_M)\right) |\Bigg] \label{eq:diff-abb} \\
       &\leq \mathbb{E}_n\Bigg[|\mathbf{1}\left\{g(X)\leq \delta_n\right\} B |\Bigg] \label{eq:diff-B} \\
       & = \mathbb{E}_n\Bigg[\mathbf{1}\left\{g(X)\leq \delta_n\right\} \Bigg] B\label{eq:diff-B_aug}
\end{align}
Where $B$ is an arbitrary constant, $\pi_i,\pi_j\in\Hclass\cup \{\pi_e\}$. \cref{eq:diff-1} is Jensen's inequality. \cref{eq:diff-3} follows because the difference is only non-zero when the sets differ. \cref{eq:diff-4} follows because the sets differ only when model error causes permutations in the ordering of $\pi_e$ and the 2 closest performing models on either side. This occurs when the estimation error is larger than half of the minimum distance between the untrialed model and any other model, as maximum error in both values can \enquote{bridge the gap}. \cref{eq:diff-abb} is by the definitions of $\delta_n$ and $g(X)$. \cref{eq:diff-B} follows from the boundedness of the difference from \cref{assumption:bounded}, where $B = Y_{max}-Y_{min}$. \cref{eq:diff-B_aug} follows because the indicator is non-negative.  From here, we cannot directly apply the margin condition of \cref{assump:margin}, because the expectation is empirical; instead, we have:
\begin{align}
P\!\left(\sqrt{n}\,|R_1| > \epsilon\right)
&\;\le\; P\!\left(\mathbb{E}_n\!\bigl[\mathbf{1}\{g(X_i)\le\delta_n\}\bigr] \ge \tfrac{\epsilon}{B\sqrt{n}}\right)
   \label{eq:diff-split-2}\\
   &\;\le\; \mathbb{E}\!\left(\mathbb{E}_n\!\bigl[\mathbf{1}\{g(X_i)\le\delta_n\}\bigr] \right)\tfrac{B\sqrt{n}}{\epsilon}
   \label{eq:diff-split-markov}\\
   &\;=\; \mathbb{E}\!\left(\frac{1}{n}\sum_{i=1}^n\bigl[\mathbf{1}\{g(X_i)\le\delta_n\}\bigr] \right)\tfrac{B\sqrt{n}}{\epsilon}\label{eq:diff-split-empirical-exp}\\
      &\;=\; P\!\left(g(X_i)\le\delta_n \right)\tfrac{B\sqrt{n}}{\epsilon}\label{eq:diff-split-linearity}\\
      &\;\le\; \delta_n^{\beta}\tfrac{B\sqrt{n}}{\epsilon}\label{eq:diff-split-assp}\\
&\;\longrightarrow\; 0 . \label{eq:diff-bct}
\end{align}
 \cref{eq:diff-split-2} uses $|R_1|\le B\mathbb{E}_n[\mathbf{1}\{g(X_i)\le\delta_n\}]$. \cref{eq:diff-split-markov} follows from Markov's inequality. \Cref{eq:diff-split-empirical-exp} is by definition. \cref{eq:diff-split-linearity} is by linearity of expectation. \cref{eq:diff-split-assp} follows from the margin condition (\cref{assump:margin}), and \cref{eq:diff-bct} follows since $\delta_n = o_P(n^{-\frac{1}{2\beta}})$ by \cref{asmp:convergence_nuisance}, so $\delta_n^\beta\sqrt{n} \to 0$ in probability. Because this holds for every $\epsilon>0$, $R_1 = o_P(n^{-1/2})$.

We note that \emph{this is different} than the required convergence rate for the plug-in estimators for ${\mu}_j(X,Z)$, which was $o_P(n^{-\frac{1}{2(1+\alpha)}})$. This is because in \citet{byun2024}, or our 6th condition of \cref{thmt@@ByunLemma}, at the step analogous to \cref{eq:diff-4}, the term in the right hand side inside of the indicator is the same as the term outside of the indicator. After application of the margin condition, this results in a squared convergence rate. In our case, the term outside of the indicator is $B$, a constant, which does not converge. We are unable to improve this, since the convergence of $\widehat{f}_M$ is not necessarily related to the convergence of $\widehat{\psi}_L$, whose error is bounded by $B$.  
We have now shown:
\begin{align}
	\widehat{L}(\pi_e) - L(\pi_e) =&\mathbb{E}_n[\bLneu h(X,Y,\Sagree)]-\mathbb{E}[\bLneu h(X,Y,\Sagree)] +\nonumber\\
                               &\mathbb{E}_n[\bLmon h(X,Y,\Sletight)]-\mathbb{E}[\bLmon h(X,Y,\Sletight)] +\nonumber\\
                               &\mathbb{E}_n[\bLcor\varphi(X,Z;{d})]-\mathbb{E}[\bLcor\varphi(X,Z;{d})] +\nonumber\\
                               &  \mathbb{E}_n[\bLdef Y_{min}]-\mathbb{E}[\bLdef Y_{min}] \nonumber\\
	 & + o_P(n^{-1/2}) + \underbrace{o_P(n^{-1/2})}_{R_1}\label{eq:final-construction-split}\\
	 =&\mathbb{E}_n[\bLneu h(X,Y,\Sagree)+\bLmon h(X,Y,\Sletight)+\\
     &\bLcor\varphi(X,Z;{d})+\bLdef Y_{min}]-\\
     &\mathbb{E}[\bLneu h(X,Y,\Sagree)+\bLmon h(X,Y,\Sletight)+\\
     &\bLcor\varphi(X,Z;{d})+\bLdef Y_{min}]\\
	 & + o_P(n^{-1/2})\label{eq:final-construction}
\end{align}
In \cref{eq:final-construction-split}, the term outside the underbrace corresponds to $R_2$. In \cref{eq:final-construction}, we have the difference of an empirical mean and the corresponding true mean. By the CLT, this is a random variable that converges to a normal distribution. By Slutsky's, the $o_P(n^{-1/2})$ terms vanish asymptotically. 

It directly follows that since $\widehat{L}(\pi_e)$ is an empirical function of i.i.d.\ random variables, we can apply the CLT to show that:
\begin{align}
	\sqrt{n}(\widehat{L}(\pi_e) - L(\pi_e)) & \sim N\left(0, \sigma_L^2\right)
\end{align}
By the CLT, where $\sigma_L^2 = Var({\psi}_L)$. It follows:
\ConfidenceIntervals*

\begin{proof}
By \cref{theorem:estimator}, $\sqrt{n}(\widehat{L}(\pi_e)-L(\pi_e)) \xrightarrow{d} N(0,\sigma_L^2)$ with $\sigma_L^2 = \Var(\psi_L)$. We must show $\widehat{\Var}(\widehat{\psi}_L) \xrightarrow{P} \sigma_L^2$. Decompose
\begin{equation*}
\widehat{\Var}(\widehat{\psi}_L) - \Var(\psi_L)
= \underbrace{\widehat{\Var}(\widehat{\psi}_L) - \Var(\widehat{\psi}_L)}_{(A)}
+ \underbrace{\Var(\widehat{\psi}_L) - \Var(\psi_L)}_{(B)}.
\end{equation*}
For term $(A)$, $\widehat{\psi}_L$ is uniformly bounded by a constant independent of $n$: \cref{assumption:bounded} bounds $Y\in[Y_{\min},Y_{\max}]$, and the influence terms $\widehat{\lambda}_j$ are bounded by Condition~1 of \cref{thmt@@ByunLemma}. Under cross-fitting, the evaluation-fold observations are i.i.d. conditional on the training fold, on which $\widehat{\psi}_L$ is a fixed bounded function; a law of large numbers for this bounded conditionally-i.i.d. sample therefore gives $\widehat{\Var}(\widehat{\psi}_L)-\Var(\widehat{\psi}_L)=o_P(1)$, so $(A)=o_P(1)$.

For term $(B)$, the rate condition \cref{asmp:convergence_nuisance} and margin condition \cref{assump:margin} give $\widehat{\psi}_L \xrightarrow{P} \psi_L$ (as established in the proof of \cref{theorem:estimator}). Since $\widehat{\psi}_L$ is uniformly bounded, the bounded convergence theorem upgrades convergence in probability to moment convergence, $\mathbb{E}[\widehat{\psi}_L]\to\mathbb{E}[\psi_L]$ and $\mathbb{E}[\widehat{\psi}_L^2]\to\mathbb{E}[\psi_L^2]$, hence $\Var(\widehat{\psi}_L)\to\Var(\psi_L)$ and $(B)=o_P(1)$.

Thus, $\widehat{\Var}(\widehat{\psi}_L)\xrightarrow{P}\sigma_L^2$, and by Slutsky's theorem the interval attains asymptotic coverage $1-\gamma$. 
\end{proof}
We now consider the upper bound. \\
\begin{align}
U(\pi_e)& = \mathbb{E}[\bUneu\mathbb{E}[Y|X,\Pi \in \Sagree(X), Z] \nonumber\\
	         & \quad + \bUmon\mathbb{E}[Y|X,\Pi \in \Sgetight(X,Z), Z]\nonumber\\
	         & \quad + \bUcor\min_{i\in \Scorrect(X,Z)}\mathbb{E}[Y|X,\Pi= i, Z]  \nonumber\\
	         & \quad + \bUdef Y_{max} \nonumber
	].   
\end{align}
Where the sets in the indicators are defined as they are in \cref{thmt@@BoundsTheorem}. We note that the structure of this expression is identical to that of the lower bound, except for the use of minima instead of maxima, and the different definitions of the sets in the indicators. Since the proof was agnostic to the specific definitions of the sets, and a minimum can be interpreted as a maximum over negative values, the proof follows identically. 
\end{proof}

\newpage
\begin{figure*}[t]
    \centering
    \begin{subfigure}[b]{0.24\textwidth}
        \centering
        \includegraphics[width=\textwidth]{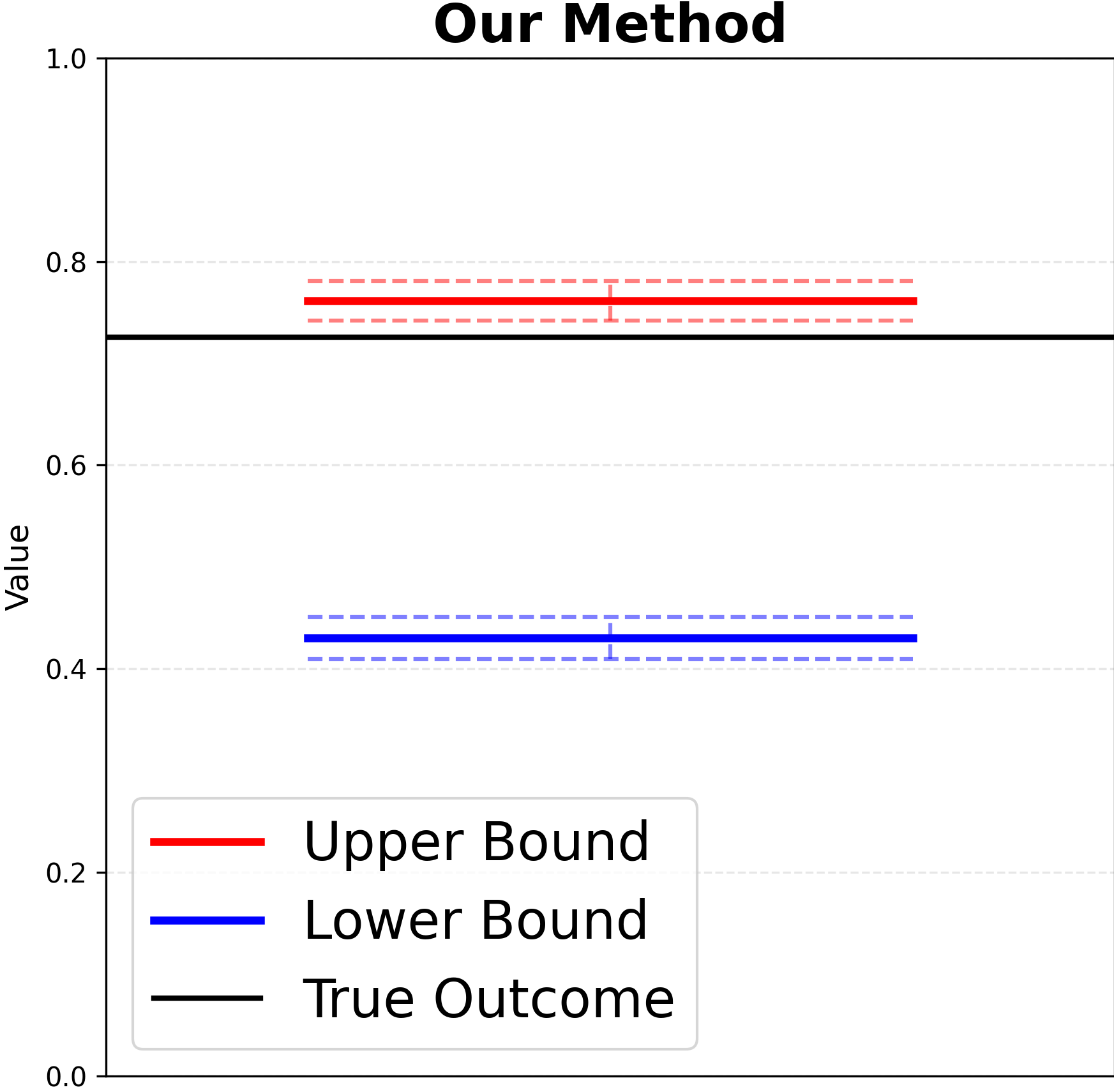}
        \caption{}
        \label{fig:abl_bounds_A}
    \end{subfigure}
    \hfill
    \begin{subfigure}[b]{0.24\textwidth}
        \centering
        \includegraphics[width=\textwidth]{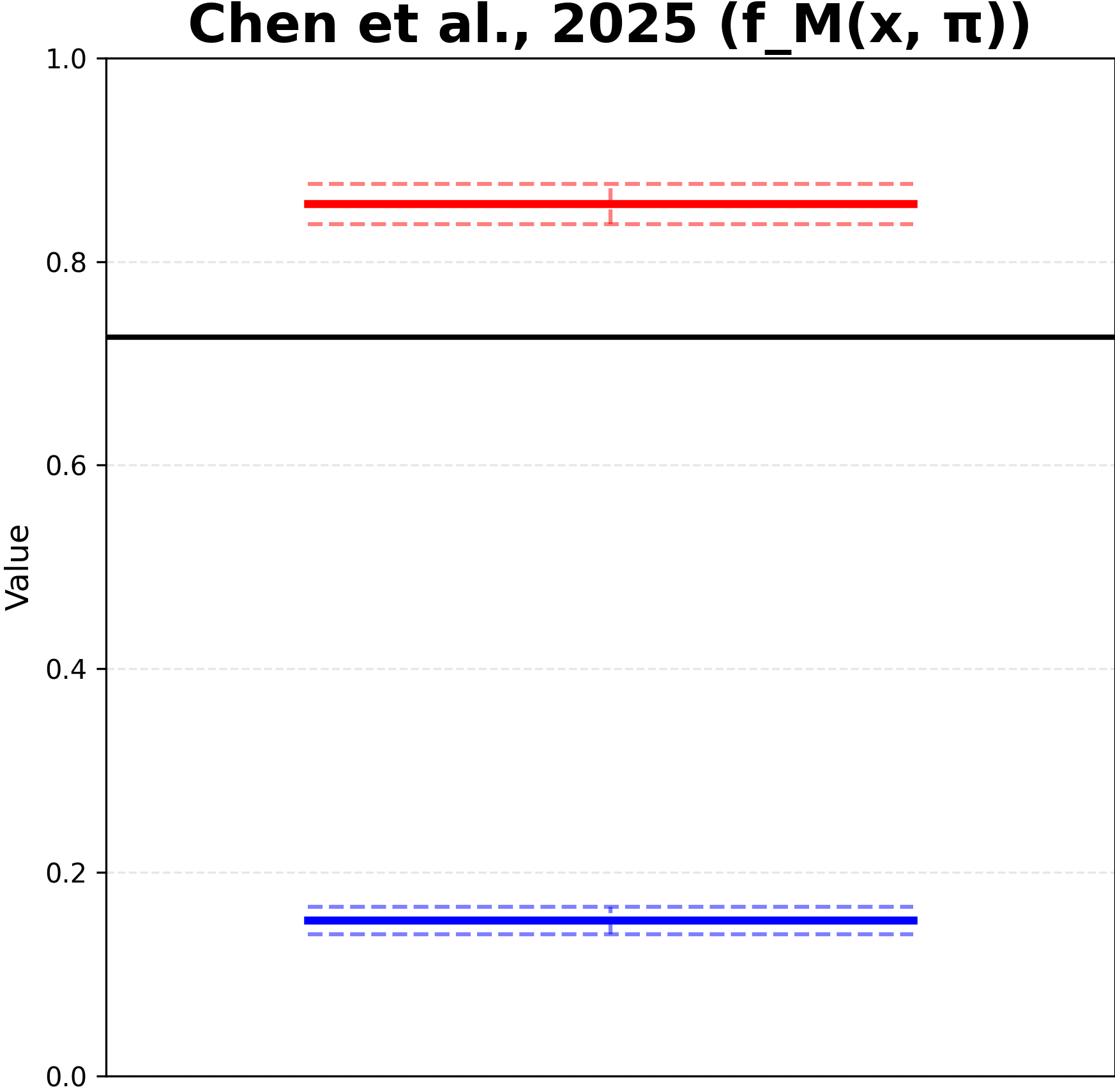}
        \caption{}
        \label{fig:bounds_B}
    \end{subfigure}
    \hfill
    \begin{subfigure}[b]{0.24\textwidth}
        \centering
        \includegraphics[width=\textwidth]{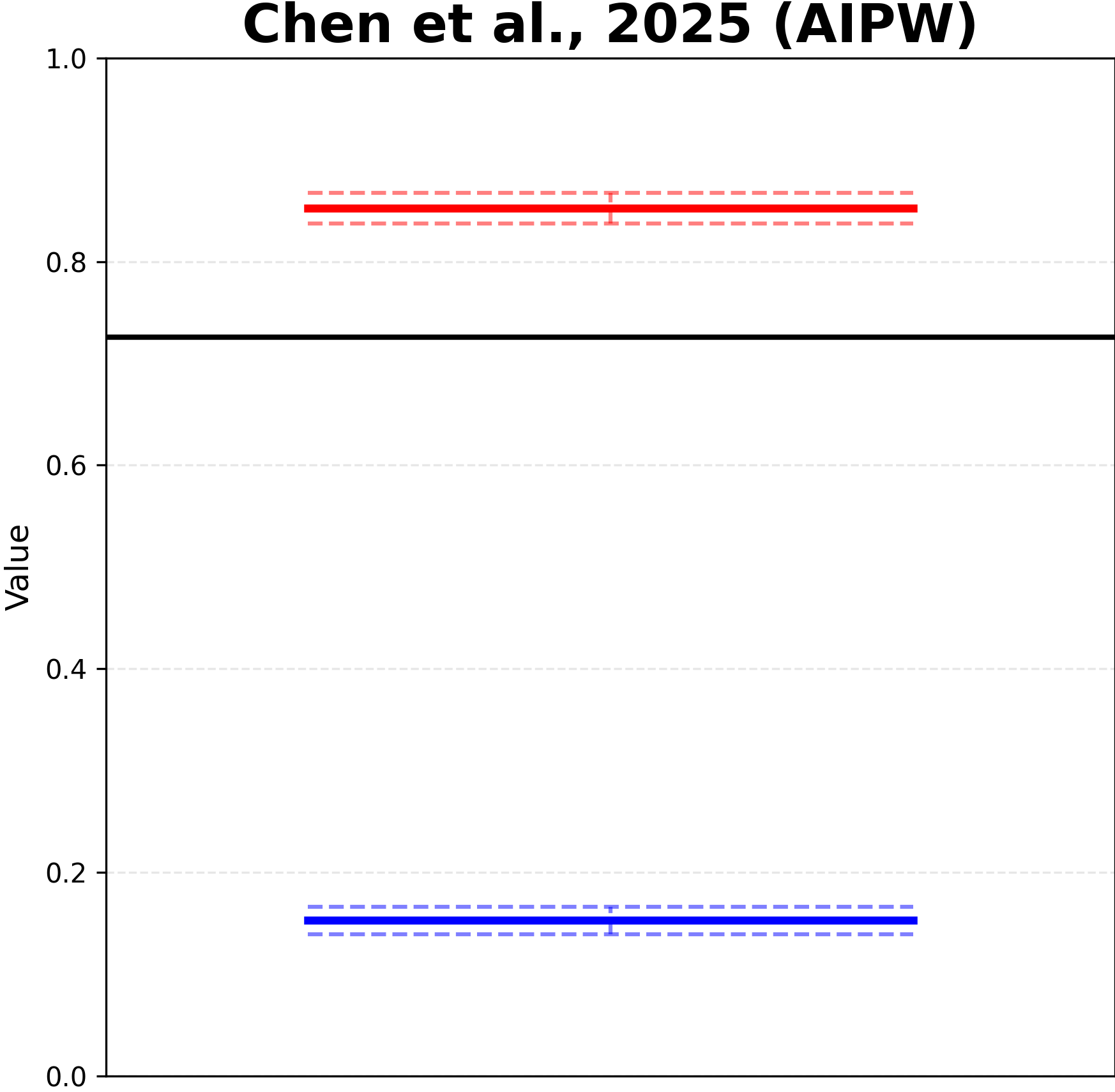}
        \caption{}
        \label{fig:bounds_C}
    \end{subfigure}
    \hfill
    \begin{subfigure}[b]{0.24\textwidth}
        \centering
        \includegraphics[width=\textwidth]{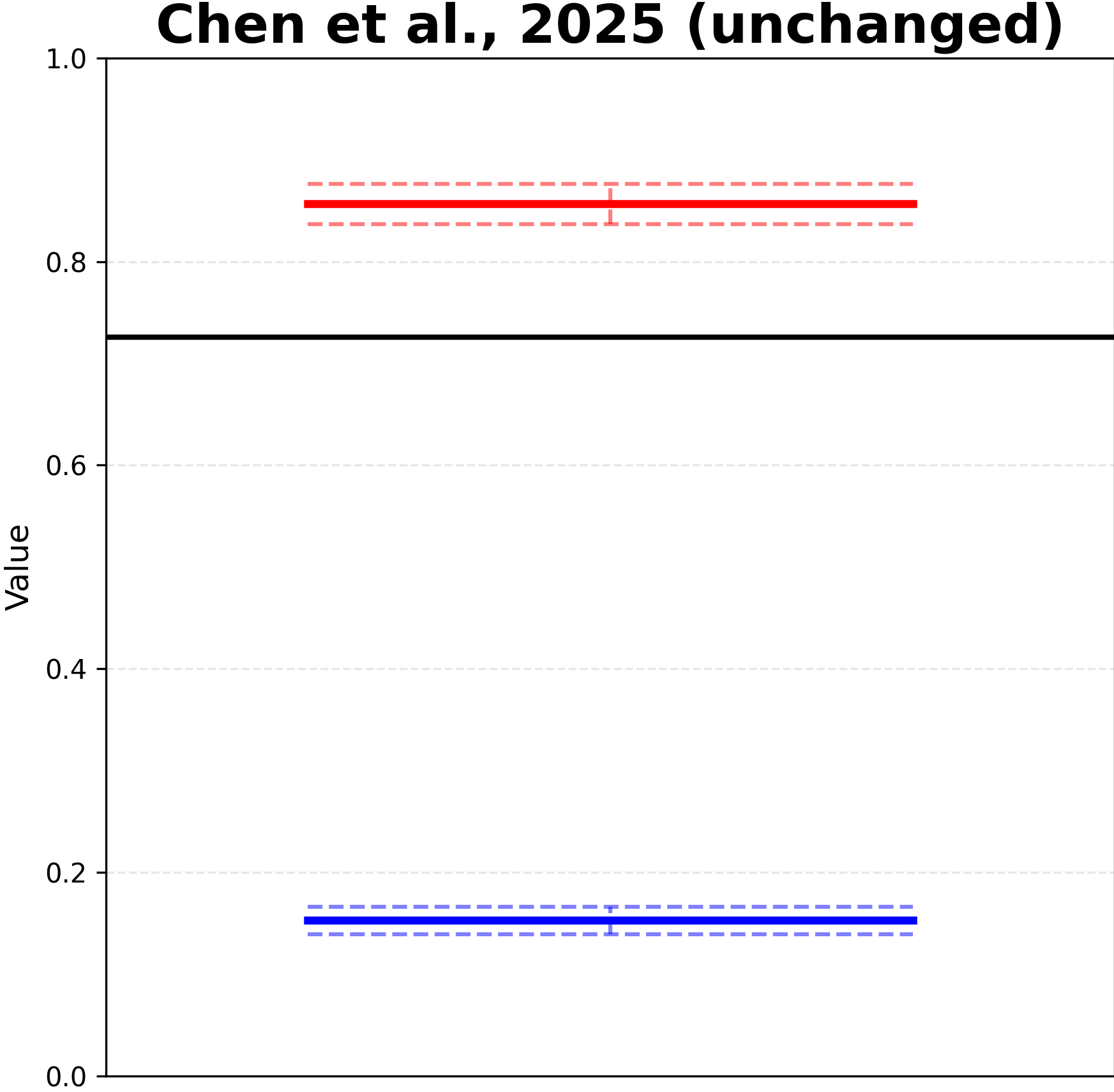}
        \caption{}
        \label{fig:abl_bounds_D}
    \end{subfigure}
   \caption{Ablation study of generated bounds on downstream outcomes. The figure displays the lower bound and upper bound computed using our proposed method in \cref{fig:abl_bounds_A}, the existing method in \citet{chen2025} using performance stratified across age groups in \cref{fig:bounds_B}, the existing method using AIPW estimation in \cref{fig:bounds_C}, and the existing method without any modifications in  \cref{fig:abl_bounds_D}. 90\% two-sided confidence intervals are shown.}
   \label{fig:abl_simulation_results}
\end{figure*}
\section{Further Simulation Details}
\label{app:experiments}
\begin{figure}[t]
    \centering
        \includegraphics[width=\columnwidth/3]{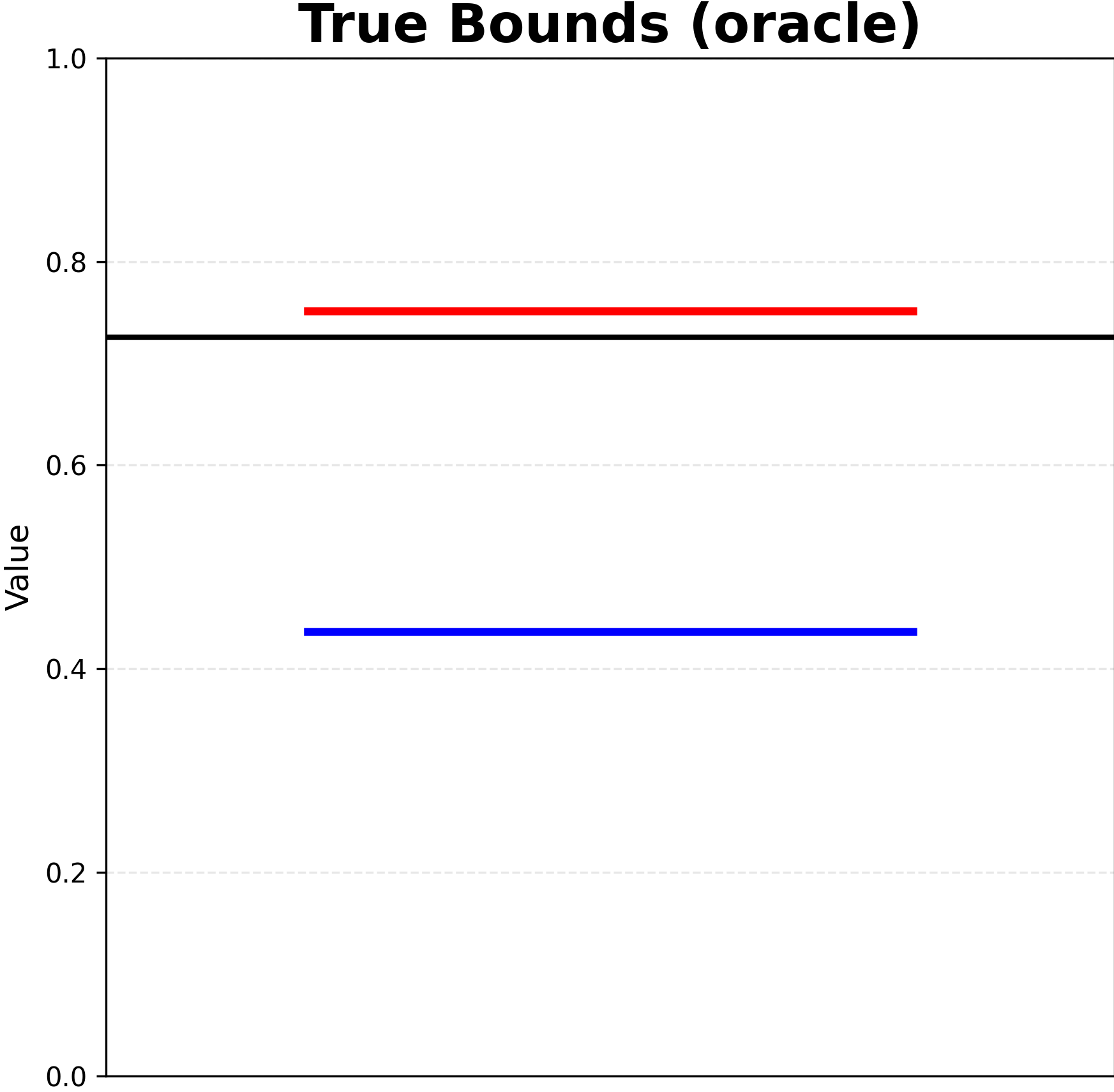}
        \caption{The figure displays the lower and upper bound computed with our method using the true theoretical values, as if we had access to the counterfactual outcomes for all models.}
        \label{fig:abl_bounds_E}
\end{figure}

\begin{table}[ht]
\centering
\small
\begin{tabular}{lcccccc}
\toprule
Method & Bound width & Lower (mean) & Lower 90\% CI & Upper (mean) & Upper 90\% CI & True value \\
\midrule
Our Method & 0.331 & 0.430 & [0.409, 0.451] & 0.761 & [0.742, 0.781] & 0.726 \\
True Bounds (oracle) & 0.315 & 0.437 & -- & 0.751 & -- & 0.726 \\
\midrule
Chen et al., 2025 ($f_M(x, \pi)$ stratified) & 0.704 & 0.153 & [0.139, 0.166] & 0.857 & [0.837, 0.877] & 0.726 \\
Chen et al., 2025 (AIPW) & 0.700 & 0.153 & [0.139, 0.166] & 0.852 & [0.837, 0.867] & 0.726 \\
Chen et al., 2025 (unchanged) & 0.704 & 0.153 & [0.139, 0.166] & 0.857 & [0.837, 0.877] & 0.726 \\
\bottomrule
\end{tabular}
\caption{(Experiment 1) Numerical summary of the ablation study results.}
\label{tab:abl_simulation_results}
\end{table}

\begin{table}[t]
    \centering
    \small
\begin{tabular}{llcccccc}
\toprule
Experiment & Bounds & Bound width & Lower (mean) & Lower 90\% CI & Upper (mean) & Upper 90\% CI & True value \\
\midrule
Experiment 3 & Estimated & 0.140 & 0.624 & [0.607, 0.641] & 0.764 & [0.749, 0.779] & 0.696 \\
Experiment 3 & Ground truth & 0.144 & 0.615 & -- & 0.760 & -- & 0.696 \\
\midrule
Experiment 5 & Estimated & 0.204 & 0.586 & [0.568, 0.604] & 0.790 & [0.776, 0.804] & 0.697 \\
Experiment 5 & Ground truth & 0.209 & 0.578 & -- & 0.787 & -- & 0.697 \\
\bottomrule
\end{tabular}
    \caption{Semi-synthetic experiment results comparison between experiment 3 and experiment 5, which differ only in the presence of missing data. }
   \label{tab:exp_3_vs5}
\end{table}
The experimental trial consisted of 1,891 defendants from the dataset whose various covariates, such as demographics, charge history, and staff recommendations were recorded. The original dataset also contained an indicator of whether the judge was shown the risk score predicted by the Public Safety Assessment (PSA) model, the judge's detention decision, and the outcome of whether the defendant failed to appear in court (FTA). We use these data to construct all the variables in our RCT, as outlined in the experiments section of the main text. 

In the original setting of \citet{imai2023experimental}, judges were able to give high or low bail, and models provided integer risk scores on a scale of 1-6 on not only FTA, but also new criminal activity and new violent criminal activity. The score is also always shown to the judge. Our method is still applicable to these settings, however we choose to artificially binarize the model prediction and modify the setting to an alerting context so that a meaningful neutral action ($\hat{z}_0$) exists. We chose to binarize the judge's decision and omit the other scores largely for simplicity. 

\citet{van2025accurate} discusses a setting in which performance is dependent on the data distribution on which it is measured. Thus, under a performance-dependent intervention, the dataset will shift, causing confounding that may lead to orthogonal changes between potential outcomes and performance. We highlight this due to its similarity to our setting, but emphasize that it does not apply. This is because performance is not measured relative to the potential outcomes, but instead relative to counterfactual outcomes that do not change with any interventions.  

The nuisance models used to learn the relationships in the data to simulate judge decisions and outcomes were simple logistic regressions, without any feature engineering outside of the existing features in the dataset. The nuisance models used in the AIPW estimation were lasso regularized linear regressions with cross-validated $L_1$ penalty. 
\subsection{Additional Experiment Details}\label{app:add_exp}
We now present additional details to some of the experiments discussed. We present some of the figures for each experiment, but any graphs and tables presented have analogous versions for every experiment. They are available in the GitHub repository. In no experiment do the falsification tests fail. 
\paragraph{Experiment 1}\label{app_exp_1}
Suppose the alerting model discussed in \hyperref[exp_3]{Experiment 3} already exists and a two-arm RCT was conducted comparing the FTA of the alerting model against a control arm (never alert). We are considering lowering the threshold of the alerting model so that it alerts when the same risk score is >1 instead, without running an RCT. We use our method to develop bounds on the downstream FTA from the data in the two-arm RCT.

We use the same experimental setup as in \hyperref[exp_3]{Experiment 3}, except that we now have an additional trialed model. Thus, we first randomly assign each defendant to either the control arm (never alert) or the trialed model (alert if risk score > 3). The values of $A, Y$ are generated only for the appropriate model per the random assignment.

We show our results of this experiment in \cref{fig:exp_1}. We see that our bounds remain tighter than the existing methods, though both are significantly wider than in \hyperref[exp_4]{Experiment 4}. The difference between the two methods is also less pronounced. This is because the additional trialed model decreases the likelihood that the untrialed model disagrees with all trialed models, which activates the $\bLcor,\bUcor$ terms less often. This demonstrates that our method is more effective in settings with more limited data, and has diminishing returns as the number of trialed models increases.

\paragraph{Experiment 2}\label{app_exp_2}
Suppose a model that alerts for offenders with a \citet{imai2023experimental} risk score greater than 3 already exists, and a two-arm RCT was conducted comparing the FTA of the alerting model against a control arm (never alert). We develop a new model that still alerts when the risk score is $>3$; however, it generates those risk scores in a fundamentally different way. We want to evaluate the effectiveness of the new risk score algorithm without running an RCT. We use our method to develop bounds on the downstream FTA from the data in the two-arm RCT.

We first randomly assign each defendant to either the control arm (never alert) or the trialed model (alert if risk score > 3). The values of $A, Y$ are generated only for the appropriate model per the random assignment. The untrialed model doesn't use the risk scores from the \citet{imai2023experimental} data. Instead, it finds the top 5 non-demographic features most correlated with FTA in the data. It then does a simple tally of the number of those features that are present for a given defendant, and uses that as the risk score, where 0 features present corresponds to a risk score of 1, and 5 features present corresponds to a risk score of 6.

We show the results of this experiment in \cref{fig:exp_2}. We see that our bounds are somewhat looser but still comparable to the bounds in \cref{fig:exp_1}. This demonstrates that our method still remains more effective than existing methods even when the trialed model is fundamentally different from the untrialed model.

\paragraph{Experiment 5}\label{app_exp_5}
This experiment is identical to Experiment 3, but with missing values of $Z$. $Z$ is defined as the counterfactual outcome under no bail from the judge. Thus, it is only observed for individuals where $A = 0$, but never observed for individuals where $A = 1$. In this experiment, we set all values of $Z$ for such individuals to $Z=-1$, which prevents the logic in $\bLcor, \bUcor, \bLmon, \bUmon$ from activating. As discussed in \cref{app: exp_5_miss}, this is a possibly misspecified experiment.

We see that our bounds are somewhat looser but still comparable to the bounds in \cref{fig:simulation_results1}. They remain tighter than the existing methods, which are not affected by the missingness of $Z$. We see that in this case, the violation of MAR does not significantly affect the bounds. This is because the fallback to $\bLdef, \bUdef$ branches remain valid, the $\bLneu, \bUneu$ branches for missing $Z$ values are only impacted through the estimation of $\mu(X, Z)$, and the $\bLcor, \bUcor, \bLmon, \bUmon$ branches do not get activated sufficiently to significantly affect the bounds.    

\clearpage
\vspace*{\fill}
\begin{center}
   \includegraphics[width=\linewidth]{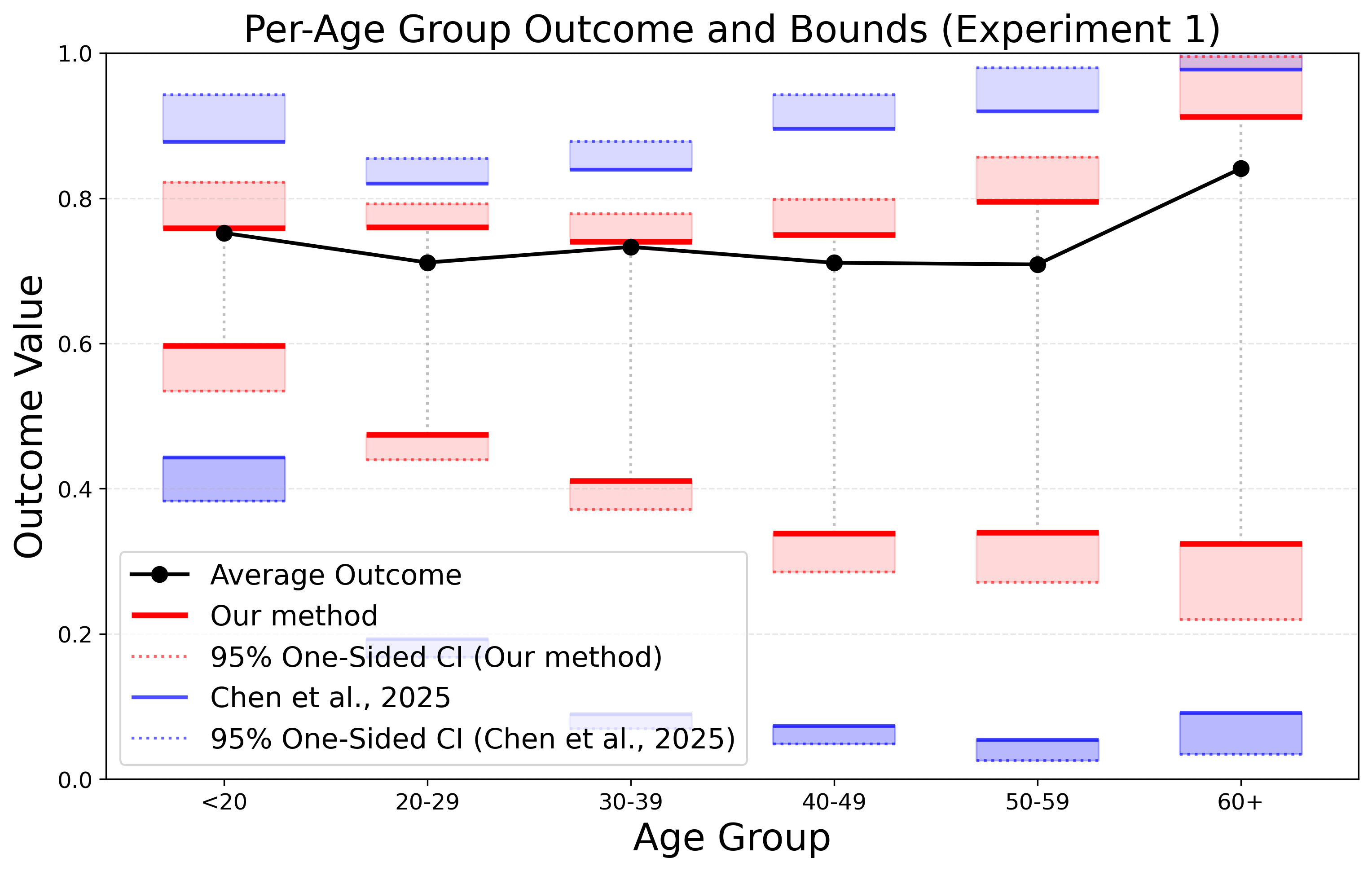}
   \captionof{figure}{(Experiment 1) Bounds on downstream outcomes for an untrialed model with two trialed models in the RCT, plotted as functions of individual age.}
   \label{fig:exp_1}
\end{center}
\vspace*{\fill}

\clearpage
\vspace*{\fill}
\begin{center}
   \includegraphics[width=\linewidth]{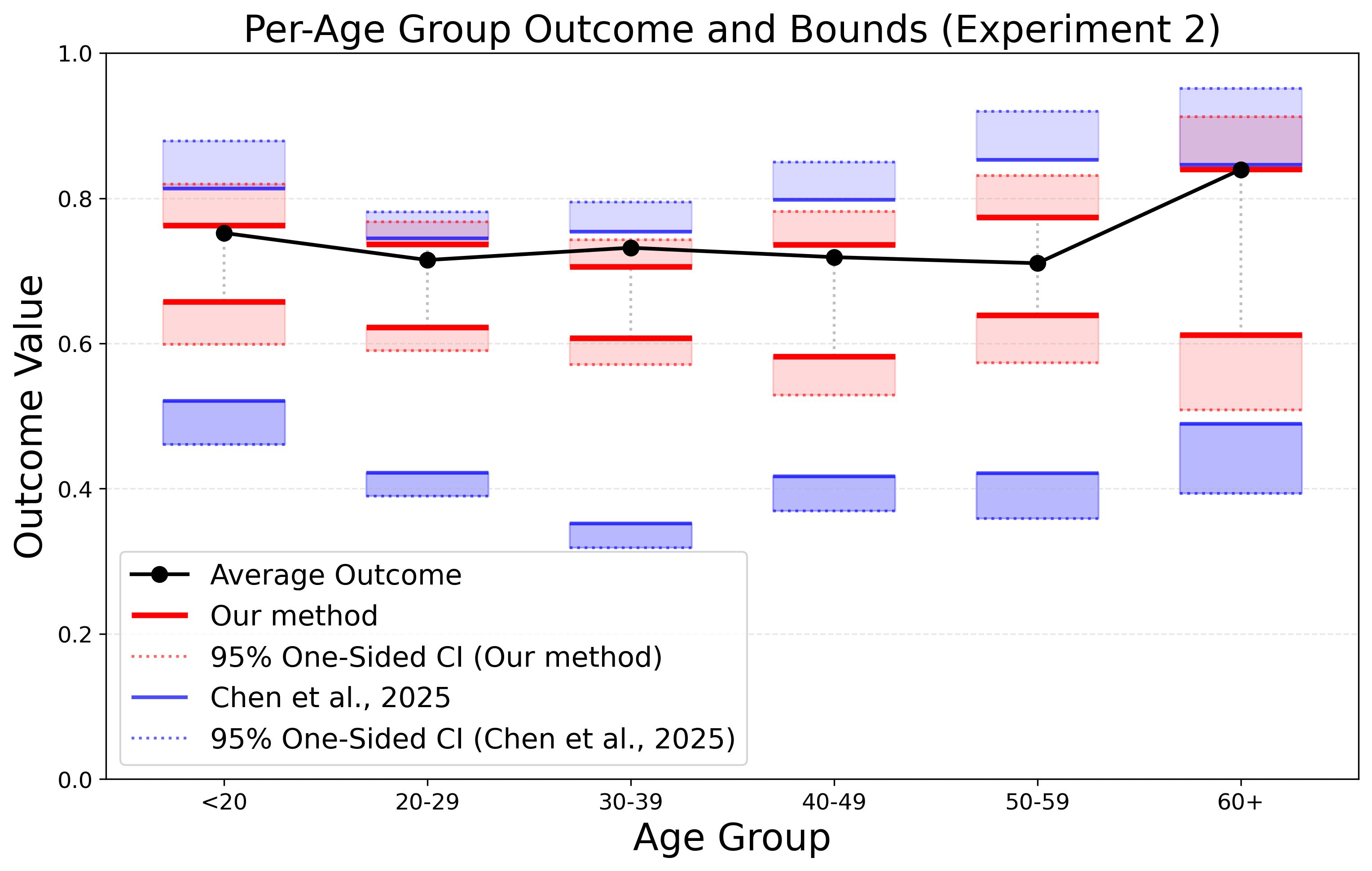}
   \captionof{figure}{(Experiment 2) Bounds on downstream outcomes for an untrialed model under a covariate tally alerting rule, plotted as functions of age.}
   \label{fig:exp_2}
\end{center}
\vspace*{\fill}

\clearpage
\vspace*{\fill}
\begin{center}
   \includegraphics[width=\linewidth]{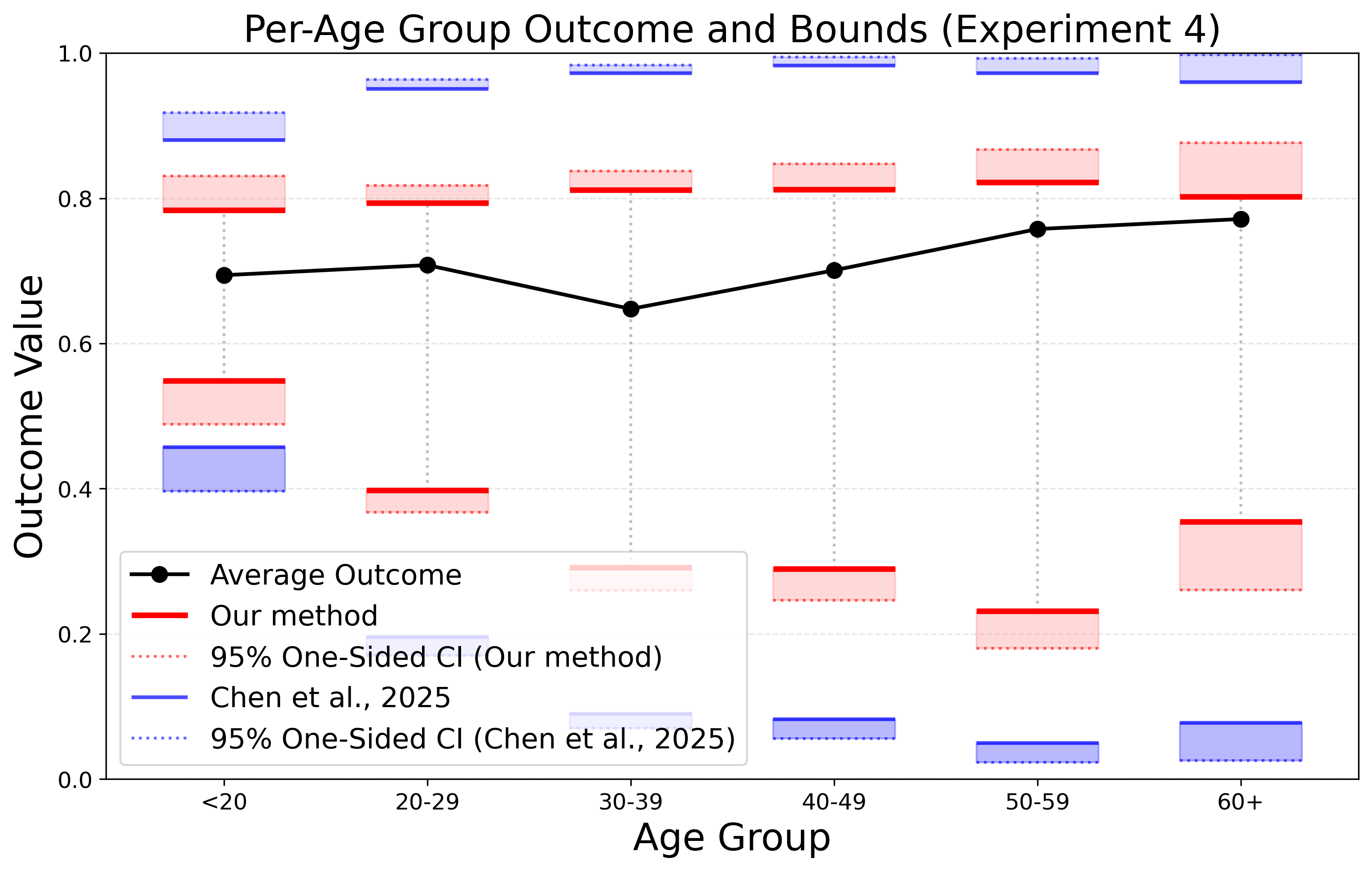}
   \captionof{figure}{(Experiment 4) Bounds on downstream outcomes for an untrialed model with a single trialed model in the RCT, plotted as functions of age.}
   \label{fig:exp_4}
\end{center}
\vspace*{\fill}

\clearpage
\vspace*{\fill}
\begin{center}
   \includegraphics[width=\linewidth]{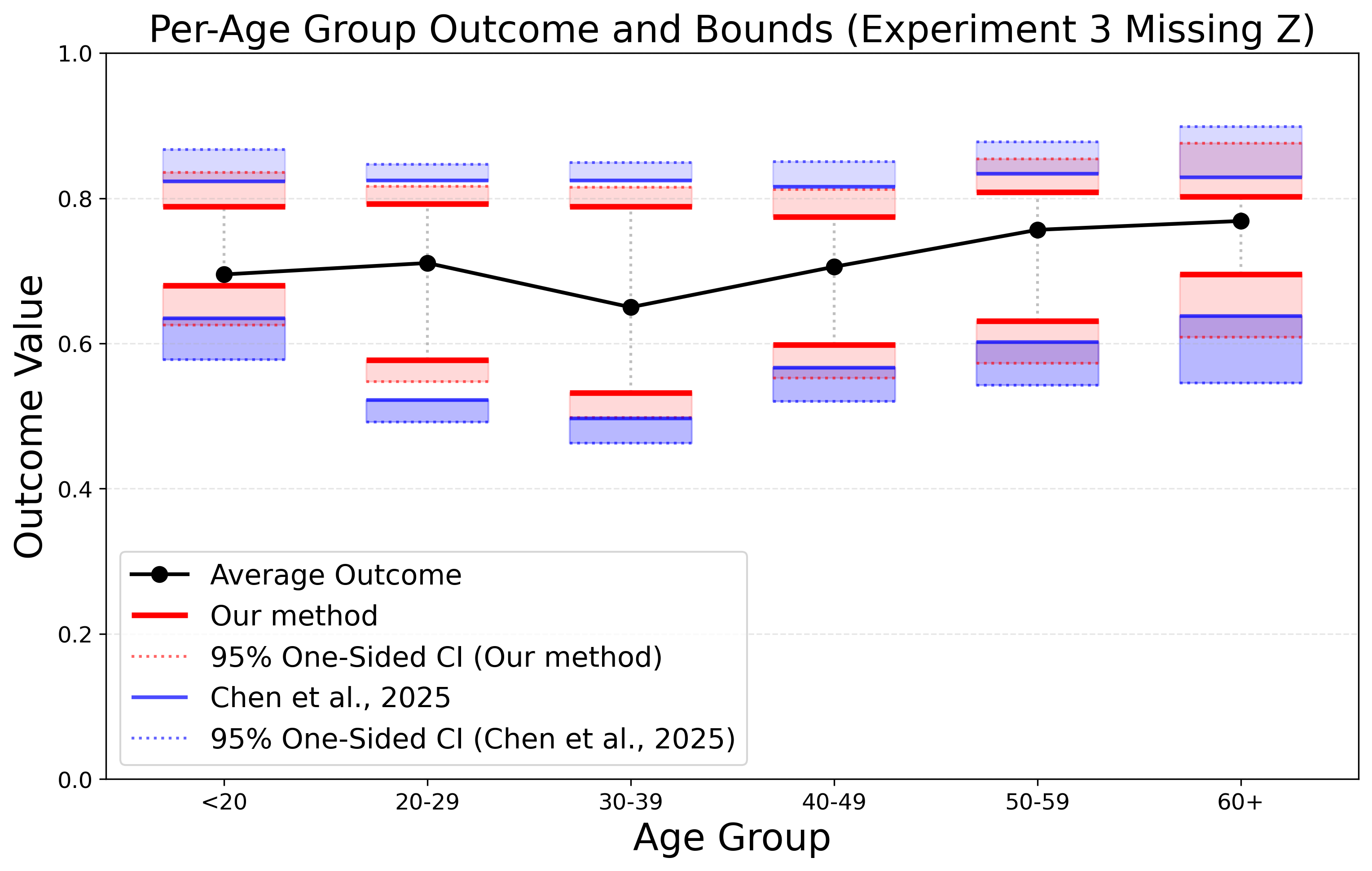}
   \captionof{figure}{(Experiment 5) Bounds on downstream outcomes for an untrialed model with a single trialed model in the RCT, plotted as functions of age, with missing data.}
   \label{fig:exp_5}
\end{center}
\vspace*{\fill}

\clearpage

\begin{table}[ht]
\centering
\small
\begin{tabular}{lcccccc}
\toprule
Patient\_ID & age\_group & UPPER & LOWER & UPPER (ground-truth) & LOWER (ground-truth) & TRUE\_EXPECTED \\
\midrule
0 & 4 & 1.000 & 1.000 & 0.681 & 0.681 & 0.681 \\
1 & 4 & 0.909 & 0.000 & 0.651 & 0.000 & 0.652 \\
2 & 3 & 1.000 & 1.000 & 0.818 & 0.818 & 0.818 \\
3 & 2 & 1.000 & 0.909 & 1.000 & 0.472 & 0.476 \\
4 & 2 & 1.000 & 1.000 & 0.777 & 0.777 & 0.777 \\
5 & 4 & 1.000 & 1.000 & 0.715 & 0.715 & 0.715 \\
6 & 3 & 1.000 & 0.000 & 0.612 & 0.000 & 0.613 \\
7 & 2 & 0.091 & 0.000 & 0.672 & 0.000 & 0.672 \\
8 & 3 & 1.000 & 1.000 & 0.687 & 0.687 & 0.687 \\
9 & 4 & 1.000 & 0.000 & 0.676 & 0.000 & 0.677 \\
10 & 2 & 0.091 & 0.000 & 0.665 & 0.000 & 0.666 \\
11 & 4 & 1.000 & 0.000 & 1.000 & 0.570 & 0.572 \\
12 & 2 & 1.000 & 1.000 & 1.000 & 0.647 & 0.649 \\
\bottomrule
\end{tabular}
\caption{The bound values generated for the first 13 individuals in the dataset (Experiment 3). The UPPER and LOWER columns refer to the values of the upper and lower bounds generated using this paper's method. These sometimes fall outside $[0,1]$ or fail to capture the true value due to estimation error. The ground truth columns are the bounds generated using the true theoretical values. The TRUE\_EXPECTED column is the true expected outcome for the individual, which was unknown to the methods. }
\label{fig:Exp_table}
\end{table}

\begin{table}[ht]
\centering
\small
\begin{tabular}{lcccccc}
\toprule
Method & Outcome cov. & Interval width & Lower CI cov. & Lower CI width & Upper CI cov. & Upper CI width \\
\midrule
Our Method & 1.00 & 0.313 & 0.97 & 0.042 & 0.91 & 0.039 \\
True Bounds (oracle) & 1.00 & 0.315 & -- & -- & -- & -- \\
Chen et al. ($f_M(x, \pi)$) & 1.00 & 0.692 & 0.00 & 0.027 & 0.00 & 0.039 \\
Chen et al. (unchanged) & 1.00 & 0.692 & 0.00 & 0.027 & 0.00 & 0.039 \\
Chen et al. (AIPW) & 1.00 & 0.693 & 0.00 & 0.027 & 0.00 & 0.031 \\
\bottomrule
\end{tabular}
\caption{Monte Carlo coverage and width across 100 seeds (Experiment 1). The {Outcome cov.} column is how often the final bound interval [lower, upper] covers the true outcome. The {Lower/Upper CI cov.}: how often the two-sided 90\% CI around the lower (resp.\ upper) bound covers the true lower (resp.\ upper) bound. Widths are means. }
\label{tab:coverage-experiment}
\end{table}

\begin{table}[ht]
\centering
\small
\begin{tabular}{lcc}
\toprule
Policy & Precision & Alert fraction \\
\midrule
Control & -- & 0.000 \\
$\mathfrak{R} > 1$ & 0.349 & 0.802 \\
$\mathfrak{R} > 3$ & 0.456 & 0.281 \\
CustomRiskScore & 0.413 & 0.456 \\
\bottomrule
\end{tabular}
\caption{Overall precision and alert frequency of each alerting policy used across the experiments, evaluated on the full dataset ($n = 1891$). The alert fraction is the proportion of defendants for whom the policy raises an alert. Precision is the fraction of raised alerts for which the defendant would fail to appear if released. The Control policy never alerts, so its precision is undefined (\texttt{--}).  ``$\mathfrak{R} > 1$'' and ``$\mathfrak{R} > 3$'' are the models that threshold the \citet{imai2023experimental} PSA failure-to-appear risk score, and CustomRiskScore is the custom alerting system introduced in \hyperref[app_exp_2]{Experiment 2}.}
\label{tab:model-precision-alert}
\end{table}

\begin{table}[ht]
\centering
\small
\begin{tabular}{lcc}
\toprule
FTAScore & Count & Proportion \\
\midrule
1 & 375 & 0.198 \\
2 & 531 & 0.281 \\
3 & 454 & 0.240 \\
4 & 286 & 0.151 \\
5 & 189 & 0.100 \\
6 & 56 & 0.030 \\
\bottomrule
\end{tabular}
\caption{Baseline distribution of the \citet{imai2023experimental} PSA failure-to-appear risk score (integer values 1--6) across all defendants in the dataset ($n = 1891$). This is the score thresholded by the $\mathfrak{R} > 1$ and $\mathfrak{R} > 3$ policies in \cref{tab:model-precision-alert}.}
\label{tab:fta-score-distribution}
\end{table}

\begin{table*}[ht]
\centering
\small
\begin{tabular}{llccccc}
\toprule
Bound & Branch & Experiment 1 & Experiment 2 & Experiment 3 & Experiment 4 & Experiment 5 \\
\midrule
Lower & def & 698 (36.91\%) & 238 (12.59\%) & 289 (15.28\%) & 987 (52.19\%) & 406 (21.47\%) \\
Lower & cor & 529 (27.97\%) & 356 (18.83\%) & 242 (12.80\%) & 529 (27.97\%) & 125 (6.61\%) \\
Lower & neu & 375 (19.83\%) & 1028 (54.36\%) & 1360 (71.92\%) & 375 (19.83\%) & 1360 (71.92\%) \\
Lower & mon & 289 (15.28\%) & 269 (14.23\%) & 0 (0.00\%) & 0 (0.00\%) & 0 (0.00\%) \\
Upper & def & 287 (15.18\%) & 124 (6.56\%) & 242 (12.80\%) & 529 (27.97\%) & 372 (19.67\%) \\
Upper & cor & 987 (52.19\%) & 507 (26.81\%) & 289 (15.28\%) & 987 (52.19\%) & 159 (8.41\%) \\
Upper & neu & 375 (19.83\%) & 1028 (54.36\%) & 1360 (71.92\%) & 375 (19.83\%) & 1360 (71.92\%) \\
Upper & mon & 242 (12.80\%) & 232 (12.27\%) & 0 (0.00\%) & 0 (0.00\%) & 0 (0.00\%) \\
\bottomrule
\end{tabular}
\caption{Frequency of each bound branch across all individuals ($n = 1891$) for Experiments 1--4 (fully observed $Z$) and Experiment 5 (\hyperref[app_exp_5]{Experiment 3 with missing $Z$}). Each cell reports the count of individuals and, in parentheses, the percentage; each individual activates exactly one lower and one upper branch.}
\label{tab:branch-activation-combined}
\end{table*}

\newpage
\section{Related Work}
\label{app:related_work}

\textbf{Offline Policy Evaluation} Our work is related to offline policy evaluation \citep{Su2020AdaptiveESA, Wu2023DistributionalOPA, Rothfuss2023HallucinatedACA, Namkoong2020OffpolicyPEA}. If we interpret a model as a deterministic policy that dictates treatments, we are seeking to evaluate the effectiveness of a new decision-making policy, using recorded outcomes under different policies. However, our work differs due to the presence of a human decision-maker, whose compliance with the policy is dependent on many factors such as their perception of the policy's effectiveness. Furthermore, in offline policy evaluation, if the unknown policy makes a prediction not present in data, we cannot estimate its effect due to a positivity violation. However, our setting specifically focuses on using counterfactual correctness to resolve this situation. 

\textbf{Transportability and Z-Identification}
Transportability studies compute the causal effect of an intervention in a ``target'' environment from data on the same intervention in a ``source'' environment. The two environments differ structurally, often due to a selection node \citep{Pearl2015GeneralizingEFA, Mao2022CausalTFA, Shi2021DataIIA, Lal2023AFFA, stuart2011use, pearl2011transportability}. 
Conversely, Z-identification studies compute the causal effect of an intervention based on data of a different intervention in the same environment \citep{correa2020calculus, bareinboim2012zid, lee2019gid, correa2020soft}.
We are related to both fields, but more closely to z-identification, as we are estimating the causal effect of an untrialed model (different intervention) deployed in the existing environment of the RCT. 
We differ from both approaches in that we make parametric assumptions, like monotonicity, and our goal is not to give conditions for unbiased point estimation from infinite data. Instead, we give an interval containing the effect from finite data when the above algorithms preclude unbiased point estimation. 

\textbf{Partial Identification}
Partial identification is concerned with bounding causal effects when point estimation is impossible. It begins with Manski's assumption-free natural bounds and progresses to sharper bounds under graphical assumptions via linear and polynomial programming \citep{manski1990nonparametric, robins1989analysis, chickering1997finitedatapid, balke1994probabilistic, balke1997bounds, richardson2014parametricpid, zhang2021continuousivpid, zhang2021generalpid, bellot2024pagbounding, bellot2024pagboundingestimation,  duarte2021automatedapproachcausalinference}.
We differ from these approaches because our bounds use non-graphical assumptions pertinent to the application (as well as graphical assumptions), and we use experimental (interventional) data instead of observational data. Furthermore, the causal effect we bound occurs in a structurally different environment. 

\textbf{Evaluation of Machine Learning Models under Distribution Shift}
Distribution shift occurs when the environment in which a model is deployed differs structurally (not randomly) from the environment in which it is trained, violating the assumption that train and test sets are identically distributed \citep{Taori2020MeasuringRTA, Hendrycks2020TheMFA, Wiles2021AFAA, Koh2020WILDSABA, Miller2021AccuracyOTA}. In our setting, the presence of a human decision-maker causes the deployment environment to be differently distributed than the training environment, because the human can choose to implement a different action from the machine's suggestion.  We overcome this challenge by making assumptions about how the human will act, and the RCT data that we learn from is unconfounded. 
 
\textbf{Learning to Defer}
In our work, we mention deferral as a possible example of a neutral outcome. Our setting is closely related to learning to defer, in which an AI classifier may abstain on a given input and pass the decision to a human expert, with the goal of training the classifier to complement rather than duplicate the human \citep{Verma2022CalibratedLTA, Mozannar2023WhoSPA}. Like our work, L2D recognizes that the human decision-maker, not the model, is the final arbiter of the outcome, and that the two should be modeled jointly. However, L2D differs from our setting in several respects. First, L2D is concerned more with specifically how to train models, while our work focuses on the evaluation of already trained models. Second, L2D models an explicit, learnable rejector that routes each input to either the model or the human, while in our setting the human's response to a model prediction is unobserved and mediated by their trust in the model's perceived performance. Finally, L2D typically assumes the human prediction is available on deferred examples, whereas our framework must reason about counterfactual outcomes that are never directly observed.
 
\section{Missing Data}\label{missingness}

In the main text, we assume that $Z$ is observed for all subjects, and derive our results in that context.  In many practical settings, however, $Z$ may be missing for some subjects, due to limitations of data collection and/or due to being defined in counterfactual terms (e.g., in our motivating example for the semi-synthetic experiments, $Z$ is a counterfactual FTA indicator that is unobserved when bail is granted).

In this section, we discuss an example missingness structure under which our method remains valid. We limit our discussion to a simple missingness mechanism, to demonstrate how our method can hold under missingness. There may exist more interesting and unique structures in which our method remains valid. However, we leave the exploration of these structures to future work.

Let $R \in \{0,1\}$ be an observation indicator with $R = 1$ when $Z$ is observed and $R = 0$ otherwise, so the analyst observes $(X, \Pi, \hat{Z}, M, Y, R, Z)$ rather than fully observing $Z$.

Suppose we have the augmented DAG in \cref{fig:missingness_dag}, which is a simple extension of the DAG in \cref{fig:dag_a} to include the missingness indicator $R$, where:
$$R = f_R(X,\epsilon_R),\qquad \epsilon_R \perp (\epsilon_{XZ},\epsilon_Y,\dots).$$
The analyst observes $(X,\Pi,\hat Z,M,Y,R)$ and observes $Z$ only when $R=1$.

\begin{figure}[t]
\centering
\scalebox{0.80}{
\begin{tikzpicture}[>=stealth, node distance=2cm]
\tikzstyle{square} = [draw, very thick, minimum size=7mm, inner sep=3pt]
\begin{scope}
\path[->, very thick]
node[draw, square] (p) {$\Pi$}
node[above of=p] (d) {$D$}
node[right of=p, draw, square] (a) {$\hat{Z}$}
node[right of=a] (y) {$Y$}
node[right of=y, draw, square] (m) {$M$}
node[above of=a] (x) {$X$}
node[above of=y ] (z) {$Z$}
node[right of=z, red, draw, circle, dotted] (q) {$A$}
node[above of=x, draw, circle] (r) {$R$}
(p) edge[blue, bend right=20] (m)
(x) edge[red] (m)
(x) edge[blue, bend left=20] (q)
(p) edge[blue] (a)
(a) edge[blue] (q)
(q) edge[blue] (y)
(m) edge[blue] (q)
(d) edge[blue] (p)
(x) edge[blue] (a)
(x) edge[red, dotted, <->] (z)
(x) edge[blue] (y)
(z) edge[red] (y)
(x) edge[blue] (r)
;
\end{scope}
\end{tikzpicture}
}
\caption{Augmented version of \cref{fig:dag_a} including missingness. Node styles and edge colors follow
\cref{fig:dag_a}.}
\label{fig:missingness_dag}
\end{figure}
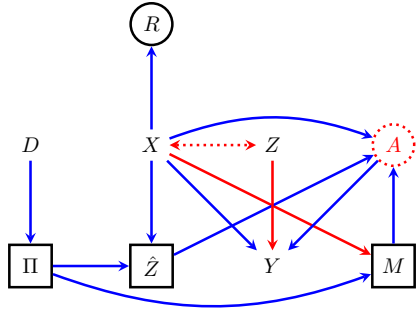

We define augmented versions of our bounds, $L'(\pi_e)$ and $U'(\pi_e)$, where $L'(\pi_e) \le L(\pi_e)$ and $U'(\pi_e) \ge U(\pi_e)$. In these augmented versions, we modify the definitions of $\bLcor, \bUcor, \bLmon, \bUmon$ to additionally require $R = 1$, so that these branches are only activated when $Z$ is observed. Thus, for individuals with $R = 0$, the only branches that can activate are $\bLdef, \bUdef, \bLneu, \bUneu$. The branches $\bLdef, \bUdef$ take the constant values $Y_{\min}, Y_{\max}$; under \cref{assumption:bounded} they are always valid, and can only loosen the inequality, never violate it, regardless of the missingness mechanism. The branches $\bLneu, \bUneu$, however, require estimation of a nuisance function:
$$\hat h(X,Y,Z,\Sagree)=\widehat\mu_{\Sagree}(X,Z)+\frac{\mathbf{1}\{\Pi\in\Sagree\}}{P(\Pi\in\Sagree\mid X)}\bigl(Y-\widehat\mu_{\Sagree}(X,Z)\bigr),$$
which references the individual's $Z$ through $\widehat\mu_{\Sagree}(X,Z)$ and is therefore not directly computable when $Z$ is missing.

To resolve this, we consider the neutral contribution to the population bound, noting that the neutral selector is a function of $X$ only. Let
\begin{itemize}
  \item $b(X) := \bLneu(X)$ --- a function of $X$ only,
  \item $g(X,Z) := \mathbb E[Y\mid X,\Pi\in\Sagree,Z]$ --- a function of $(X,Z)$.
\end{itemize}
Thus, the neutral contribution to the lower bound is $\mathbb E\big[b(X)\,g(X,Z)\big]$, where the expectation is taken over the joint distribution of $(X,Z)$. We consider only the lower bound, but the upper bound is identical. Because $b(X)$ depends on $X$ alone, the law of iterated expectations gives
\begin{align*}
\mathbb E\big[b(X)\,g(X,Z)\big]
&= \mathbb E\big[\,\mathbb E[b(X)\,g(X,Z)\mid X]\,\big]\\
&= \mathbb E\big[\,b(X)\,\mathbb E[g(X,Z)\mid X]\,\big].
\end{align*}
From the randomization we have $\Pi\perp Z\mid X$, so the inner term satisfies
\begin{align*}
\mathbb E[g(X,Z)\mid X]&=\mathbb E\big[\mathbb E[Y\mid X,\Pi\in\Sagree,Z]\mid X\big]\\
&= \mathbb E\big[\mathbb E[Y\mid X,\Pi\in\Sagree,Z]\mid X,\Pi\in\Sagree\big]\\
&= \mathbb E[Y\mid X,\Pi\in\Sagree].
\end{align*}
We then define $\nu(X):=\mathbb E[Y\mid X,\Pi\in\Sagree]$. Since $\nu(X)$ depends only on $X$, it can be estimated regardless of whether $Z$ is observed, and we estimate $\widehat\nu(X)$ in place of $\widehat\mu_{\Sagree}(X,Z)$. Thus the neutral term of our estimator becomes
$$\hat h_{\text{marg}}(X,Y, \Sagree)=\widehat\nu(X)+\frac{\mathbf{1}\{\Pi\in\Sagree\}}{P(\Pi\in\Sagree\mid X)}\bigl(Y-\widehat\nu(X)\bigr).$$

We use that estimator for all individuals. For the other branches on the individuals with $R = 1$, we must show that the nuisance functions in \cref{theorem:estimator} do not become inconsistent under missingness. To estimate our bounds, we require nuisance functions of the form $\mathbb E[Y\mid X,\Pi=\pi_j,Z]$. This cannot be estimated directly, since $Z$ is missing for $R = 0$ subjects; instead, we can only estimate
$$\mathbb E[Y\mid X,\Pi=\pi_j,Z,R=1].$$
Because in \cref{fig:missingness_dag} we have $R=f_R(X,\epsilon_R)$ with $\epsilon_R$ independent of the rest, it follows that
$$R \perp (Z,Y)\mid X \;\Rightarrow\; R\perp Y \mid X,\Pi,Z,$$
and therefore
$$\mathbb E[Y\mid X,\Pi=\pi_j,Z,R=1] \;=\; \mathbb E[Y\mid X,\Pi=\pi_j,Z] \;=\; \mu_j(X,Z).$$
For the fitted $\widehat\mu_j$ to be consistent, and for the margin and rate conditions of \cref{theorem:estimator} to carry over, the observed-$Z$ subsample must be representative of the population $X$-support. This requires the following assumption:

\begin{assumption}[Missingness Positivity]\label{assump:missing_positivity}
There exists $\eta > 0$ such that $\pi_R(x):=P(R=1\mid X=x)\ \ge\ \eta$ for all $x\in\mathcal X$.
\end{assumption}

This prevents missingness structures in which $Z$ is never observed for some values of $X$. For example, in a diagnostic context, if $Z$ were never observed for patients above age 80, the assumption would be violated. Under these conditions, the nuisance functions estimated on the $R=1$ subset converge to the original nuisance functions, and our estimator remains consistent. The model-assignment propensity $P(\Pi\mid X)$ is fixed by RCT design and does not involve $R$, so \cref{assump:known_prop} is unaffected; under \cref{assump:missing_positivity} the observed-$Z$ subsample shares the population $X$-support, so the margin conditions and nuisance rates transfer, with the observed fraction affecting only the constant in the rate. Therefore, under this missingness structure and \cref{assump:missing_positivity}, $\widehat L'(\pi_e)$ and $\widehat U'(\pi_e)$ are consistent and asymptotically normal (\cref{theorem:estimator}) for the valid bounds $L'(\pi_e)\le \mathbb E[Y(\pi_e)]\le U'(\pi_e)$.
\subsection{Relation to Experiment 5}\label{app: exp_5_miss}
In \hyperref[app_exp_5]{Experiment 5}, these conditions are clearly not satisfied. The missingness of $Z$ depends directly on the value of $A$. In this experiment, $Z$ is defined as the counterfactual outcome under no judge intervention, when it is available. When $A=1$, the judge gives bail to the defendant, and $Z$ is never observed. Conversely, when $A=0$, the judge does not give bail to the defendant, and $Z$ is always observed. 

This clearly does not match the missingness structure described in \cref{fig:missingness_dag} and violates \cref{assump:missing_positivity}. Regardless, as mentioned, we run it for demonstration purposes. 
\subsection{Specific Violated Assumptions}
We also discuss how our method can handle violations of specific assumptions posited in our method. Since each branch corresponds to a specific assumption, if any of the assumptions is violated, the corresponding indicator and its branch can be removed from the bounds. This would cause any values of $x$ that would have activated the branch to activate $b_L/b_U = \text{def}$ instead, resulting in a valid albeit looser bound. For example, if we are in a diagnostic setting, and there is no adequate neutral action (model cannot defer to physician), then we have violated \cref{assumption:neutral}, but we can run the method as is, since the neutral branches ($\bsU=\lneu$, $\bsL=\lneu$) will never activate regardless, and the bounds will be constructed using the other assumptions. Under violations of \cref{assumptions:correct} or \cref{assumption:monotonicity}, the corresponding branches (1 or 2, respectively) can instead default to $Y_{\min}, Y_{\max}$ bounds, maintaining validity, which trivially follows from \cref{assumption:bounded}.

\newpage
\section{Flowchart}
\label{app:flowchart}

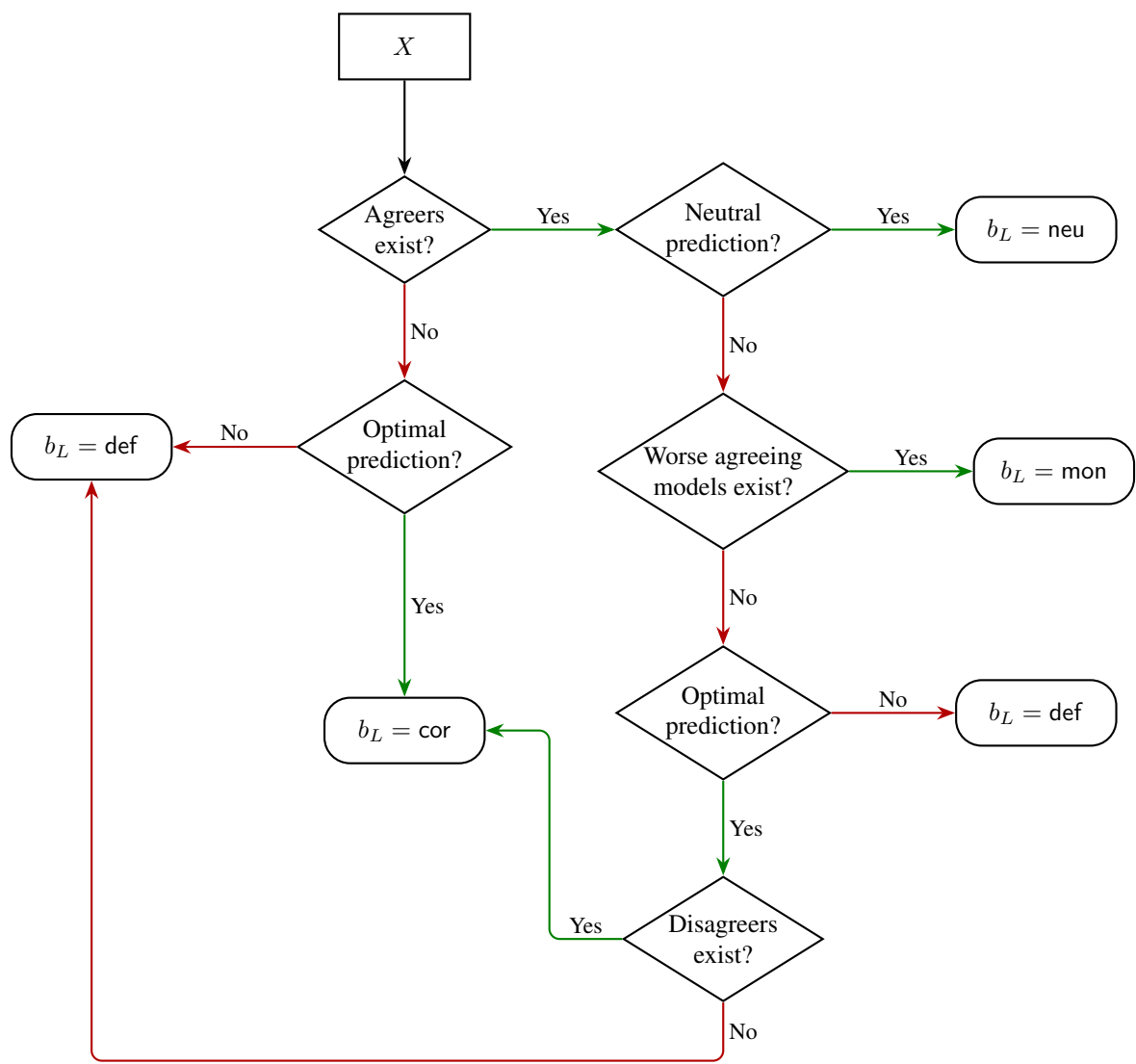
\begin{figure}[h]
\centering
\begin{tikzpicture}[
    node distance=1.3cm and 1.7cm,
    base/.style={draw, thick, align=center, minimum height=0.9cm},
    process/.style={base, rectangle, minimum width=1.8cm},
    decision/.style={base, diamond, aspect=1.6, inner sep=1pt, minimum width=1.6cm},
    terminal/.style={base, rectangle, rounded corners=10pt, minimum width=2.2cm},
    arrow/.style={-{Stealth[length=2.5mm]}, thick, rounded corners=4pt},
    yes/.style={arrow, draw=green!50!black},
    no/.style={arrow, draw=red!70!black},
    el/.style={font=\footnotesize, inner sep=2pt}
]

\node[process] (x) {$X$};
\node[decision, below=of x] (agree) {Agreers\\exist?};

\node[decision, right=of agree]    (neutral) {Neutral\\prediction?};
\node[terminal, right=of neutral]  (avg)     {$\bsL=\lneu$};

\node[decision, below=of neutral]  (worse)   {Worse agreeing\\models exist?};
\node[terminal, right=of worse]    (bagree)  {$\bsL=\lmon$};

\node[decision, below=of worse]    (optdec)  {Optimal\\prediction?};
\node[terminal, right=of optdec]   (term)    {$\bsL=\ldef$};

\node[decision, below=of optdec]   (disagree){Disagreers\\exist?};

\node[decision, below=of agree]    (optact)  {Optimal\\prediction?};
\node[terminal, left=of optact]    (ymin)    {$\bsL=\ldef$};
\node[terminal, below=2.5cm of optact] (bdisagree) {$\bsL=\lcor$};

\draw[arrow] (x) -- (agree);

\draw[yes] (agree)   -- node[el, above]{Yes} (neutral);
\draw[no]  (agree)   -- node[el, right]{No}  (optact);

\draw[yes] (neutral) -- node[el, above]{Yes} (avg);
\draw[no]  (neutral) -- node[el, right]{No}  (worse);

\draw[yes]  (worse)   -- node[el, above]{Yes}  (bagree);
\draw[no] (worse)   -- node[el, right]{No} (optdec);

\draw[no]  (optdec)  -- node[el, above]{No}  (term);
\draw[yes] (optdec)  -- node[el, right]{Yes} (disagree);

\draw[no]  (optact)  -- node[el, above]{No}  (ymin);
\draw[yes] (optact)  -- node[el, right]{Yes} (bdisagree);

\draw[yes] (disagree.west) -- node[el, above]{Yes} ++(-1,0) |- (bdisagree.east);

\draw[no]  (disagree.south) -- node[el, right]{No} ++(0,-0.8) -| (ymin.south);

\end{tikzpicture}
\caption{Flowchart with logic equivalent to that of $L(\pi_e)$}
\label{fig:lower-bound-flowchart}
\end{figure}
\end{document}